\documentclass{article}

    \PassOptionsToPackage{numbers, compress}{natbib}

    \usepackage[final]{neurips_2021}

\usepackage[utf8]{inputenc} 
\usepackage[T1]{fontenc}    
\usepackage{hyperref}       
\usepackage{url}            
\usepackage{booktabs}       
\usepackage{amsfonts}       
\usepackage{nicefrac}       
\usepackage{microtype}      
\usepackage{xcolor}         

\usepackage[chang]{changdefs}
\numberwithin{theorem}{section}
\allowdisplaybreaks[2]
\usepackage{mathtools} 
\usepackage[shortlabels]{enumitem}
\setlist{nosep}

\usepackage{graphicx}
\usepackage{wrapfig}
\graphicspath{{./figures/}}
\usepackage{caption}
\usepackage[leftcaption]{sidecap}
\usepackage{makecell}
\usepackage{multirow}
\makeatletter
\let\c@table\c@figure 
\let\ftype@table\ftype@figure 
\makeatother

\newcommand{\clSt}{{\tilde \clS}}
\newcommand{\clXt}{{\tilde \clX}}
\newcommand{\aseq}[1]{\overset{\text{\tiny $#1$}}{=}}
\newcommand{\assubseteq}[1]{\subseteq^{\text{$#1$}}}
\newcommand{\assupseteq}[1]{\supseteq^{\text{$#1$}}}

\let\origtimes\times
\renewcommand{\times}{\!\origtimes\!}
\let\origotimes\otimes
\renewcommand{\otimes}{\!\origotimes\!}

\usepackage{centernot}

\newcommand{\subpm}[1]{\scriptsize{$\pm\!$#1}}

\usepackage{xspace}
\newcommand{\ourmodel}{CyGen\xspace}
\newcommand{\ourmodelpt}{\ourmodel{}{\small{(PT)}}\xspace}

\usepackage{titlesec}
\titlespacing*{\section}{2pt}{4pt}{2pt} 
\titlespacing*{\subsection}{1pt}{3pt}{1pt} 
\titlespacing*{\subsubsection}{1pt}{2pt}{1pt} 

\title{On the Generative Utility of Cyclic Conditionals}

\author{%
  Chang Liu$^1$\thanks{Correspondence to: Chang Liu <\texttt{changliu@microsoft.com}>.},
  Haoyue Tang$^2$\thanks{Work done during an internship at Microsoft Research Asia.},
  Tao Qin$^1$, Jintao Wang$^2$, Tie-Yan Liu$^1$ \\
  $^1$ Microsoft Research Asia, Beijing, 100080. \;
  $^2$ Tsinghua University, Beijing, 100084.
}

\begin{document}
\abovedisplayskip=3pt
\belowdisplayskip=4pt
\abovedisplayshortskip=3pt
\belowdisplayshortskip=4pt

\vspace{-1pt}
\maketitle
\vspace{-1pt}

\begin{abstract}
  We study whether and how can we model a joint distribution $p(x,z)$ using two conditional models $p(x|z)$ and $q(z|x)$ that form a cycle.
  This is motivated by the observation that deep generative models, in addition to a likelihood model $p(x|z)$, often also use an inference model $q(z|x)$ for extracting representation, 
  but they rely on a usually uninformative prior distribution $p(z)$ to define a joint distribution, which may render problems like posterior collapse and manifold mismatch.
  To explore the possibility to model a joint distribution using only $p(x|z)$ and $q(z|x)$, we study their \emph{compatibility} and \emph{determinacy}, corresponding to the existence and uniqueness of a joint distribution whose conditional distributions coincide with them. 
  We develop a general theory for operable equivalence criteria for compatibility, 
  and sufficient conditions for determinacy. 
  Based on the theory, we propose a novel generative modeling framework \ourmodel that only uses the two cyclic conditional models.
  We develop methods to achieve compatibility and determinacy, and to use the conditional models to fit and generate data. 
  With the prior constraint removed, \ourmodel better fits data and captures more representative features, supported by both synthetic and real-world experiments. 
\end{abstract}

\vspace{-1pt}
\section{Introduction} \label{sec:intr}
\vspace{-1pt}

Deep generative models have achieved a remarkable success in the past decade for generating realistic complex data $x$ 
and extracting useful representations through their latent variable $z$.
Variational auto-encoders (VAEs)~\citep{kingma2014auto, rezende2014stochastic, bornschein2016bidirectional, bowman2016generating, kingma2016improved, van2018sylvester} follow the Bayesian framework and specify a prior distribution $p(z)$ and a likelihood model $p(x|z)$,
so that a joint distribution $p(x,z) = p(z) p(x|z)$ is defined for generative modeling (the joint induces a distribution $p(x)$ on data).
An inference model $q(z|x)$ is also used to approximate the posterior distribution $p(z|x)$ (derived from the joint $p(x,z)$), which serves for extracting representations. 
Other frameworks like generative adversarial nets~\citep{goodfellow2014generative, donahue2017adversarial, dumoulin2017adversarially}, flow-based models~\citep{dinh2015nice, papamakarios2017masked, kingma2018glow, grathwohl2019ffjord} and diffusion-based models~\citep{sohl2015deep, ho2020denoising, song2021score, kong2020diffwave} follow the same structure, with different choices of the conditional models $p(x|z)$ and $q(z|x)$ and training objectives.
While for the prior $p(z)$, there is often not much knowledge 
for complex data (like images, text, audio),
and these models widely adopt an uninformative prior such as a standard Gaussian.
This however, introduces some side effects:
\vspace{-1pt}
\begin{itemize}[leftmargin=18pt]
  \item \textbf{Posterior collapse}~\citep{bowman2016generating, he2019lagging, razavi2019preventing}: 
    The standard Gaussian prior tends to squeeze $q(z|x)$ towards the origin for all $x$, 
    which degrades the representativeness of the inferred $z$ for $x$ and hurts downstream tasks in the latent space like classification and clustering.
  \item \textbf{Manifold mismatch}~\citep{davidson2018hyperspherical, falorsi2018explorations, kalatzis2020variational}:
    Typically the likelihood model is continuous (keeps topology), 
    so the standard Gaussian prior would restrict the modeled data distribution to a simply-connected support,
    which limits the capacity for fitting data from a non-(simply) connected support.
\end{itemize}
\vspace{-1pt}
While there are works trying to mitigate the two problems, they require either a strong domain knowledge~\citep{kocaoglu2018causalgan, kalatzis2020variational}, or additional cost to learn a complicated prior model~\citep{li2018graphical, dai2019diagnosing, vahdat2020nvae} sometimes even at the cost of inconvenient inference~\citep{pang2020learning, xiao2021vaebm}. 

One question then naturally emerges: \emph{Can we model a joint distribution $p(x,z)$ only using the likelihood $p(x|z)$ and inference $q(z|x)$ models?}
If we can, the limitations from specifying or learning a prior are then removed from the root.
Also, the inference model $q(z|x)$ is then no longer a struggling approximation to a predefined posterior but participates in defining the joint distribution (avoid ``inner approximation'').
Modeling conditionals is also argued to be much easier than modeling marginal or joint distributions directly 
\citep{alain2014regularized, bengio2013generalized, bengio2014deep}.
In many cases, one may even have better knowledge on the conditionals than on the prior, \eg shift/rotation invariance of image representations (CNNs~\citep{lecun1989backpropagation} / SphereNet~\citep{coors2018spherenet}), and rules to extract frequency/energy features for audio~\citep{ren2020fastspeech}. 
It is then more natural and effective to incorporate this knowledge into the conditionals than using an uninformative prior. 

\vspace{-1pt}
In this paper, we explore such a possibility, and develop both a systematic theory and a novel generative modeling framework \ourmodel (\textbf{Cy}clic-conditional \textbf{Gen}erative model).

\vspace{-1pt}
\bfone Theoretical analysis on the question amounts to two sub-problems: can two given cyclic conditionals correspond to a common joint, and if yes, can they determine the joint uniquely.
We term them \emph{compatibility} and \emph{determinacy} of two conditionals, corresponding to the existence and uniqueness of a common joint.
For this, we develop novel compatibility criteria and sufficient conditions for determinacy. 
Beyond existing results, ours are operable (vs. existential~\citep{berti2014compatibility}) and self-contained (vs. need a marginal~\citep{bengio2013generalized, bengio2014deep, lamb2017gibbsnet, grover2019uncertainty}), and are general enough to cover both continuous and discrete cases.
Our compatibility criteria are also equivalence (vs. unnecessary~\citep{abrahams1984note, arnold1989compatible, arnold2001conditionally, arnold2012conditionally}) conditions.
The seminal book~\citep{arnold2012conditionally} makes extensive analysis for various parametric families.
Besides the equivalence criteria, we 
also extend their general analysis 
beyond the product support case, and also cover the Dirac case.

\vspace{-1pt}
\bftwo In addition to its independent contribution, the theory also enables generative modeling using only the two cyclic conditional models, \ie the \ourmodel framework.
We develop methods for achieving compatibility and determinacy to make an eligible generative model, and for fitting and generating data to serve as a generative model.
Efficient implementation techniques are designed.
Note \ourmodel also determines a prior implicitly; it just does not need an explicit model for the prior (vs.~\citep{li2018graphical, dai2019diagnosing, vahdat2020nvae, pang2020learning}).
We show the practical utility of \ourmodel in both synthetic and real-world tasks.
The improved performance in downstream classification and data generation demonstrates the advantage to mitigate the posterior collapse and manifold mismatch problems. 

\vspace{-1pt}
\subsection{Related work} \label{sec:relw}
\vspace{-2pt}

\textbf{Dependency networks} (\citep{heckerman2000dependency}; similarly~\citep{hofmann1998nonlinear})
are perhaps the first to pursue the idea of modeling a joint by a set of conditionals.
They use Gibbs sampling to determine the joint 
and are equivalent 
to undirected graphical models. 
They do not allow latent variables, so compatibility is not a key consideration as the data already specifies a joint as the common target of the conditionals.
Beyond that, we introduce latent variables to better handle sensory data like images,
for which we analyze the conditions for compatibility and determinacy and design novel methods to solve this different task.

\vspace{-1pt}
\textbf{Denoising auto-encoders} (DAEs).
AEs~\citep{rumelhart1986learning, baldi1989neural} aim to extract data features by enforcing reconstruction through its encoder and decoder, 
which are deterministic hence insufficient determinacy (see Sec.~\ref{sec:determ-dirac}). 
DAEs~\citep{vincent2008extracting, bengio2013generalized, bengio2014deep} use a probabilistic encoder and decoder for robust reconstruction against random data corruption.
Their utility as a generative model is first noted through the equivalence to score matching (implies modeling $p(x)$) for a Gaussian RBM~\citep{vincent2011connection} or an infinitesimal Gaussian corruption~\citep{alain2014regularized}.
In more general cases, the utility to modeling the joint $p(x,z)$ is studied via the Gibbs chain, \ie the Markov chain with transition kernel $p(x'|z') q(z'|x)$.
Under a global~\citep{bengio2013generalized, lamb2017gibbsnet, grover2019uncertainty} or local~\citep{bengio2014deep} shared support condition, its stationary distribution $\pi(x,z)$ exists uniquely.
But this is \emph{not really determinacy}: even incompatible conditionals can have this unique existence, in which case $\pi(z|x) \ne q(z|x)$~\citep{heckerman2000dependency, bengio2013generalized}.
Moreover, 
the Gibbs chain does not give an explicit expression of $\pi(x,z)$ (thus intractable likelihood evaluation), and requires many iterations 
to converge for data generation and even for training (Walkback~\citep{bengio2013generalized}, GibbsNet~\citep{lamb2017gibbsnet}), making the algorithms costly and unstable.

\vspace{-1pt}
As for \emph{compatibility}, it is not really covered in DAEs.
Existing results only consider the statistical consistency (unbiasedness under infinite data) of the $p(x|z)$ estimator by fitting $(x,z)$ data from $p^*(x) q(z|x)$~\citep{bengio2013generalized, bengio2014deep, lamb2017gibbsnet, grover2019uncertainty}, where $p^*(x)$ denotes the true data distribution.
Particularly, they require a marginal $p^*(x)$ in advance, so that the joint is already defined by $p^*(x) q(z|x)$ regardless of $p(x|z)$, while compatibility (as well as determinacy) is a matter only of the two conditionals.

\vspace{-1pt}
More crucially, the DAE loss is not proper for optimizing $q(z|x)$ as it promotes a mode-collapse behavior.
This hinders both compatibility and determinacy (Sec.~\ref{sec:meth-data}):
one may not use $q(z|x)$ for inference, and 
data generation may depend on initialization.
In contrast, \ourmodel explicitly enforces compatibility and guarantees determinacy, and enables likelihood evaluation and better generation.


\vspace{-1pt}
\textbf{Dual learning} \hspace{4pt}
considers conversion between two modalities in both directions, \eg, machine translation~\citep{he2016dual, xia2017duals, xia2017duali} and image style transfer~\citep{kim2017learning, zhu2017unpaired, yi2017dualgan, lin2019exploring}.
Although we also consider both directions, 
the fundamental distinction is that in generative modeling there is no data of the latent variable $z$ (not even unpaired). 
Technically, they did not consider determinacy: they require a marginal 
to determine a joint. 
We find their determinacy is actually insufficient (see Sec.~\ref{sec:determ-dirac}).
Their cycle-consistency loss~\citep{kim2017learning, zhu2017unpaired, yi2017dualgan} is a version of our compatibility criterion in the Dirac case (see Sec.~\ref{sec:compt-dirac-crit}),
and we extend it to allow probabilistic conversion 
(see Sec.~\ref{sec:meth-compt}).

\vspace{-1pt}
\section{Compatibility and Determinacy Theory} \label{sec:thry}
\vspace{-2pt}

To be a generative model, a system needs to determine a distribution on the data variable $x$.
With latent variable $z$, this amounts to determining a joint distribution over $(x,z)$,
which calls for compatibility and determinacy analysis for cyclic conditionals.
In this section we build a general theory on the conditions for compatibility and determinacy.
We begin with formalizing the problems.

\vspace{-1pt}
\textbf{Setup.} \hspace{4pt}
Denote the measure spaces of the two random variables $x$ and $z$ as $(\bbX, \scX, \xi)$ and $(\bbZ, \scZ, \zeta)$\footnote{
  The symbol $\bbZ$ overwrites the symbol for the set of integers, which is not involved in this paper.
}, 
where $\scX$, $\scZ$ are the respective sigma-fields, and the base measures $\xi$, $\zeta$ (\eg, Lebesgue measure on Euclidean spaces, counting measure on finite/discrete spaces) are sigma-finite. 
We use $\clX \in \scX$, $\clZ \in \scZ$ to denote measurable sets, and use ``{\small $\aseq{\xi}$}'', ``$\assubseteq{\xi}$'' as the extensions of ``$=$'', ``$\subseteq$'' up to a set of $\xi$-measure-zero (Def.~\ref{def:assame-assubseteq}).
Following the convention in machine learning, we call a ``probability measure'' as a ``distribution''. 
We do not require any further structures such as topology, metric, or linearity,
for the interest of the most general conclusions that unify Euclidean/manifold and finite/discrete spaces and allow $\bbX$, $\bbZ$ to have different dimensions or types.

\vspace{-1pt}
Joint and conditional distributions are defined on the product measure space $(\bbX\times\bbZ, \scX\otimes\scZ, \xi\otimes\zeta)$,
where ``$\times$'' is the usual Cartesian product, $\scX\otimes\scZ := \sigma(\scX\times\scZ)$ is the sigma-field generated by measurable rectangles from $\scX\times\scZ$, and $\xi\otimes\zeta$ is the product measure~\citep[Thm.~18.2]{billingsley2012probability}.
Define the \emph{slice} of $\clW \in \scX\otimes\scZ$ at $z$ as $\clW_z := \{x \mid (x,z) \in \clW\} \in \scX$~\citep[Thm.~18.1(i)]{billingsley2012probability}, 
and its \emph{projection} onto $\bbZ$ as $\clW^\bbZ := \{z \mid \exists x \in \bbX \st (x,z) \in \clW\} \in \scZ$ (Appx.~\ref{supp:meas-prod}).
In a similar style, denote the \emph{marginal} of a joint $\pi$ on $\bbZ$ as $\pi^\bbZ(\clZ) := \pi(\bbX \times \clZ)$. 
To keep the same level of generality, we follow the general definition of conditionals (\citep[p.457]{billingsley2012probability}; see also Appx.~\ref{supp:meas-cond}):
the conditional $\pi(\clX|z)$ of a joint $\pi$ is the density function (R-N derivative) 
of $\pi(\clX \times \cdot)$ w.r.t $\pi^\bbZ$.
We highlight the key characteristic under this generality that $\pi(\cdot|z)$ can be arbitrary on a set of $\pi^\bbZ$-measure-zero, particularly, outside the support of $\pi^\bbZ$. 
Appx.~\ref{supp:meas} and~\ref{supp:lemma} provide more background details and our technical preparations that are also of independent interest.
The goal of analysis can then be formalized below.
\begin{definition}[compatibility and determinacy] \label{def:compt-determ}
  We say two conditionals $\mu(\clX|z)$, $\nu(\clZ|x)$ are \emph{compatible},
  if there exists a joint distribution $\pi$ on $(\bbX\times\bbZ, \scX\otimes\scZ)$ such that 
  $\mu(\clX|z)$ and $\nu(\clZ|x)$ are its conditional distributions. 
  We say two compatible conditionals have \emph{determinacy} on a set $\clS \in \scX\otimes\scZ$, if there is only one joint distribution concentrated on $\clS$ that makes them compatible.
\end{definition}
\vspace{-3pt}
To put the concept into practical use, the analysis aims at operable conditions for compatibility and determinacy.
We consider two cases separately (still unifying continuous and discrete cases), as they correspond to different types of generative models, and lead to qualitatively different conclusions.

\vspace{-1pt}
\subsection{Absolutely Continuous Case} \label{sec:ac}
\vspace{-3pt}

We first consider the case where for any $z \in \bbZ$ and any $x \in \bbX$,
\footnote{There may be problems if absolute continuity holds only for $\zeta$-a.e. $z$ and $\xi$-a.e. $x$; see Appx. Example~\ref{exmp:joint-ac}.}
the conditionals $\mu(\cdot|z)$ and $\nu(\cdot|x)$ are either absolutely continuous (w.r.t $\xi$ and $\zeta$, resp.)~\citep[p.448]{billingsley2012probability}, or zero in the sense of a measure.
Equivalently, they have density functions $p(x|z)$ and $q(z|x)$ (non-negative by definition; may integrate to zero).
This case include ``smooth'' distributions on Euclidean spaces or manifolds, and \emph{all} distributions on finite/discrete spaces.
Many generative modeling frameworks use density models thus count for this case, including VAEs~\citep{kingma2014auto, rezende2014stochastic, rezende2015variational, kingma2016improved, van2018sylvester} and diffusion-based models~\citep{sohl2015deep, ho2020denoising, song2021score}.

\vspace{-1pt}
\subsubsection{Compatibility criterion in the absolutely continuous case} \label{sec:compt-ac-crit}
\vspace{-2pt}

One may expect that when absolutely continuous conditionals $p(x|z)$ and $q(z|x)$ are compatible, their joint is also absolutely continuous (w.r.t $\xi\otimes\zeta$) with some density $p(x,z)$.
This intuition is verified by our Lem.~\ref{lem:joint-ac} in Appx.~\ref{supp:proofs-joint-ac}.
One could then safely apply density function formulae and get $\frac{p(x|z)}{q(z|x)} = \frac{p(x,z)}{p(z)} / \frac{p(x,z)}{p(x)} = \frac{p(x)}{p(z)}$
factorizes into a function of $x$ and a function of $z$. 
Conversely, if the ratio factorizes as such $\frac{p(x|z)}{q(z|x)} = a(x) b(z)$, one could get $p(x|z) \frac{1}{A b(z)} = q(z|x) \frac{a(x)}{A}$ where $A := \int_\bbX a(x) \xi(\ud x)$, which defines a joint density and compatibility is achieved.
This intuition leads to the classical compatibility criterion [\citealp{abrahams1984note}; \citealp[Thm.~4.1]{arnold1989compatible}; \citealp[Thm.~1]{arnold2001conditionally}; \citealp[Thm.~2.2.1]{arnold2012conditionally}].

However, the problem is more complicated than imagined.
\citet[Example~9]{berti2014compatibility} point out that the classical criterion is unfortunately \emph{not necessary}.
The subtlety is about on which region does this factorization have to hold.
The classical criterion requires it to be the positive region of $p(x|z)$ which also needs to coincide with that of $q(z|x)$.
But as mentioned, conditional $\mu(\cdot|z)$ can be arbitrary outside the support of the marginal $\pi^\bbZ$ (similarly for $\nu(\cdot|x)$),
which may lead to additional positive regions that violate the requirement.
\footnote{The flexibility of $p(x|z)$ on a $\xi$-measure-zero set for a given $z$ (similarly for $q(z|x)$) is not a vital problem, as one can adjust the conditions to hold only a.e.}
To address the problem, \citet{berti2014compatibility} give an equivalence criterion (Thm.~8),
but it is \emph{existential} thus less useful as the definition of compatibility itself is existential.
Moreover, these criteria are restricted to either Euclidean or discrete spaces.

Next we give our \emph{equivalence} criterion that is \emph{operable}. 
In addressing the subtlety with regions, we first introduce a related concept that helps identify appropriate regions.

\begin{wrapfigure}{r}{.320\textwidth}
  \centering
  \vspace{-12pt}
  \includegraphics[width=.318\textwidth]{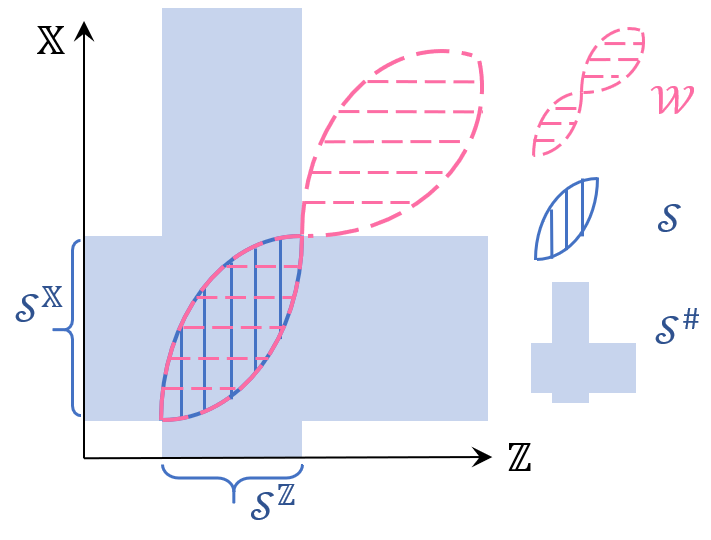}
  \vspace{-8pt}
  \caption{Illustration of a $\xi\otimes\zeta$-complete component $\clS$ of $\clW$.
  }
  \vspace{-8pt}
  \label{fig:irrcomp}
\end{wrapfigure}
\phantom{a}
\vspace{-14pt}

\begin{definition}[$\xi\otimes\zeta$-complete component] \label{def:irrcomp}
  For a set $\clW \in \scX\otimes\scZ$, we say that a set $\clS \in \scX\otimes\scZ$ is a \emph{$\xi\otimes\zeta$-complete component} of $\clW$, if
  $\clS^\sharp \cap \clW \aseq{\xi\otimes\zeta} \clS$, where $\clS^\sharp := \clS^\bbX \times \bbZ \cup \bbX \times \clS^\bbZ$ is the \emph{stretch} of $\clS$.
\end{definition}
\vspace{-1pt}
Fig.~\ref{fig:irrcomp} illustrates the concept.
Roughly, the stretch $\clS^\sharp$ of $\clS$ represents the region where the conditionals are a.s. determined if $\clS$ is the \emph{support}\footnote{
  While the typical definition of support requires a topological structure which is absent under our generality,
  Def.~\ref{def:support} in Appx.~\ref{supp:lemma-general} defines such a concept for absolutely continuous distributions.
} of the joint.
If $\clS$ is a complete component of $\clW$, it is complete under stretching and intersecting with $\clW$.
Such a set $\clS$ is an a.s. subset of $\clW$ (Lem.~\ref{lem:irrcomp-assubseteq}), while has a.s. the same slice as $\clW$ does for almost all $z \in \clS^\bbZ$ and $x \in \clS^\bbX$ (Lem.~\ref{lem:irrcomp-inteq}). 
This is critical for the normalizedness of distributions in our criterion.
Appx.~\ref{supp:lemma-irrcomp} shows more facts. 
With this concept, our compatibility criterion is presented below.


\begin{wrapfigure}{r}{.470\textwidth}
  \centering
  \vspace{-18pt}
  \includegraphics[width=.460\textwidth]{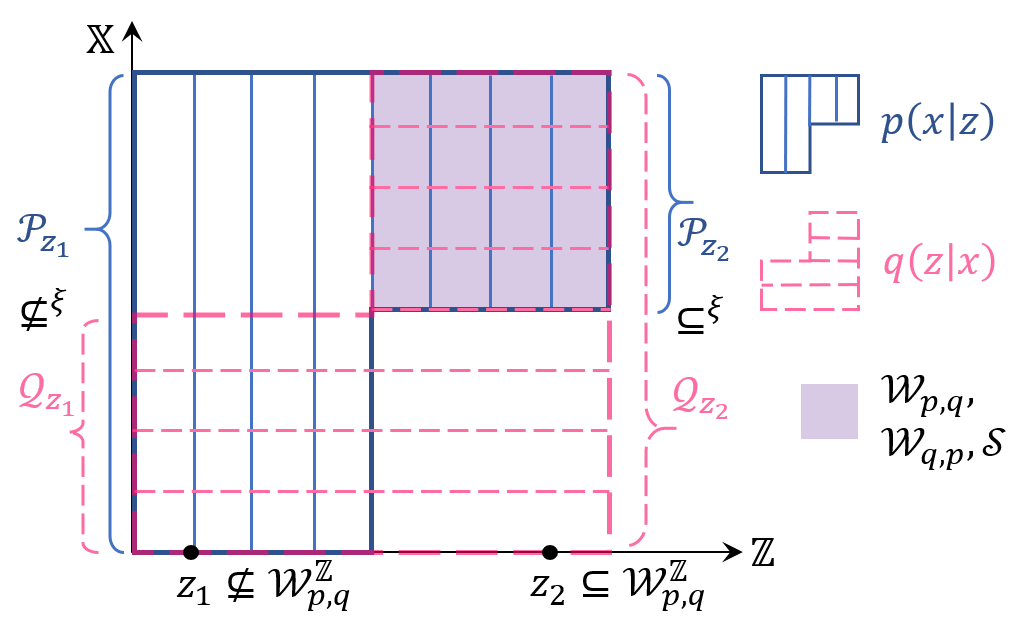}
  \vspace{0pt}
  \caption{Illustration of our compatibility criterion in the absolutely continuous case (Thm.~\ref{thm:compt-ac}).
    The conditionals are uniform on the respective depicted slices. 
    For condition \bfi, $\clP_z \assubseteq{\xi} \clQ_z$ is \emph{not} satisfied on the left half, \eg $z_1$, so $\clW_{p,q}$ does not cover the left half;
    it is satisfied on the right half, \eg $z_2$,
    so $\clW_{p,q}$ is composed of slices $\clP_z$ on the right half, making the top-right quadrant (shaded).
    Similarly, $\clW_{q,p}$ is the same region,
    and it is a $\xi\otimes\zeta$-complete component of itself.
    It also satisfies other conditions thus is a complete support $\clS$.
  }
  \vspace{-16pt}
  \label{fig:compt-ac}
\end{wrapfigure}
\phantom{a}
\vspace{-14pt}

\begin{theorem}[compatibility criterion, absolutely continuous] \label{thm:compt-ac}
  Let $p(x|z)$ and $q(z|x)$ be the density functions of two everywhere absolutely continuous (or zero) conditional distributions, and define:
  \begin{align}
    & \clP_z := \{x \mid p(x|z) > 0\}, \clP_x := \{z \mid p(x|z) > 0\}, \\ 
    & \clQ_z := \{x \mid q(z|x) > 0\}, \clQ_x := \{z \mid q(z|x) > 0\}.
  \end{align}
  Then they are compatible, if and only if they have a \emph{complete support} $\clS$,
  defined as a
  \bfi $\xi\otimes\zeta$-complete component of both
  \begin{align}
    \clW_{p,q} \! := \!\!\!\!\! \bigcup_{z: \clP_z \assubseteq{\xi} \clQ_z} \!\!\!\! \clP_z \times \{z\}, \;
    \clW_{q,p} \! := \!\!\!\!\! \bigcup_{x: \clQ_x \assubseteq{\zeta} \clP_x} \!\!\!\! \{x\} \times \clQ_x,
    \label{eqn:def-pqsupp}
  \end{align}
  such that:
  \bfii $\clS^\bbX \assubseteq{\xi} \clW_{q,p}^\bbX$, $\clS^\bbZ \assubseteq{\zeta} \clW_{p,q}^\bbZ$,
  \bfiii $(\xi\otimes\zeta)(\clS) > 0$, and
  \bfiv $\frac{p(x|z)}{q(z|x)}$ factorizes as $a(x) b(z)$, $\xi\otimes\zeta$-a.e. on $\clS$,\footnote{
  Formally, there exist functions $a$ on $\clS^\bbX$ and $b$ on $\clS^\bbZ$ s.t. $(\xi\otimes\zeta)\{(x,z) \in \clS \mid \frac{p(x|z)}{q(z|x)} \ne a(x) b(z)\} = 0$.} where
  \bfv $a(x)$ is $\xi$-integrable on $\clS^\bbX$.
  For sufficiency,
  \begin{align} \label{eqn:def-suff-pi-unroll}
    \pi(\clW) := \frac{\int_{\clW \cap \clS} q(z|x) \lrvert{a(x)} (\xi\otimes\zeta)(\ud x \ud z)}{\int_{\clS^\bbX} \lrvert{a(x)} \xi(\ud x)},
  \end{align}
  $\forall \clW \in \scX\otimes\scZ$, is a compatible joint of them. 
\end{theorem}
Fig.~\ref{fig:compt-ac} shows an illustration of the conditions.
To understand the criterion, conditions \bfiv and \bfv stem from the starting inspiration, which also shows a hint for \eqref{eqn:def-suff-pi-unroll}. 
Other conditions handle the subtlety to find a region $\clS$ where \bfiv and \bfv must hold.
This is essentially the support of a compatible joint $\pi$ as there is no need and no way to control conditionals outside the support.

For necessity, informally, if $z$ 
is in the support of $\pi^\bbZ$, then $p(x|z)$ determines the distribution on $\bbX \times \{z\}$;
particularly, the joint $\pi$ should be a.e. positive on $\clP_z$, which in turn asks $q(z|x)$ to be so.
This means $\clP_z \assubseteq{\xi} \clQ_z$
(unnecessary equal, since $q(z|x)$ is ``out of control'' outside the joint support),
which leads to the definitions of $\clW_{p,q}$ and $\clW_{q,p}$.
The joint support should be contained within the two sets in order to avoid \emph{support conflict}
(\eg, although the bottom-left quadrant in Fig.~\ref{fig:compt-ac} is part of the intersection of positive regions of the conditionals, 
a joint on it is required by $p(x|z)$ to also cover the top-left, on which $q(z|x)$ does not agree).
Condition \bfi indicates $\clS \assubseteq{\xi\otimes\zeta} \clW_{p,q}$ and $\clW_{q,p}$
so $\clS$ satisfies this requirement and also makes the ratio in \bfiv a.e. well-defined. 
The complete-component condition in \bfi also makes the conditionals \emph{normalized} on $\clS$:
as mentioned, such an $\clS$ has a.s. the same slice as $\clW_{p,q}$ does for a given $z$ in support $\clS^\bbZ$,
so the integral of $p(x|z)$ on $\clS_z$ is the same as that on $(\clW_{p,q})_z = \clP_z$ which is 1 by construction; similarly for $q(z|x)$.
In contrast, Appx. Example~\ref{exmp:no-irrcomp} 
shows $\clS = \clW_{p,q} \cap \clW_{q,p}$ is inappropriate. 
Conditions \bfii and \bfiii cannot be guaranteed by condition \bfi (Appx. Example~\ref{exmp:irrcomp-projsubset}), while are needed 
to rule out special cases (Appx. Lem.~\ref{lem:irrcomp-assame}, Example~\ref{exmp:irrcomp-assame}).
Appx.~\ref{supp:proofs-compt-ac} gives a formal proof. 
Finally, although the criterion relies on the \emph{existence} of such a complete support, candidates are few (if any), so it is \emph{operable}. 

\vspace{-1pt}
\subsubsection{Determinacy in the absolutely continuous case} \label{sec:determ-ac}
\vspace{-1pt}

When compatible, absolutely continuous cyclic conditionals are very likely to have determinacy. 
\begin{theorem}[determinacy, absolutely continuous] \label{thm:determ-ac}
  Let $p(x|z)$ and $q(z|x)$ be two compatible conditional densities, and $\clS$ be a complete support that makes them compatible (necessarily exists due to Thm.~\ref{thm:compt-ac}).
  Suppose that $\clS_z \aseq{\xi} \clS^\bbX$, for $\zeta$-a.e. $z$ on $\clS^\bbZ$, or $\clS_x \aseq{\zeta} \clS^\bbZ$, for $\xi$-a.e. $x$ on $\clS^\bbX$.
  Then their compatible joint supported on $\clS$ is unique,
  which is given by \eqref{eqn:def-suff-pi-unroll}.
\end{theorem}
\vspace{-1pt}
Proof is given in Appx.~\ref{supp:proofs-determ-ac}.
The condition in the theorem roughly means that the complete support $\clS$ is ``rectangular''.
From the perspective of Markov chain, this corresponds to the \emph{irreducibility} of the Gibbs chain 
for the unique existence of a stationary distribution.
When the conditionals have multiple such complete supports, on each of which the compatible joint is unique,
while globally on $\bbX\times\bbZ$, they may have multiple compatible joints. 
In general, determinacy in the absolutely continuous case is \emph{sufficient}, particularly we have the following strong conclusion in a common case (\eg, for VAEs).
\begin{corollary} \label{cor:determ-ac-full-supp}
  We call two conditional densities \emph{have a.e.-full supports}, if $p(x|z) > 0, q(z|x) > 0$ for $\xi\otimes\zeta$-a.e. $(x,z)$. 
  If they are compatible, then their compatible joint is unique,
  since $\bbX\times\bbZ$ is the $\xi\otimes\zeta$-unique complete support (Prop.~\ref{prop:full-supp-irrcomp} in Appx.~\ref{supp:proofs-full-supp-irrcomp}),
  which satisfies the condition in Thm.~\ref{thm:determ-ac}.
\end{corollary}
\vspace{-1pt}

\vspace{-1pt}
\subsection{Dirac Case} \label{sec:dirac}
\vspace{-1pt}

Many other prevailing generative models, including generative adversarial networks (GANs)~\citep{goodfellow2014generative} and flow-based models~\citep{dinh2015nice, papamakarios2017masked, kingma2018glow, grathwohl2019ffjord},
use a deterministic function $x = f(z)$ as the likelihood model.
In such cases, the conditional $\mu(\clX|z) = \delta_{f(z)}(\clX) := \bbI[f(z) \in \clX], \forall \clX \in \scX$ is a Dirac measure.
Note it does not have a density function when $\xi$ assigns zero to all single-point sets, \eg the Lebesgue measure on Euclidean spaces, so we keep the measure notion. 
This case is not exclusive to the absolutely continuous case: a Dirac conditional on a discrete space is also absolutely continuous.

\vspace{-1pt}
\subsubsection{Compatibility criterion in the Dirac case} \label{sec:compt-dirac-crit}
\vspace{-1pt}

\begin{wrapfigure}{r}{.328\textwidth}
  \centering
  \vspace{-34pt}
  \includegraphics[width=.308\textwidth]{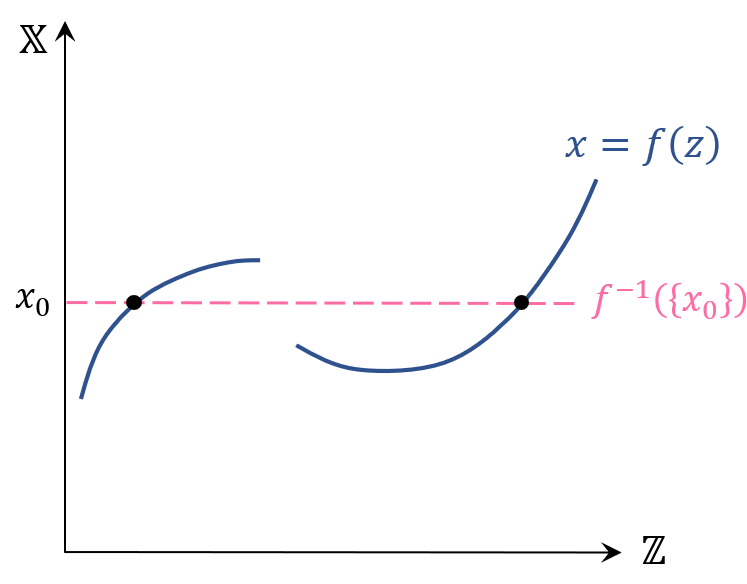}
  \vspace{-8pt}
  \caption{Illustration of our compatibility criterion in the Dirac case (Thm.~\ref{thm:compt-dirac}).
  }
  \vspace{-24pt}
  \label{fig:compt-dirac}
\end{wrapfigure}
\phantom{a}
\vspace{-16pt}

Compatibility criterion is easier to imagine in this case.
As illustrated in Fig.~\ref{fig:compt-dirac}, it is roughly that the other-way conditional $\nu(\cdot|x)$ could find a way to put its mass only on the curve;
otherwise support conflict is rendered.
\begin{theorem}[compatibility criterion, Dirac] \label{thm:compt-dirac}
  Suppose that $\scX$ contains all the single-point sets: $\{x\} \in \scX, \forall x \in \bbX$.
  Conditional distribution $\nu(\clZ|x)$ is compatible with $\mu(\clX|z) := \delta_{f(z)}(\clX)$ where function $f: \bbZ \to \bbX$ is $\scX$/$\scZ$-measurable\footnote{
    For meaningful discussion, we require $f$ to be $\scX$/$\scZ$-measurable, which includes any function between discrete sets and continuous functions when $\scX$ and $\scZ$ are the Borel sigma-fields.
  }, if and only if there exists $x_0 \in \bbX$ such that $\nu(f^{-1}(\{x_0\}) | x_0) = 1$.
\end{theorem}
See Appx.~\ref{supp:proofs-compt-dirac} for proof.
Note such $x_0$ must be in the image set $f(\bbZ)$, otherwise $\nu(f^{-1}(\{x_0\}) | x_0) = \nu(\emptyset | x_0) = 0$.
For a typical GAN generator, the preimage set $f^{-1}(\{x_0\})$ is discrete, so a compatible inference model must not be absolutely continuous.
What may be counter-intuitive is that $\nu(\cdot|x)$ is not required to concentrate on the curve for \emph{any} $x$;
one $x_0$ is sufficient as $\delta_{(x_0, f(x_0))}$ is a compatible~joint.

Nevertheless, in practice one often desires the compatibility to hold over a set $\clX$ to make a useful model. 
When $\nu(\cdot|x)$ is also chosen in the Dirac form $\delta_{g(x)}$, this can be achieved by minimizing $\bbE_{p(x)} \ell \big( x, f(g(x)) \big)$,
where $p(x)$ is a distribution on $\clX$ and $\ell$ is a metric on $\bbX$.
This is the \emph{cycle-consistency loss} used in dual learning~\citep{kim2017learning, zhu2017unpaired, yi2017dualgan, lin2019exploring}.
When $f$ is invertible, minimizing the loss (\ie, $g = f^{-1}$ a.e. on $\clX$) is also necessary, as $f^{-1}(x)$ only has one element. 
Particularly, flow-based models are naturally compatible (so are their injective variants~\citep{ardizzone2019analyzing, xiao2020invertible, brehmer2020flows} by Thm.~\ref{thm:compt-dirac}). 

\vspace{-1pt}
\subsubsection{Determinacy in the Dirac case} \label{sec:determ-dirac}
\vspace{-1pt}

As mentioned, for any $x_0$ satisfying the condition, two compatible conditionals have the determinacy on this point $\{x_0\}$ with the unique joint $\delta_{(x_0, f(x_0))}$.
But when such $x_0$ is not unique, the distribution over these $x_0$ values is not determined, so the two conditionals do not have the determinacy globally on $\bbX\times\bbZ$. 
This is similar to the absolutely continuous case with multiple complete supports;
particularly, each $\{(x_0, f(x_0))\}$ is a complete support for discrete $\bbX$ and $\bbZ$.
This meets one's intuition: compatible Dirac conditionals can only determine a curve in $\bbX\times\bbZ$, 
but cannot determine a distribution on the curve.
One exception is when $f(z) \equiv x_0$ is constant, so this $x_0$ is the only candidate. 
The joint then degenerates to a distribution on $\bbZ$, which is fully determined by $\nu(\cdot|x_0)$.

In general, determinacy in the Dirac case is \emph{insufficient},
and this type of generative models (GANs, flow-based models) have to specify a prior to define a joint.

\vspace{-1pt}
\section{Generative Modeling using Cyclic Conditionals} \label{sec:meth}
\vspace{-1pt}

The theory suggests it is possible that cyclic conditionals achieve compatibility and a sufficient determinacy,
so that they can determine a useful joint without specifying a prior.
Note a certain prior is implicitly determined by the conditionals; we find we just do not need an explicit model for it. 
This inspires \ourmodel, a novel framework that only uses \textbf{Cy}clic conditionals for \textbf{Gen}erative modeling.

For the eligibility as a generative model, compatibility and a sufficient determinacy are required.
For the latter, we just shown a deterministic likelihood or inference model is not suitable, so we use absolutely continuous conditionals as the theory suggests. 
The conditionals can then be modeled by parameterized densities $p_\theta(x|z)$, $q_\phi(z|x)$.
We consider the common case where $\bbX = \bbR^{d_\bbX}$, $\bbZ = \bbR^{d_\bbZ}$, and $p_\theta(x|z)$, $q_\phi(z|x)$ have a.e.-full supports and are differentiable. 
Determinacy is then exactly guaranteed by Cor.~\ref{cor:determ-ac-full-supp}.
For compatibility, we develop an effective loss in Sec.~\ref{sec:meth-compt} to enforce it.

For the usage as a generative model, we develop methods to fit the model-determined data distribution $p_{\theta,\phi}(x)$ to the true data distribution $p^*(x)$ in Sec.~\ref{sec:meth-data},
and to generate data from $p_{\theta,\phi}(x)$ in Sec.~\ref{sec:meth-gen}.

\vspace{-1pt}
\subsection{Enforcing Compatibility} \label{sec:meth-compt}
\vspace{-1pt}

In this a.e.-full support case, the entire product space $\bbX\times\bbZ$ is the only possible complete support (Prop.~\ref{prop:full-supp-irrcomp} in Appx.~\ref{supp:proofs-full-supp-irrcomp}),
so for compatibility, condition~\bfiv in Thm.~\ref{thm:compt-ac} is the most critical one. 
For this, we do not have to find functions $a(x)$, $b(z)$ in Thm.~\ref{thm:compt-ac}, but only need to enforce such a factorization.
So we propose the following loss function to enforce compatibility:
\begin{align} \label{eqn:compt-obj-jac}
  (\min_{\theta,\phi}) \;\;
  C(\theta,\phi) := \bbE_{\rho(x,z)} \lrVert{\nabla_x \nabla_z\trs r_{\theta,\phi}(x,z)}_F^2, \text{where }
  r_{\theta,\phi}(x,z) := \log \big( p_\theta(x|z) / q_\phi(z|x) \big).
\end{align}
Here, $\rho$ is some absolutely continuous reference distribution on $\bbX\times\bbZ$, 
which can be taken as $p^*(x) q_\phi(z|x)$ in practice as it gives samples to estimate the expectation.
When $C(\theta,\phi) = 0$, we have $\nabla_x \nabla_z\trs r_{\theta,\phi}(x,z) = 0$, $\xi\otimes\zeta$-a.e. 
\citep[Thm.~15.2(ii)]{billingsley2012probability}.
By integration, this means $\nabla_z r_{\theta,\phi}(x,z) = V(z)$ hence $r_{\theta,\phi}(x,z) = v(z) + u(x)$, $\xi\otimes\zeta$-a.e.,
for some functions $V(z)$, $v(z)$, $u(x)$ s.t. $V(z) = \nabla v(z)$.
So the ratio $p_\theta(x|z) / q_\phi(z|x) = \exp\{ r_{\theta,\phi}(x,z) \} = \exp\{u(x)\} \exp\{v(z)\}$ factorizes, $\xi\otimes\zeta$-a.s. 

In the sense of enforcing compatibility, this loss generalizes the cycle-consistency loss to probabilistic conditionals.
Also, the loss is different from the Jacobian-norm regularizers in contractive AE~\citep{rifai2011contractive} and DAE~\citep{rifai2011contractive, alain2014regularized}, and explains the ``tied weights'' trick for AEs~\citep{ranzato2007sparse, vincent2008extracting, vincent2011connection, rifai2011contractive, alain2014regularized} (see Appx.~\ref{supp:meth-aereg}).

\textbf{Implication on Gaussian VAE} \hspace{4pt}
which uses additive Gaussian conditional models, $p_\theta(x|z) := \clN(x | f_\theta(z), \sigma_d^2 I_{d_\bbX})$ and $q_\phi(z|x) := \clN(z | g_\phi(x), \sigma_e^2 I_{d_\bbZ})$.
It is the vanilla and the most common form of VAE~\citep{kingma2014auto}.
As its ELBO objective drives $q_\phi(z|x)$ to meet the joint $p(z) p_\theta(x|z)$, compatibility is enforced.
Under our view, this amounts to minimizing the compatibility loss \eqref{eqn:compt-obj-jac}, which then enforces the match of Jacobians: $(\nabla_z f_\theta\trs(z))\trs = (\sigma_d^2 / \sigma_e^2) \nabla_x g_\phi\trs(x)$.
As the two sides indicate the equation is constant of both $x$ and $z$, it must be a constant, so \emph{$f_\theta(z)$ and $g_\phi(x)$ must be affine},
and the joint is also a Gaussian~[\citealp{bhattacharyya1943some}; \citealp[Thm.~3.3.1]{arnold2012conditionally}].
This conclusion coincides with the theory on additive noise models in causality~\citep{zhang2009identifiability, peters2014causal}, and explains the empirical observation that the latent space of such VAEs is quite linear~\citep{shao2018riemannian}.
It is also the root of recent analyses that the latent space coordinates the data manifold~\citep{dai2019diagnosing}, and the inference model learns an isometric embedding after a proper rescaling~\citep{nakagawa2021quantitative}.

This finding reveals that the expectation to use deep neural networks for learning a flexible nonlinear representation will be disappointed in Gaussian VAE.
So we use a non-additive-Gaussian model, \eg a flow-based model~\citep{rezende2015variational, kingma2016improved, van2018sylvester, grathwohl2019ffjord}, for at least one of $p_\theta(x|z)$ and $q_\phi(z|x)$ (often the latter).

\textbf{Efficient implementation.} \hspace{4pt}
Direct Jacobian evaluation for \eqref{eqn:compt-obj-jac} is of complexity $O(d_\bbX d_\bbZ)$, which is often prohibitively large.
We thus propose a stochastic but unbiased and much cheaper method based on Hutchinson's trace estimator~\citep{hutchinson1989stochastic}:
$\tr(A) = \bbE_{p(\eta)}[\eta\trs A \eta]$, where $\eta$ is any random vector with zero mean and identity covariance (\eg, a standard Gaussian).
As the function within expectation is 
$\lrVert{\nabla_x \nabla_z\trs r}_F^2 = \lrVert{\nabla_z \nabla_x\trs r}_F^2 = \tr\big( (\nabla_z \nabla_x\trs r)\trs \nabla_z \nabla_x\trs r \big)$,
applying the estimator yields a formulation that reduces gradient evaluation complexity to $O(d_\bbX + d_\bbZ)$:
{\belowdisplayskip=1pt
\begin{align} \label{eqn:compt-obj-jac-hut}
  (\min_{\theta,\phi}) \;\;
  C(\theta,\phi)
  ={} & \bbE_{\rho(x,z)} \bbE_{p(\eta_x)} \lrVert{\nabla_z \big( \eta_x\trs \nabla_x r_{\theta,\phi}(x,z) \big)}_2^2, \text{where }
  \bbE[\eta_x] = 0, \Var[\eta_x] = I_{d_\bbX}.
\end{align} }%
%
As concluded from the above analysis on Gaussian VAE, we use a flow-based model for the inference model $q_\phi(z|x)$.
But in common instances evaluating the inverse of the flow is intractable~\citep{rezende2015variational, kingma2016improved, van2018sylvester} or costly~\citep{grathwohl2019ffjord}.
This however, disables the use of automatic differentiation tools for estimating the gradients in the compatibility loss.
Appx.~\ref{supp:meth-flownoinv} explains this problem in detail and shows our solution.

\vspace{-1pt}
\subsection{Fitting Data} \label{sec:meth-data}
\vspace{-1pt}

After achieving compatibility, Cor.~\ref{cor:determ-ac-full-supp} guarantees the a.e.-fully supported conditional models uniquely determine a joint, hence a data distribution $p_{\theta,\phi}(x)$.
To fit $p_{\theta,\phi}(x)$ to the true data distribution $p^*(x)$, an explicit expression is required.
For this, \eqref{eqn:def-suff-pi-unroll} is not helpful as we do not have explicit expressions of $a(x)$, $b(z)$. 
But when compatibility is given, we can safely use density function formulae: 
{\abovedisplayskip=2pt
\begin{align} \label{eqn:px-model}
  \textstyle
  p_{\theta,\phi}(x) = 1 / \frac{1}{p_{\theta,\phi}(x)}
  = 1 / \int_\bbZ \frac{p_{\theta,\phi}(z')}{p_{\theta,\phi}(x)} \zeta(\ud z')
  = 1 / \int_\bbZ \frac{q_\phi(z'|x)}{p_\theta(x|z')} \zeta(\ud z')
  = 1 / \bbE_{q_\phi(z'|x)} [1 / p_\theta(x|z')],
\end{align} }%
which is an explicit expression in terms of the two conditionals.
Although other expressions are possible, this one has a simple form, and the Monte-Carlo expectation estimation in $\bbZ$ has a lower variance than in $\bbX$ since usually $d_\bbZ \ll d_\bbX$. 
We can thus fit data by maximum likelihood estimation:
{\abovedisplayskip=4pt
\belowdisplayskip=2pt
\begin{align} \label{eqn:mle}
  (\min_{\theta,\phi}) \;\;
  \bbE_{p^*(x)} [-\log p_{\theta,\phi}(x)] = \bbE_{p^*(x)} [ \log \bbE_{q_\phi(z'|x)} [1 / p_\theta(x|z')] ].
\end{align} }%
The loss function can be estimated using the reparameterization trick~\citep{kingma2014auto} to reduce variance, and the \texttt{logsumexp} trick is adopted for numerical stability.
This expression can also serve for data likelihood evaluation.
The final training process of \ourmodel is the joint optimization with the compatibility loss. 

\textbf{Comparison with DAE.} \hspace{4pt}
%
We note that the DAE loss~\citep{vincent2008extracting, bengio2013generalized} $\bbE_{p^*(x) q_\phi(z'|x)} [-\log p_\theta(x|z')]$ is a \emph{lower bound} of \eqref{eqn:mle} due to Jensen's inequality,
so it is not suitable for maximizing likelihood.
In fact, the DAE loss minimizes $\bbE_{q_\phi(z)} \KL(q_\phi(x|z) \Vert p_\theta(x|z))$ for $p_\theta(x|z)$ to match $q_\phi(x|z)$, where $q_\phi(z)$ and $q_\phi(x|z)$ are induced from the joint $p^*(x) q_\phi(z|x)$,
but it is not a proper loss for $q_\phi(z|x)$ as a \emph{mode-collapse} behavior is promoted:
the optimal $q_\phi(z|x)$ only concentrates on the point(s) of $\argmin_{z'} p_\theta(x|z')$, 
and an additional entropy term $-\bbE_{q_\phi(z)} \bbH[q_\phi(x|z)]$ is required to optimize the same $\KL$ loss.
This behavior \emph{hurts determinacy}, as $q_\phi(z|x)$ tends to be a (mixture of) Dirac measure (Sec.~\ref{sec:determ-dirac}).
The resulting Gibbs chain may also converge differently 
depending on initialization, as ergodicity is broken.
This behavior also hurts compatibility, as $q_\phi(x|z)$ deviates from $p_\theta(x|z)$ (not Dirac), and does not match the Gibbs stationary distribution~\citep{heckerman2000dependency, bengio2013generalized}.
In contrast, \ourmodel follows a more fundamental logic: enforce compatibility explicitly and follow the maximum likelihood principle faithfully.
It leads to a proper loss for both conditionals that does not hinder determinacy. 

\vspace{-1pt}
\subsection{Data Generation} \label{sec:meth-gen}
\vspace{-1pt}

Generating samples from the learned data distribution $p_{\theta,\phi}(x)$ 
is not as straightforward as typical models that specify a prior, 
since ancestral sampling is not available.
But it is still tractable via Markov chain Monte Carlo methods (MCMCs).
%
%
We propose using \emph{dynamics-based MCMCs}, which are often more efficient than Gibbs sampling (used in DAE~\citep{bengio2013generalized} and GibbsNet~\citep{lamb2017gibbsnet}). 
They only require an \emph{unnormalized} density function of the target distribution, which is readily available in \ourmodel when compatible:
$p_{\theta,\phi}(x) = \frac{p_{\theta,\phi}(x)}{p_{\theta,\phi}(z)} p_{\theta,\phi}(z)
= \frac{p_\theta(x|z)}{q_\phi(z|x)} p_{\theta,\phi}(z) \propto \frac{p_\theta(x|z)}{q_\phi(z|x)}$
for any $z \in \bbZ$.
In practice, this $z$ can be taken as a sample from $q_\phi(z|x)$ to lie in a high probability region for a confident estimate.

Stochastic gradient Langevin dynamics (SGLD)~\citep{welling2011bayesian} is a representative instance, which has been shown to produce complicated realistic samples in energy-based~\citep{du2019implicit}, score-based~\citep{song2019generative} and diffusion-based~\citep{ho2020denoising, song2021score} models.
It gives the following transition:
{\abovedisplayskip=2pt
\begin{align} \label{eqn:langevin-px}
  \textstyle \!
  x^{(t+1)} = x^{(t)} + \varepsilon \nabla_{x^{(t)}} \log \frac{p_\theta(x^{(t)}|z^{(t)})}{q_\phi(z^{(t)}|x^{(t)})} + \sqrt{2 \varepsilon} \, \eta^{(t)}_x, \text{where }
  z^{(t)} \sim q_\phi(z|x^{(t)}),
  \eta^{(t)}_x \sim \clN(0,I_{d_\bbX}), \!
\end{align} }%
and $\varepsilon$ is a step size parameter.
Method to draw $z \sim p_{\theta,\phi}(z)$ can be developed symmetrically (see Appx. \eqref{eqn:langevin-pz}).
Also applicable are other dynamics-based MCMCs~\citep{chen2014stochastic, ding2014bayesian, ma2021there}, and particle-based variational inference methods~\citep{liu2016stein, chen2018unified, liu2019understanding_a, liu2019understanding_b, taghvaei2019accelerated} which are more sample-efficient.

\vspace{-2pt}
\section{Experiments} \label{sec:expm}
\vspace{-2pt}

\newcommand{\daefigs}{pinwheel-dae-pvar1e-02/res-PTelbo1000-ginit0_01}
\newcommand{\vaefigs}{pinwheel-elbo/figs-pvar1e-2-2}
\newcommand{\biganfigs}{pinwheel-bigan-pvar1e-02/js_07}
\newcommand{\ourfigspt}{figs-nllhmarg-jnlitex-PTelbo1000-wcm1e-5-pvar1e-2-1}
\newcommand{\ourfigsptnocompt}{figs-nllhmarg-jnlitex-PTelbo1000-wcm0-pvar1e-2-1}
\newcommand{\ourfigsnopt}{figs-nllhmarg-jnlitex-wcm1e-5-pvar1e-2-1}
\newcommand{\elemfigwidth}{1.74cm}
\newcommand{\elemfigwidthamend}{0.12cm}

We demonstrate the power of \ourmodel for data generation and representation learning. 
Baselines include DAE, and generative models using Gaussian prior \eg VAE and BiGAN (Appx.~\ref{supp:expm-baselines}).
For a fair comparison, all methods use the same architecture, which is an additive Gaussian $p_\theta(x|z)$ and a Sylvester flow (Householder version)~\citep{van2018sylvester} for $q_\phi(z|x)$ (Appx.~\ref{supp:expm-arch}), as required by \ourmodel (Sec.~\ref{sec:meth-compt}).
It is necessarily probabilistic for determinacy, so we exclude flow-based generative models and common BiGAN/GibbsNet architectures, which are deterministic.
We also considered GibbsNet~\citep{lamb2017gibbsnet} which also aims at the prior issue, but it does not produce reasonable results using the same architecture, due to its unstable training process (see Appx.~\ref{supp:expm-baselines}).
Codes: {\small \url{https://github.com/changliu00/cygen}}.

\vspace{-2pt}
\subsection{Synthetic Experiments} \label{sec:expm-synth}
\vspace{-2pt}

\begin{SCtable}[][t]
  \vspace{-4pt}
  \centering
  \small
  \setlength{\tabcolsep}{0.6pt}
  \begin{tabular}{cccc@{\hspace{1.4pt}}|@{\hspace{1pt}}cc}
    data & DAE & VAE & BiGAN & \textbf{\ourmodel} & \textbf{\ourmodelpt} \\
    \begin{minipage}{\elemfigwidth}
      \centering
      \raisebox{-.5\height}{\includegraphics[width=\elemfigwidth]{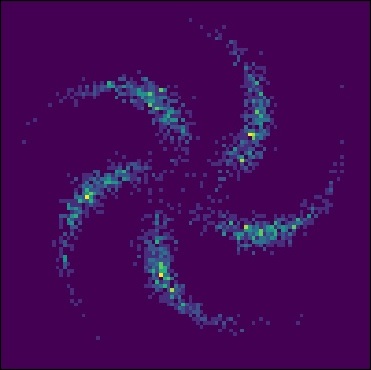}}
    \end{minipage}
    &
    \begin{minipage}{\elemfigwidth}
      \centering
      \raisebox{-.5\height}{\includegraphics[width=\elemfigwidth]{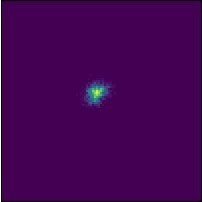}}
    \end{minipage}
    &
    \begin{minipage}{\elemfigwidth}
      \centering
      \raisebox{-.5\height}{\includegraphics[width=\elemfigwidth]{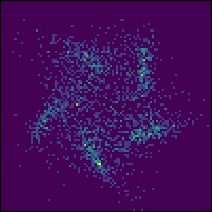}}
    \end{minipage}
    &
    \begin{minipage}{\elemfigwidth}
      \centering
      \raisebox{-.5\height}{\includegraphics[width=\elemfigwidth]{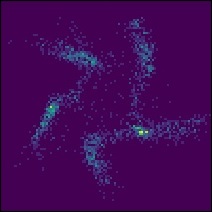}}
    \end{minipage}
    &
    \begin{minipage}{\elemfigwidth}
      \centering
      \raisebox{-.5\height}{\includegraphics[width=\elemfigwidth]{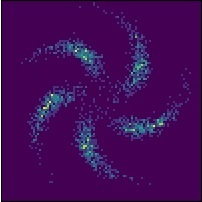}}
    \end{minipage}
    &
    \begin{minipage}{\elemfigwidth}
      \centering
      \raisebox{-.5\height}{\includegraphics[width=\elemfigwidth]{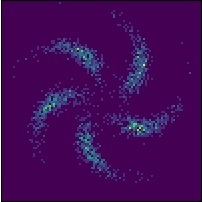}}
    \end{minipage}
    \\ \addlinespace[3pt]
    \makecell[r]{class-wise \\ aggregated \\ posterior} &
    \begin{minipage}{\elemfigwidth}
      \centering
      \raisebox{-.5\height}{\includegraphics[width=\elemfigwidth+\elemfigwidthamend]{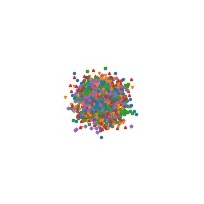}}
    \end{minipage}
    &
    \begin{minipage}{\elemfigwidth}
      \centering
      \raisebox{-.5\height}{\includegraphics[width=\elemfigwidth+\elemfigwidthamend]{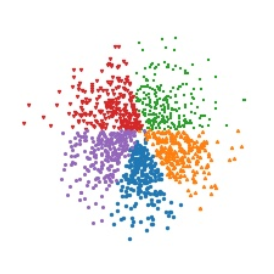}}
    \end{minipage}
    &
    \begin{minipage}{\elemfigwidth}
      \centering
      \raisebox{-.5\height}{\includegraphics[width=\elemfigwidth+\elemfigwidthamend]{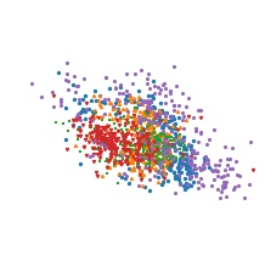}}
    \end{minipage}
    &
    \begin{minipage}{\elemfigwidth}
      \centering
      \raisebox{-.5\height}{\includegraphics[width=\elemfigwidth+\elemfigwidthamend]{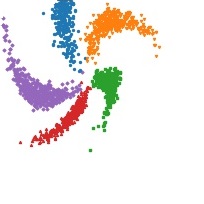}}
    \end{minipage}
    &
    \begin{minipage}{\elemfigwidth}
      \centering
      \raisebox{-.5\height}{\includegraphics[width=\elemfigwidth+\elemfigwidthamend]{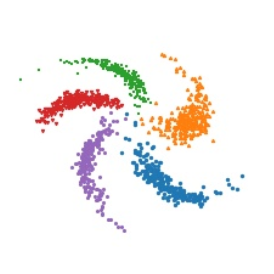}}
    \end{minipage}
  \end{tabular}
  \captionsetup[table]{name=Figure,font=small,skip=-2pt}
  \captionof{table}{Generated data (DAE and \ourmodel use SGLD) and class-wise aggregated posteriors of DAE, VAE, BiGAN and \ourmodel.
    Also shows results of \ourmodelpt that is PreTrained as a VAE.
    (Best view in color.)
  }
  \label{tab:pinw-gen-post}
\end{SCtable}

\begin{SCtable}[][t]
  \centering
  \small
  \setlength{\tabcolsep}{0.6pt}
  \begin{tabular}{r@{\hspace{4pt}}cccc@{\hspace{1.4pt}}|@{\hspace{1pt}}c}
    iteration & 1100 & 1200 & 1300 & 1400 & 30000
    \\
    \makecell[r]{\textbf{\ourmodelpt}, \\ using SGLD} &
    \begin{minipage}{\elemfigwidth}
      \centering
      \raisebox{-.5\height}{\includegraphics[width=\elemfigwidth]{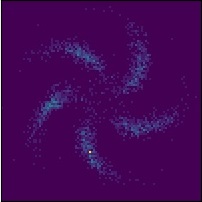}}
    \end{minipage}
    &
    \begin{minipage}{\elemfigwidth}
      \centering
      \raisebox{-.5\height}{\includegraphics[width=\elemfigwidth]{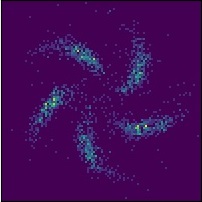}}
    \end{minipage}
    &
    \begin{minipage}{\elemfigwidth}
      \centering
      \raisebox{-.5\height}{\includegraphics[width=\elemfigwidth]{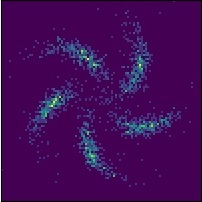}}
    \end{minipage}
    &
    \begin{minipage}{\elemfigwidth}
      \centering
      \raisebox{-.5\height}{\includegraphics[width=\elemfigwidth]{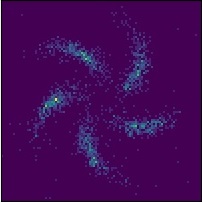}}
    \end{minipage}
    &
    \begin{minipage}{\elemfigwidth}
      \centering
      \raisebox{-.5\height}{\includegraphics[width=\elemfigwidth]{synth/\ourfigspt/xhist-langv-z-30000.jpg}}
    \end{minipage}
    \\ \addlinespace[1pt]
    \makecell[r]{\textbf{\ourmodelpt}, \\ using Gibbs} &
    \begin{minipage}{\elemfigwidth}
      \centering
      \raisebox{-.5\height}{\includegraphics[width=\elemfigwidth]{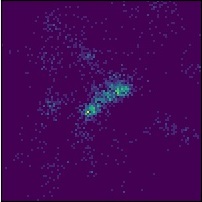}}
    \end{minipage}
    &
    \begin{minipage}{\elemfigwidth}
      \centering
      \raisebox{-.5\height}{\includegraphics[width=\elemfigwidth]{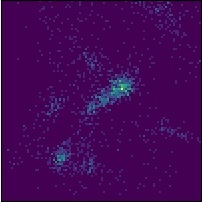}}
    \end{minipage}
    &
    \begin{minipage}{\elemfigwidth}
      \centering
      \raisebox{-.5\height}{\includegraphics[width=\elemfigwidth]{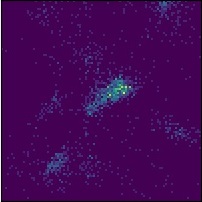}}
    \end{minipage}
    &
    \begin{minipage}{\elemfigwidth}
      \centering
      \raisebox{-.5\height}{\includegraphics[width=\elemfigwidth]{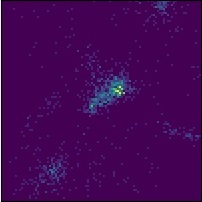}}
    \end{minipage}
    &
    \begin{minipage}{\elemfigwidth}
      \centering
      \raisebox{-.5\height}{\includegraphics[width=\elemfigwidth]{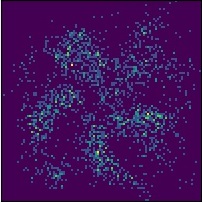}}
    \end{minipage}
    \\ \addlinespace[2pt]
    compt. loss &
    $7.0\e{3}$ & $5.4\e{3}$ & $7.7\e{3}$ & $6.2\e{3}$ & $4.6\e{3}$
    \\[-2pt] \midrule \addlinespace[2pt]
    \makecell[r]{\textbf{\ourmodelpt} \\ w/o compt., \\ using SGLD} &
    \begin{minipage}{\elemfigwidth}
      \centering
      \raisebox{-.5\height}{\includegraphics[width=\elemfigwidth]{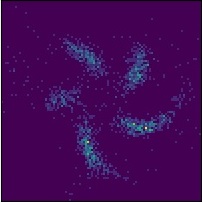}}
    \end{minipage}
    &
    \begin{minipage}{\elemfigwidth}
      \centering
      \raisebox{-.5\height}{\includegraphics[width=\elemfigwidth]{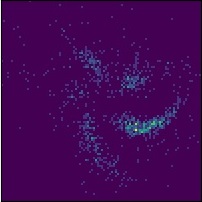}}
    \end{minipage}
    &
    \begin{minipage}{\elemfigwidth}
      \centering
      \raisebox{-.5\height}{\includegraphics[width=\elemfigwidth]{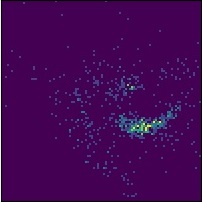}}
    \end{minipage}
    &
    \begin{minipage}{\elemfigwidth}
      \centering
      \raisebox{-.5\height}{\includegraphics[width=\elemfigwidth]{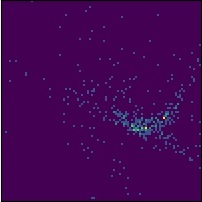}}
    \end{minipage}
    &
    \begin{minipage}{\elemfigwidth}
      \centering
      \raisebox{-.5\height}{\includegraphics[width=\elemfigwidth]{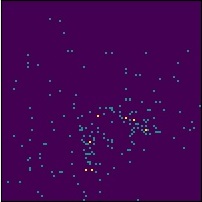}}
    \end{minipage}
    \\ \addlinespace[2pt]
    compt. loss &
    $1.1\e{5}$ & $1.6\e{5}$ & $2.6\e{5}$ & $8.6\e{5}$ & $1.2\e{8}$
  \end{tabular}
  \hspace{-5pt}
  \captionsetup[table]{name=Figure,font=small,skip=10pt,width=76pt}
  \captionof{table}{Generated data along the training process of \ourmodel after VAE pretraining (iteration 1000),
    using SGLD (rows~1,3) and Gibbs sampling (row~2) for generation, and with (rows~1,2) and without the compatibility loss (row~3) for training.
    See also Appx. Fig.~\ref{tab:pinw-compt-dae}.
  }
  \label{tab:pinw-gibbs-compt}
  \vspace{-8pt}
\end{SCtable}

For visual verification of the claims, we first consider a 2D toy dataset (Fig.~\ref{tab:pinw-gen-post} top-left). 
Appx.~\ref{supp:expm-synth} shows more details and results, including the investigation on another similar dataset.

\textbf{Data generation.} 
The learned data distributions (as the histogram of generated data samples) are shown in Fig.~\ref{tab:pinw-gen-post} (row~1).
We see the five clusters are blurred to overlap in VAE's distribution and are still connected in BiGAN's, due to the specified prior.
In contrast, our \ourmodel fits this distribution much better; particularly it clearly separates the five non-connected clusters. 
This verifies the advantage to overcome the \emph{manifold mismatch} problem. 
As for DAE, 
it cannot capture the data distribution due to collapsed inference model and insufficient determinacy (Sec.~\ref{sec:meth-data}).

\textbf{Representation.} 
Class-wise aggregated posteriors (as the scatter plot of $z$ samples from $q_\phi(z|x) p^*(x|y)$ for each class/cluster $y$) in Fig.~\ref{tab:pinw-gen-post} (row~2)
show that \ourmodel mitigates the \emph{posterior collapse} problem,
as the learned inference model $q_\phi(z|x)$ better separates the classes with a margin in the latent space. 
This more informative and representative feature would benefit downstream tasks like classification or clustering in the latent space.
In contrast, the specified Gaussian prior squeezes the VAE latent clusters to touch, and the BiGAN latent clusters even to mix.
The mode-collapsed inference model of DAE locates all latent clusters in the same place.

\textbf{Incorporating knowledge into conditionals.} 
\ourmodel alone (without pretraining) already performs well. 
When knowledge is available, we can further incorporate it into the conditional models. 
Fig.~\ref{tab:pinw-gen-post} shows pretraining \ourmodel's likelihood model as in a VAE (\ourmodelpt) embodies VAE's knowledge that the prior is centered and centrosymmetric, as the (all-class) aggregated posterior ($\approx$ prior) is such.
Note its data generation quality is not sacrificed.
Appx. Fig.~\ref{tab:pinw-prior} verifies this directly via the priors.

\textbf{Comparison of data generation methods.}
We then make more analysis on \ourmodel.
Fig.~\ref{tab:pinw-gibbs-compt} (rows~1,2) shows generated data of \ourmodel using SGLD and Gibbs sampling.
We see SGLD better recovers the true distribution, and is more robust to slight incompatibility. 

\textbf{Impact of the compatibility loss.} 
Fig.~\ref{tab:pinw-gibbs-compt} (rows~1,3) also shows the comparison with training \ourmodel without the compatibility loss. 
We see the compatibility is then indeed out of control, 
which invalidates the likelihood estimation \eqref{eqn:mle} for fitting data and the gradient estimation in \eqref{eqn:langevin-px} for data generation,
leading to the failure in row~3.
Along the training process of the normal \ourmodel, we also find a smaller compatibility loss makes better generation (esp. using Gibbs sampling).

\vspace{-2pt}
\subsection{Real-World Experiments} \label{sec:expm-real}
\vspace{-2pt}

\begin{wraptable}{r}{9.1cm}
  \vspace{-27pt}
  \centering
  \setlength{\tabcolsep}{1.6pt}
  \begin{tabular}{cc|c}
    DAE & VAE & \textbf{\ourmodelpt} \\
    \includegraphics[width=.21\textwidth]{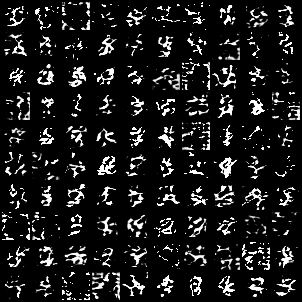} &
    \includegraphics[width=.21\textwidth]{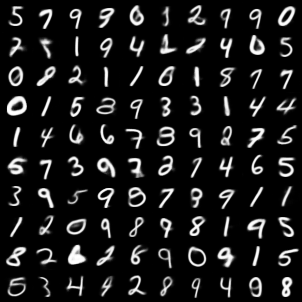} &
    \includegraphics[width=.21\textwidth]{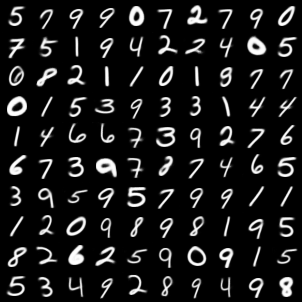}
    \\ \addlinespace[-1pt]
    \includegraphics[width=.21\textwidth]{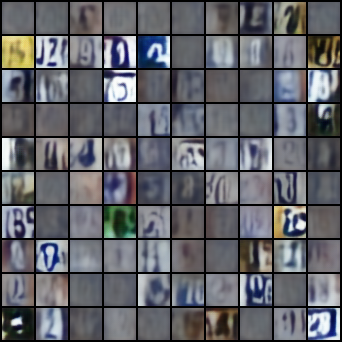} &
    \includegraphics[width=.21\textwidth]{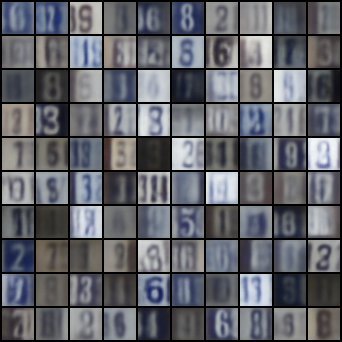} &
    \includegraphics[width=.21\textwidth]{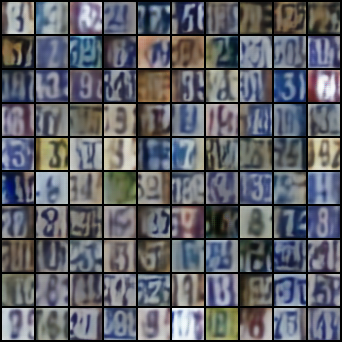}
  \end{tabular}
  \vspace{-6pt}
  \captionsetup[table]{name=Figure}
  \captionof{table}{Generated data on the MNIST and SVHN datasets.}
  \label{tab:realdata}
  \vspace{-12pt}
\end{wraptable}

We test the performance of \ourmodel on real-world image datasets MNIST and SVHN. 
We consider the VAE-pretrained version, \ourmodelpt, for more stable training. 
Appx.~\ref{supp:expm-real} shows more details.
On these datasets, even BiGAN cannot produce reasonable results using the same architecture, similar to GibbsNet. 

\textbf{Data generation.} 
From Fig.~\ref{tab:realdata}, 
We see that \ourmodelpt generates both sharp and diverse samples, as a sign to mitigate \emph{manifold mismatch}.
DAE samples are mostly imperceptible, due to the mode-collapsed $q_\phi(z|x)$ and the subsequent lack of determinacy (Sec.~\ref{sec:meth-data}). \vspace{-\parskip}

\begin{wraptable}{r}{7.82cm}
  \vspace{-24pt}
  \centering
  \caption{Downstream classification accuracy (\%) using learned representation by various models. \\
    $^\dagger$: Results from~\citep{lamb2017gibbsnet} using a different, deterministic architecture (not suitable for \ourmodel).
  }
  \vspace{-4pt}
  \setlength{\tabcolsep}{1.6pt}
  \begin{tabular}{ccc@{\hspace{0pt}}c@{\hspace{1pt}}c@{\hspace{1pt}}|@{\hspace{1pt}}c}
    \toprule
    & DAE & VAE & BiGAN$^\dagger$ & GibbsNet$^\dagger$ & \textbf{{\small{\ourmodel}}{}{\scriptsize{(PT)}}} \\
    \midrule
    MNIST & 98.0\subpm{0.1} & 94.5\subpm{0.3} & 91.0 & 97.7 & \textbf{98.3\subpm{0.1}} \\
    SVHN  & 74.5\subpm{1.0} & 30.8\subpm{0.2} & 66.7 & \textbf{79.6} & \textbf{75.8\subpm{0.5}} \\
    \bottomrule
  \end{tabular}
  \vspace{-10pt}
  \label{tab:downstream}
\end{wraptable}
VAE samples are a little blurry as a typical behavior due to the simply-connected prior.
This observation is also quantitatively supported by the FID score~\citep{heusel2017gans, seitzer2020fid} on SVHN: \ourmodel achieves 102, while DAE 157 and VAE 128 (lower is better).

\textbf{Representation.} 
We then show in Table~\ref{tab:downstream} that \ourmodelpt's latent representation is more informative for the downstream classification task, as an indicator to avoid \emph{posterior collapse}.
BiGAN and GibbsNet make random guess using the same probabilistic flow architecture, and their reported results in~\citep{lamb2017gibbsnet} using a different, deterministic architecture (not suitable for \ourmodel due to insufficient determinacy) are still not always better, due to the prior constraint.
We conclude that \ourmodel achieves both superior generation and representation learning performance. 

\vspace{-3pt}
\section{Conclusions and Discussions} \label{sec:conc}
\vspace{-3pt}

In this work we investigate the possibility of defining a joint distribution using two conditional distributions, under the motivation for generative modeling without an explicit prior.
We develop a systematic theory with novel and operable equivalence criteria for compatibility and sufficient conditions for determinacy,
and propose a novel generative modeling framework \ourmodel that only uses cyclic conditional models.
Methods for achieving compatibility and determinacy, fitting data and data generation are developed.
Experiments show the benefits of \ourmodel over DAE and prevailing generative models that specify a prior in overcoming manifold mismatch and posterior collapse. 

The novel \ourmodel framework broadens the starting point to build a generative model, and the general theory could also foster a deeper understanding of other machine learning paradigms, \eg, dual learning and self-supervised learning, and inspire more efficient algorithms.



\bibliographystyle{abbrvnat}
\bibliography{dualdgm}

\onecolumn
\appendix
\setlength{\parskip}{\parskip+1pt}
\belowdisplayskip=2pt
\belowdisplayshortskip=2pt

\newcommand{\bfthrone}{\textbf{(3.1)}\xspace}
\newcommand{\bfforone}{\textbf{(4.1)}\xspace}
\newcommand{\bffortwo}{\textbf{(4.2)}\xspace}
\newcommand{\bffivone}{\textbf{(5.1)}\xspace}
\newcommand{\bffivtwo}{\textbf{(5.2)}\xspace}
\newcommand{\bffivthr}{\textbf{(5.3)}\xspace}
\newcommand{\bfsixone}{\textbf{(6.1)}\xspace}
\newcommand{\bfsixtwo}{\textbf{(6.2)}\xspace}
\newcommand{\bfsixthr}{\textbf{(6.3)}\xspace}

\section*{Supplementary Materials}

\section{Background in Measure Theory} \label{supp:meas}

\subsection{The Integral} \label{supp:meas-int}

The \emph{integral} of a nonnegative measurable function $f$ on a measure space $(\Omega, \scF, \mu)$ is defined as:
\begin{align} \label{eqn:def-int}
  \int f \dd \mu := \sup \sum_i \mu(\clW^{(i)}) \inf_{\omega \in \clW^{(i)}} f(\omega),
\end{align}
where the supremum is taken over all finite decompositions $\{\clW^{(i)}\}$ of $\Omega$ into $\scF$-sets~\citep[p.211]{billingsley2012probability}.
For a general measurable function, its integral is defined as the subtraction from the integral of its positive part $f^+(\omega) := \max\{0, f(\omega)\}$ with the integral of its negative part $f^-(\omega) := \max\{0, -f(\omega)\}$.
A measurable function is said to be \emph{$\mu$-integrable}~\citep[p.212]{billingsley2012probability} if both integrals of its positive and negative parts are finite.

\bfi This is a general definition of integral.
When $\Omega$ is an Euclidean space and $\mu$ is the Lebesgue measure on it, this integral reduces to the Lebesgue integral (which in turn coincides with the Riemann integral when the latter exists).
When $\Omega$ is a discrete set (\ie, a finite or countable set) and $\mu$ is the counting measure, this integral reduces to summation.

\bfii The integral satisfies common properties like linearity and monotonicity~\citep[Thm.~16.1]{billingsley2012probability}, continuity under boundedness~\citep[Thm.~16.4, Thm.~16.5]{billingsley2012probability}, \etc
For a nonnegative function $f$, $\int f \dd \mu = 0$ if and only if $f = 0$, $\mu$-a.e.~\citep[Thm.~15.2]{billingsley2012probability}.

\bfiii The integral over a set $\clW \in \scF$ is defined as $\int_\clW f \dd \mu := \int \bbI_\clW f \dd \mu$~\citep[p.226]{billingsley2012probability}, where $\bbI_\clW$ is the indicator function.
\\ \bfone We thus sometimes also write $\int_\Omega f \dd \mu$ for $\int f \dd \mu$ to highlight the integral area.
By this definition, $\int_\clW f \dd \mu = 0$ if $\mu(\clW) = 0$~\citep[p.226]{billingsley2012probability}.
\\ \bftwo For two measurable functions $f$ and $g$, if $f = g$, $\mu$-a.e., then $\int_\clW f \dd \mu = \int_\clW g \dd \mu$ for any $\clW \in \scF$~\citep[Thm.~15.2]{billingsley2012probability}.
The inverse also holds if $f$ and $g$ are nonnegative and $\mu$ is sigma-finite, or $f$ and $g$ are integrable~\citep[Thm.~16.10(i,ii)]{billingsley2012probability}\footnote{
  \citealp[Thm.~16.10(iii)]{billingsley2012probability}: $f = g$, $\mu$-a.e., if $\int_\clW f \dd \mu = \int_\clW g \dd \mu$ for any $\clW$ from a pi-system $\Pi$ that generates $\scF$, and $\Omega$ is a finite or countable union of $\Pi$-sets.
}.
\\ \bfthr If $f$ is a nonnegative measurable function, then $\nu(\clW) := \int_\clW f \dd \mu$, $\forall \clW \in \scF$, is a measure on $(\Omega, \scF)$~\citep[p.227]{billingsley2012probability}\footnote{
  Its countable additivity is guaranteed by \citet[Thm.~16.9]{billingsley2012probability}.
}. Such a measure $\nu$ is finite, if and only if $f$ is $\mu$-integrable.

\subsection{Absolute Continuity and Radon-Nikodym Derivative} \label{supp:meas-ac}

For two measures $\mu$ and $\nu$ on the same measurable space $(\Omega, \scF)$, $\nu$ is said to be \emph{absolutely continuous} w.r.t $\mu$, denoted as $\nu \ll \mu$, if $\mu(\clW) = 0$ indicates $\nu(\clW) = 0$ for $\clW \in \scF$~\citep[p.448]{billingsley2012probability}.
If $\mu$ and $\nu$ are sigma-finite and $\nu \ll \mu$, the \emph{Radon-Nikodym theorem}~\citep[Thm.~32.2]{billingsley2012probability} asserts that there exists a $\mu$-unique nonnegative function $f$ on $\Omega$,
such that $\nu(\clW) = \int_\clW f(\omega) \mu(\ud\omega)$ for any $\clW \in \scF$.
Such a function $f$ is called the \emph{Radon-Nikodym (R-N) derivative} of $\nu$ w.r.t $\mu$, and is also denoted as $\frac{\ud \nu}{\ud \mu}$.
It represents the density function of $\nu$ w.r.t base measure $\mu$.

\bfi Since the general definition of integral includes summation in the discrete case, this density function also includes the probability mass function in the discrete case.

\bfii The Dirac measure $\delta_{\omega_0}(\clW) := \bbI_\clW(\omega_0)$ ($\bbI$ is the indicator function) at a single point $\omega_0 \in \Omega$ is not absolutely continuous on Euclidean spaces w.r.t the Lebesgue measure, which assigns measure 0 to the set $\{\omega_0\}$.
To be strict, the Dirac delta function is not a proper density function, since its integrals covering $\omega_0$ involve the indefinite $\infty \cdot 0$ on the component $\{\omega_0\}$ of the integral domain.
Its characteristic that such integrals equal to one, is a standalone structure from being a function.
So it is better treated as a measure of functional.

\subsection{Product Measure Space} \label{supp:meas-prod}

Two measure spaces $(\bbX, \scX, \xi)$ and $(\bbZ, \scZ, \zeta)$ induce a \emph{product measure space} $(\bbX\times\bbZ, \scX\otimes\scZ, \xi\otimes\zeta)$.

\bfi The \emph{product sigma-field} $\scX\otimes\scZ := \sigma(\scX\times\scZ)$ is the smallest sigma-field on $\bbX\times\bbZ$ containing $\scX\times\scZ$ (\citealp[Thm.~22]{galambos1995advanced}; \citealp[Def.~7.1]{rinaldo2018advanced}; equivalently, \citealp[Remark~14.10]{klenke2013probability}; \citealp[Def.~14.4]{klenke2013probability}).
Note that the Cartesian product $\scX\times\scZ$, representing the set of \emph{measurable rectangles}, is only a semiring (thus also a pi-system).
So we need to extend for a sigma-field.
For any $\clW \in \scX\otimes\scZ$, its \emph{slice} (or section) at $z \in \bbZ$, defined by:
\begin{align} \label{eqn:def-slice}
  \clW_z := \{x \mid (x,z) \in \clW\},
\end{align}
lies in $\scX$, and similarly $\clW_x \in \scZ$~\citep[Thm.~18.1(i)]{billingsley2012probability}.
We define the \emph{projection} (or restriction) of $\clW$ onto $\bbZ$, as $\clW^\bbZ := \{z \mid \exists x \in \bbX \st (x,z) \in \clW\}$.
By definition, for any $z \in \bbZ \setminus \clW^\bbZ$, $\clW_z = \emptyset$.

\bfii The \emph{product measure} $\xi\otimes\zeta$ is characterized by $(\xi\otimes\zeta) (\clX\times\clZ) = \xi(\clX) \zeta(\clZ)$ for measurable rectangles $\clX\times\clZ \in \scX\times\scZ$.
Some common conclusions require $\xi$ and $\zeta$ to be sigma-finite on $\scX$ and $\scZ$, respectively.
\\ \bfone In the characterization $(\xi\otimes\zeta) (\clX\times\clZ) = \xi(\clX) \zeta(\clZ)$, if the indefinite $0 \cdot \infty$ is met, it is zero.
To see this, consider two sets $\clX$ and $\clZ$ that satisfy $\xi(\clX) = 0$ and $\zeta(\clZ) = 0$.
Since $\zeta$ is sigma-finite, there are finite or countable disjoint $\scZ$-sets $\clZ^{(1)}, \clZ^{(2)}, \cdots$ such that $\zeta(\clZ^{(i)}) < \infty$ for any $i \ge 1$ and $\bigcup_{i=1}^\infty \clZ^{(i)} = \bbZ$.
Redefining $\clZ^{(i)}$ as $\clZ^{(i)} \cap \clZ$, we have $\bigcup_{i=1}^\infty \clZ^{(i)} = \clZ$ while still $\zeta(\clZ^{(i)}) < \infty$.
So $(\xi\otimes\zeta) (\clX \times \bbZ) = (\xi\otimes\zeta) (\clX \times \bigcup_{i=1}^\infty \clZ^{(i)}) = (\xi\otimes\zeta) (\bigcup_{i=1}^\infty \clX \times \clZ^{(i)})$.
Recalling that a measure is countably additive by definition, this is
$= \sum_{i=1}^\infty (\xi\otimes\zeta) (\clX \times \clZ^{(i)}) = \sum_{i=1}^\infty \xi(\clX) \zeta(\clZ^{(i)}) = 0$.
\\ \bftwo In this case, such a $\xi\otimes\zeta$ is sigma-finite on $\scX\times\scZ$,
and the characterization on the pi-system $\scX\times\scZ$ determines a unique sigma-finite measure on $\sigma(\scX\times\scZ) = \scX\otimes\scZ$~\citep[Thm.~10.3]{billingsley2012probability}.
See also \citet[Thm.~22]{galambos1995advanced}; \citet[Thm.~14.14]{klenke2013probability}; \citet[Thm.~7.9]{rinaldo2018advanced}.
Moreover, we have~\citep[Thm.~18.2]{billingsley2012probability}:
\begin{align} \label{eqn:prod-meas}
  (\xi\otimes\zeta)(\clW) = \int_\bbZ \xi(\clW_z) \zeta(\ud z) = \int_\bbX \zeta(\clW_x) \xi(\ud x), \;
  \forall \clW \in \scX\otimes\scZ.
\end{align}
Since for any $z \in \bbZ \setminus \clW^\bbZ$, $\clW_z = \emptyset$ (see \bfi) thus $\xi(\clW_z) = 0$, we also have
(by leveraging the additivity of integrals over a countable partition~\citep[Thm.~16.9]{billingsley2012probability} and that an a.e. zero function gives a zero integral~\citep[Thm.~15.2(i)]{billingsley2012probability}):
\begin{align} \label{eqn:prod-meas-restr}
  (\xi\otimes\zeta)(\clW) = \int_{\clW^\bbZ} \xi(\clW_z) \zeta(\ud z) = \int_{\clW^\bbX} \zeta(\clW_x) \xi(\ud x), \;
  \forall \clW \in \scX\otimes\scZ.
\end{align}

\bfiii For a function $f$ on $\bbX\times\bbZ$, if it is $\scX\otimes\scZ$-measurable, then $f(x,\cdot)$ is $\scZ$-measurable for any $x \in \bbX$, and $f(\cdot,z)$ is $\scX$-measurable for any $z \in \bbZ$~\citep[Thm.~18.1(ii)]{billingsley2012probability}.
When $f$ is $\xi\otimes\zeta$-integrable, \emph{Fubini's theorem}~\citep[Thm.~18.3]{billingsley2012probability} asserts its integral on $\bbX\times\bbZ$ can be computed iteratedly in either order:
\begin{align} \label{eqn:fubini}
  \int_{\bbX\times\bbZ} f(x,z) (\xi\otimes\zeta)(\ud x \ud z)
  = \int_\bbZ \left( \int_\bbX f(x,z) \xi(\ud x) \right) \zeta(\ud z)
  = \int_\bbX \left( \int_\bbZ f(x,z) \zeta(\ud z) \right) \xi(\ud x).
\end{align}
For any $\clW \in \scX\otimes\scZ$, the same equalities hold for function $\bbI_\clW f$.
For the first iterated integral, we have $\int_\bbX \bbI_\clW(x,z) f(x,z) \xi(\ud x) = \int_\bbX \bbI_{\clW_z}(x) f(x,z) \xi(\ud x) = \int_{\clW_z} f(x,z) \xi(\ud x)$,
and on the region $\bbZ \setminus \clW^\bbZ$, the integral $\int_{\clW_z} f(x,z) \xi(\ud x) = 0$~\citep[p.226]{billingsley2012probability} since $\clW_z = \emptyset$ on that region (see \bfi).
So we have a more general form of Fubini's theorem:
\begin{align} \label{eqn:fubini-restr}
  \int_\clW f(x,z) (\xi\otimes\zeta)(\ud x \ud z)
  = \int_{\clW^\bbZ} \left( \int_{\clW_z} f(x,z) \xi(\ud x) \right) \zeta(\ud z)
  = \int_{\clW^\bbX} \left( \int_{\clW_x} f(x,z) \zeta(\ud z) \right) \xi(\ud x).
\end{align}

\bfiv For a measure $\pi$ on the product measurable space $(\bbX\times\bbZ, \scX\otimes\scZ)$, define its marginal distributions:
$\pi^\bbX(\clX) := \pi(\clX\times\bbZ), \forall \clX \in \scX$, and
$\pi^\bbZ(\clZ) := \pi(\bbX\times\clZ), \forall \clZ \in \scZ$.

\subsection{Conditional Distributions} \label{supp:meas-cond}
In the most general case, a distribution (probability measure) $\pi$ on a measurable space $(\Omega, \scF)$ gives a \emph{conditional distribution} (conditional probability) $\pi(\clW|\omega)$ for $\clW \in \scF$ w.r.t a sub-sigma-field $\scG \subseteq \scF$.

\bfi For any $\clW \in \scF$, the function $\scG \to \bbR^{\ge 0}, \clG \mapsto \pi(\clG \cap \clW)$ gives a measure on $\scG$.
It is absolutely continuous w.r.t $\pi^\scG: \scG \to \bbR^{\ge 0}, \clG \mapsto \pi(\clG)$, the projection of $\pi$ onto $\scG$, due to the monotonicity (or (sub-)additivity) of measures.
So the R-N derivative on $\scG$ exists, which defines the conditional distribution~\citep[p.457]{billingsley2012probability}:
\begin{align} \label{eqn:def-cond}
  \pi(\clW|\omega) := \frac{\ud \pi(\cdot \cap \clW)}{\ud \pi^\scG(\cdot)}(\omega),
\end{align}
where $\omega \in \Omega$.
Note that as defined as an R-N derivative, the conditional distribution is only $\pi^\scG$-unique.

\bfii As a function of $\omega$, $\pi(\clW|\omega)$ is $\scG$-measurable and $\pi$-integrable, and satisfies~\citep[p.457, Thm.~33.1]{billingsley2012probability}:
\begin{align} \label{eqn:cond-int-genrl}
  \int_\clG \pi(\clW|\omega) \pi^\scG(\ud \omega) = \pi(\clG \cap \clW), \forall \clG \in \scG.
\end{align}
This could serve as an alternative definition of conditional probability.

\bfiii For $\pi^\scG$-a.e. $\omega$, $\pi(\cdot|\omega)$ is a distribution (probability measure) on $(\Omega, \scF)$~\citep[Thm.~33.2]{billingsley2012probability}.

\bfiv Conditional distributions on a product measurable space $(\bbX\times\bbZ, \scX\otimes\scZ)$.
Consider the sub-sigma-field $\scG := \{\bbX\} \times \scZ$.
By construction, any $\clG \in \scG$ can be formed by $\clG = \bbX \times \clZ$ for some $\clZ \in \scZ$.
So $\pi^\scG(\clG) := \pi(\clG) = \pi(\bbX \times \clZ) =: \pi^\bbZ(\clZ)$, and \eqref{eqn:cond-int-genrl} becomes
$\pi\big( (\bbX\times\clZ) \cap \clW \big) = \int_{\bbX\times\clZ} \pi(\clW | x,z) \pi^\scG(\ud x \ud z)
= \int_{\bbX\times\clZ} \pi(\clW | x,z) \pi^\bbZ(\ud z) = \int_\clZ \pi(\clW | x,z) \pi^\bbZ(\ud z)$.
This indicates that the conditional probability $\pi(\clW | x,z)$ in this case is constant w.r.t $x$.
We hence denote it as $\pi(\clW|z)$.

Consider $\clW \in \scX\otimes\scZ$ in the form $\clW = \clX \times \bbZ$ for some $\clX \in \scX$.
For any $\clG = \bbX \times \clZ \in \scG$, we have from \eqref{eqn:cond-int-genrl} that
$\int_\clG \pi(\clW|z) \pi^\scG(\ud x \ud z) = \pi(\clG \cap \clW) = \pi(\clX\times\clZ)$.
From the above deduction,
the l.h.s is $\int_{\bbX \times \clZ} \pi(\clW|z) \pi^\scG(\ud x \ud z) 
= \int_\clZ \pi(\clX \times \bbZ | z) \pi^\bbZ(\ud z)$.
Defining $\pi(\clX|z)$ as $\pi(\clX \times \bbZ | z)$ for any $\clX \in \scX$, 
we have:
\begin{align} \label{eqn:cond-int-prod-indep}
  \pi(\clX\times\clZ) = \int_\clZ \pi(\clX|z) \pi^\bbZ(\ud z) = \int_\clX \pi(\clZ|x) \pi^\bbX(\ud x), \;
  \forall \clX\times\clZ \in \scX\times\scZ.
\end{align}
This is the conditional distribution in the usual sense.
Note again that as defined as R-N derivatives, the conditional distributions $\pi(\clX|z)$ and $\pi(\clZ|x)$ are only $\pi^\bbZ$-unique and $\pi^\bbX$-unique, respectively.

For any $\clW \in \scX\otimes\scZ$, define $\pit(\clW) := \int_\bbZ \pi(\clW_z|z) \pi^\bbZ(\ud z)$.
It is easy to verify that $\pit$ is a distribution (probability measure; thus finite and sigma-finite) on $(\bbX\times\bbZ, \scX\otimes\scZ)$~\citep[p.227]{billingsley2012probability},
since $\pi(\clX|z)$ is a distribution (and thus nonnegative) on $(\bbX, \scX)$ for $\pi^\bbZ$-a.e. $z$~\citep[Thm.~33.2]{billingsley2012probability}.
For any $\clW = \clX\times\clZ \in \scX\times\scZ$, $\pit(\clW) = \int_\bbZ \pi(\clX|z) \pi^\bbZ(\ud z) = \int_\clZ \pi(\clX|z) \pi^\bbZ(\ud z) = \pi(\clX\times\clZ)$ due to \eqref{eqn:cond-int-prod-indep}.
So $\pit$ and $\pi$ agree on the pi-system $\scX\times\scZ$, which indicates that they agree on $\sigma(\scX\times\scZ) = \scX\otimes\scZ$ due to \citet[Thm.~10.3, Thm.~3.3]{billingsley2012probability}.
This means that
(see the argument in \bfii\bftwo in Supplement~\ref{supp:meas-prod} for the second line of the equation):
\begin{align}
  \label{eqn:cond-int-prod-dep}
  \pi(\clW) ={} & \int_\bbZ \pi(\clW_z|z) \pi^\bbZ(\ud z) = \int_\bbX \pi(\clW_x|x) \pi^\bbX(\ud x) \\
  \label{eqn:cond-int-prod-dep-restr}
  ={} & \int_{\clW^\bbZ} \pi(\clW_z|z) \pi^\bbZ(\ud z) = \int_{\clW^\bbX} \pi(\clW_x|x) \pi^\bbX(\ud x), \;
  \forall \clW \in \scX\otimes\scZ.
\end{align}

Finally, we formalize some definitions in the main text below.
\begin{definition} \label{def:assame-assubseteq}
  Consider a general measure space $(\Omega, \scF, \mu)$.
  \bfi We say that two measurable sets $\clS, \clSt \in \scF$ are \emph{$\mu$-a.s. the same}, denoted as ``$\clS \aseq{\mu} \clSt$'',
  if $\mu(\clS \triangle \clSt) = 0$, where ``$\triangle$'' denotes the symmetric difference between two sets.
  \bfii We say that $\clS$ is a \emph{$\mu$-a.s. subset} of $\clW$, denoted as ``$\clS \assubseteq{\mu} \clW$'', if $\mu(\clS \setminus \clW) = 0$.
\end{definition}

\section{Lemmas} \label{supp:lemma}

\subsection{Lemmas for General Probability} \label{supp:lemma-general}

\begin{lemma} \label{lem:pmmeas0}
  Let $\clO$ be a measure-zero set, $\mu(\clO) = 0$, on a measure space $(\Omega, \scF, \mu)$.
  Then for any measurable set $\clW$, we have $\mu(\clO \setminus \clW) = \mu(\clW \cap \clO) = 0$, and $\mu(\clW \cup \clO) = \mu(\clW \setminus \clO) = \mu(\clW)$.
\end{lemma}
\begin{proof}
  Due to the monotonicity of a measure~\citep[Thm.~16.1]{billingsley2012probability}, we have
  $\mu(\clO \setminus \clW) \le \mu(\clO) = 0$ and $\mu(\clW \cap \clO) \le \mu(\clO) = 0$,
  so we get $\mu(\clO \setminus \clW) = \mu(\clW \cap \clO) = 0$.
  Since $\mu(\clW \cup \clO) = \mu(\clW \cup (\clO \setminus \clW))$ and the two sets are disjoint, it equals to
  $\mu(\clW) + \mu(\clO \setminus \clW)$, which is $\mu(\clW)$ by the above conclusion.
  So we get $\mu(\clW \cup \clO) = \mu(\clW)$.
  When applying this conclusion to $\clW \setminus \clO$, we have $\mu((\clW \setminus \clO) \cup \clO) = \mu(\clW \setminus \clO)$,
  while the l.h.s is $\mu(\clW \cup \clO)$ which is $\mu(\clW)$ by the same conclusion.
  So we get $\mu(\clW \setminus \clO) = \mu(\clW)$.
\end{proof}

\begin{lemma} \label{lem:support-int}
  Let $\pi$ be an absolutely continuous distribution (probability measure) on a measure space $(\Omega, \scF, \mu)$ with a density function $f$,
  and let $\clS \in \scF$ be a measurable set.
  Then $\pi(\clS) = 1$ if and only if $\pi(\clW) = \int_{\clW \cap \clS} f \dd \mu, \forall \clW \in \scF$.
\end{lemma}
\begin{proof}
  \textbf{``Only if'':}
  Since $\clS \subseteq \Omega$, we have $\pi(\Omega \setminus \clS) = \pi(\Omega) - \pi(\clS) = 0$.
  For any $\clW \in \scF$, we have $\pi(\clW) = \pi(\clW \cap \clS) + \pi(\clW \cap (\Omega \setminus \clS))$,
  while $0 \le \pi(\clW \cap (\Omega \setminus \clS)) \le \pi(\Omega \setminus \clS) = 0$.
  So we have $\pi(\clW) = \pi(\clW \cap \clS) = \int_{\clW \cap \clS} f \dd \mu$.

  \textbf{``If'':}
  $1 = \pi(\Omega) = \int_{\Omega \cap \clS} f \dd \mu = \int_\clS f \dd \mu = \int_{\clS \cap \clS} f \dd \mu = \pi(\clS)$.
\end{proof}

\begin{lemma} \label{lem:assame-meas}
  Let $\clS$ and $\clSt$ be two measurable sets on a measure space $(\Omega, \scF, \mu)$ such that $\clS \aseq{\mu} \clSt$.
  Then $\mu(\clS \setminus \clSt) = \mu(\clSt \setminus \clS) = 0$, and $\mu(\clS) = \mu(\clSt) = \mu(\clS \cup \clSt) = \mu(\clS \cap \clSt)$.
\end{lemma}
\begin{proof}
  Let $\clD^+ := \clSt \setminus \clS$ and $\clD^- := \clS \setminus \clSt$.
  By construction, we have $\clD^+ \cap \clS = \emptyset$ and $\clD^- \subseteq \clS$,
  so we also have $\clD^+ \cap \clD^- = \emptyset$, and $\clSt = (\clS \setminus \clD^-) \cup \clD^+ = (\clS \cup \clD^+) \setminus \clD^-$.
  By definition, $\clS \aseq{\mu} \clSt$ indicates $0 = \mu(\clS \triangle \clSt) = \mu(\clD^+ \cup \clD^-) = \mu(\clD^+) + \mu(\clD^-)$,
  so we have both $\mu(\clD^+) = 0$ and $\mu(\clD^-) = 0$.
  Subsequently, $\mu(\clSt) = \mu\big( (\clS \setminus \clD^-) \cup \clD^+ \big) = \mu(\clS \setminus \clD^-) + \mu(\clD^+) = \mu(\clS \setminus \clD^-)
  = \mu(\clS) - \mu(\clD^- \cap \clS) = \mu(\clS) - \mu(\clD^-) = \mu(\clS)$, and
  $\mu(\clS \cup \clSt) = \mu(\clS \cup \clD^+) = \mu(\clS) + \mu(\clD^+) = \mu(\clS)$.
  Noting also that $\clS \cup \clSt = (\clS \cap \clSt) \cup (\clS \triangle \clSt)$ and that this is a disjoint union,
  we have $\mu(\clS \cup \clSt) = \mu(\clS \cap \clSt) + \mu(\clS \triangle \clSt) = \mu(\clS \cap \clSt)$.
\end{proof}

\begin{lemma} \label{lem:assame-equiv}
  On a measure space $(\Omega, \scF, \mu)$, ``$\cdot \aseq{\mu} \cdot$'' is an equivalence relation.
\end{lemma}
\begin{proof}
  Symmetry and reflexivity are obvious.
  For transitivity, let $\clA$, $\clB$ and $\clC$ be three measurable sets such that $\clA \aseq{\mu} \clB$ and $\clB \aseq{\mu} \clC$.
  Since $\clA \setminus \clC = \big( (\clA \setminus \clC) \cap \clB \big) \cup \big( (\clA \setminus \clC) \setminus \clB \big)
  = \big( \clA \cap (\clB \setminus \clC) \big) \cup \big( (\clA \setminus \clB) \setminus \clC \big)
  \subseteq (\clB \setminus \clC) \cup (\clA \setminus \clB)$, we have
  $\mu(\clA \setminus \clC) \le \mu(\clB \setminus \clC) + \mu(\clA \setminus \clB) = 0$ due to Lemma~\ref{lem:assame-meas}.
  Similarly, $\mu(\clC \setminus \clA) = 0$.
  So $\mu(\clA \triangle \clC) = \mu(\clA \setminus \clC) + \mu(\clC \setminus \clA) = 0$. 
\end{proof}

\begin{lemma} \label{lem:assame-setop}
  Let $\clS$ and $\clSt$ be two measurable sets on a measure space $(\Omega, \scF, \mu)$ such that $\clS \aseq{\mu} \clSt$.
  Then for any measurable set $\clW$, we have $\clS \cup \clW \aseq{\mu} \clSt \cup \clW$, $\clS \cap \clW \aseq{\mu} \clSt \cap \clW$, $\clS \setminus \clW \aseq{\mu} \clSt \setminus \clW$ and $\clW \setminus \clS \aseq{\mu} \clW \setminus \clSt$.
\end{lemma}
\begin{proof}
  Let $\clD^+ := \clSt \setminus \clS$ and $\clD^- := \clS \setminus \clSt$.
  By Lemma~\ref{lem:assame-meas}, we have $\mu(\clD^+) = 0$ and $\mu(\clD^-) = 0$.

  For any measurable set $\clW$, we have $(\clSt \cup \clW) \setminus (\clS \cup \clW) = \clSt \setminus \clS \setminus \clW = \clD^+ \setminus \clW$,
  and similarly $(\clS \cup \clW) \setminus (\clSt \cup \clW) = \clD^- \setminus \clW$.
  So $\mu\big( (\clS \cup \clW) \triangle (\clSt \cup \clW) \big)
  = \mu\big( \big( (\clS \cup \clW) \setminus (\clSt \cup \clW) \big) \cup \big( (\clSt \cup \clW) \setminus (\clS \cup \clW) \big) \big)
  = \mu\big( (\clD^- \setminus \clW) \cup (\clD^+ \setminus \clW) \big) = \mu(\clD^- \setminus \clW) + \mu(\clD^+ \setminus \clW)
  \le \mu(\clD^-) + \mu(\clD^+) = 0$, that is $\clS \cup \clW \aseq{\mu} \clSt \cup \clW$.

  Since $(\clSt \cap \clW) \setminus (\clS \cap \clW) = (\clSt \setminus \clS) \cap \clW = \clD^+ \cap \clW$
  and similarly $(\clS \cap \clW) \setminus (\clSt \cap \clW) = \clD^- \cap \clW$, we have
  $\mu\big( (\clS \cap \clW) \triangle (\clSt \cap \clW) \big)
  = \mu\big( \big( (\clS \cap \clW) \setminus (\clSt \cap \clW) \big) \cup \big( (\clSt \cap \clW) \setminus (\clS \cap \clW) \big) \big)
  = \mu\big( (\clD^- \cap \clW) \cup (\clD^+ \cap \clW) \big) = \mu(\clD^- \cap \clW) + \mu(\clD^+ \cap \clW)
  \le \mu(\clD^-) + \mu(\clD^+) = 0$, so $\clS \cap \clW \aseq{\mu} \clSt \cap \clW$.

  Since $(\clSt \setminus \clW) \setminus (\clS \setminus \clW) = \clSt \setminus \clW \setminus \clS = \clSt \setminus \clS \setminus \clW = \clD^+ \setminus \clW$
  and similarly $(\clS \setminus \clW) \setminus (\clSt \setminus \clW) = \clD^- \setminus \clW$, we have
  $\mu\big( (\clS \setminus \clW) \triangle (\clSt \setminus \clW) \big)
  = \mu\big( \big( (\clS \setminus \clW) \setminus (\clSt \setminus \clW) \big) \cup \big( (\clSt \setminus \clW) \setminus (\clS \setminus \clW) \big) \big)
  = \mu\big( (\clD^- \setminus \clW) \cup (\clD^+ \setminus \clW) \big) = \mu(\clD^- \setminus \clW) + \mu(\clD^+ \setminus \clW)
  \le \mu(\clD^-) + \mu(\clD^+) = 0$, so $\clS \setminus \clW \aseq{\mu} \clSt \setminus \clW$.

  Since $(\clW \setminus \clSt) \setminus (\clW \setminus \clS) = \clW \setminus (\clW \setminus \clS) \setminus \clSt
  = (\clW \cap \clS) \setminus \clSt = (\clS \setminus \clSt) \cap \clW = \clD^- \cap \clW$
  and similarly $(\clW \setminus \clS) \setminus (\clW \setminus \clSt) = \clD^+ \cap \clW$, we have
  $\mu\big( (\clW \setminus \clS) \triangle (\clW \setminus \clSt) \big)
  = \mu\big( \big( (\clW \setminus \clS) \setminus (\clW \setminus \clSt) \big) \cup \big( (\clW \setminus \clSt) \setminus (\clW \setminus \clS) \big) \big)
  = \mu\big( (\clD^+ \cap \clW) \cup (\clD^- \cap \clW) \big) = \mu(\clD^+ \cap \clW) + \mu(\clD^- \cap \clW)
  \le \mu(\clD^+) + \mu(\clD^-) = 0$, so $\clW \setminus \clS \aseq{\mu} \clW \setminus \clSt$.
\end{proof}

\begin{definition} \label{def:asunique}
  We say that a set satisfying a certain condition is \emph{$\mu$-unique}, if for any two such sets $\clS$ and $\clSt$, it holds that $\clS \aseq{\mu} \clSt$.
\end{definition}

\begin{lemma} \label{lem:support-asunique}
  Let $\pi$ be an absolutely continuous distribution (probability measure) on a measure space $(\Omega, \scF, \mu)$ with a density function $f$.
  If a set $\clS \in \scF$ satisfies $\pi(\clS) = 1$ and that $f > 0$, $\mu$-a.e. on $\clS$,
  then such an $\clS$ is $\mu$-unique.
\end{lemma}
\begin{proof}
  Suppose we have two such sets $\clS$ and $\clSt$.
  By Lemma~\ref{lem:support-int}, we know that for any $\clW \in \scF$, $\pi(\clW) = \int_{\clW \cap \clS} f \dd \mu
  = \int_\clW \bbI_\clS f \dd \mu = \int_\clW \bbI_\clSt f \dd \mu$.
  So by \citet[Thm.~16.10(ii)]{billingsley2012probability}, we know that $\bbI_\clS f = \bbI_\clSt f$, $\mu$-a.e.

  Since $f > 0$, $\mu$-a.e. on $\clS$, we know that $\bbI_\clS = \bbI_\clSt$, $\mu$-a.e. on $\clS$.
  This means that $\mu\{\omega \in \clS \mid \bbI_\clS \ne \bbI_\clSt\}
  = \mu\{\omega \in \clS \mid \omega \notin \clSt\} = \mu(\clS \setminus \clSt) = 0$.
  Symmetrically, since $f > 0$, $\mu$-a.e. also on $\clSt$, we know that $\mu(\clSt \setminus \clS) = 0$.
  So we have $\mu(\clS \triangle \clSt) = \mu( (\clS \setminus \clSt) \cup (\clSt \setminus \clS) ) = \mu(\clS \setminus \clSt) + \mu(\clSt \setminus \clS) = 0$,
  which means that $\clS \aseq{\mu} \clSt$.
\end{proof}

The $\mu$-unique set $\clS$ in the lemma serves as another form of the \emph{support} of a distribution.
The standard definition of the support requires a topological structure and $\scF$ is the corresponding Borel sigma-field.
If given absolute continuity $\pi \ll \mu$, this lemma enables the generality that does not require a topological structure.
The condition $\pi(\clS) = 1$ prevents $\clS$ to be too small, while the condition that $f > 0$, $\mu$-a.e. on $\clS$ prevents $\clS$ to be too large.
\begin{definition}[support of an absolutely continuous distribution (without topology)] \label{def:support}
  Define the \emph{support} of an absolutely continuous distribution (probability measure) $\pi$ on a measure space $(\Omega, \scF, \mu)$,
  as the $\mu$-unique set $\clS \in \scF$ such that $\pi(\clS) = 1$ and for any density function $f$ of $\pi$, it holds that $f > 0$, $\mu$-a.e. on $\clS$.
\end{definition}

\subsection{Lemmas for Product Probability} \label{supp:lemma-prod}

In this subsection and the following, let $(\bbX\times\bbZ, \scX\otimes\scZ, \xi\otimes\zeta)$ be the product measure space by the two individual ones $(\bbX, \scX, \xi)$ and $(\bbZ, \scZ, \zeta)$,
where $\xi$ and $\zeta$ are sigma-finite.

\begin{lemma} \label{lem:marg-ac}
  For a measure $\pi$ on the product measure space $(\bbX\times\bbZ, \scX\otimes\scZ, \xi\otimes\zeta)$,
  if $\pi \ll \xi\otimes\zeta$, then $\pi^\bbX \ll \xi$ and $\pi^\bbZ \ll \zeta$.
\end{lemma}
\begin{proof}
  For any $\clX \in \scX$ such that $\xi(\clX) = 0$, we have $(\xi\otimes\zeta)(\clX \times \bbZ) = \xi(\clX) \zeta(\bbZ) = 0$,
  where the last equality is verified in \bfii\bfone in Supplement~\ref{supp:meas-prod} when $\zeta(\bbZ) = \infty$.
  Since $\pi \ll \xi\otimes\zeta$, this means that $\pi(\clX \times \bbZ) = \pi^\bbX(\clX) = 0$.
  So $\pi^\bbX \ll \xi$.
  Similarly, $\pi^\bbZ \ll \zeta$.
\end{proof}

\begin{lemma} \label{lem:slice-ae}
  For an assertion $t(x,z)$ on $\clW \in \scX\otimes\scZ$,
  $t(x,z)$ holds $\xi\otimes\zeta$-a.e. on $\clW$, if and only if $t(x,z)$ holds $\xi$-a.e. on $\clW_z$, for $\zeta$-a.e. $z$ on $\clW^\bbZ$.
\end{lemma}
\begin{proof}
  By the definition of ``$t(x,z)$ holds $\xi\otimes\zeta$-a.e. on $\clW$'', we have:
  \begin{flalign}
    & (\xi\otimes\zeta)\{(x,z) \in \clW \mid \neg t(x,z)\} = 0
    & \hspace{-8in} \text{(Since $\xi$ and $\zeta$ are sigma-finite, from \eqref{eqn:prod-meas-restr},)} \\
    \Longleftrightarrow{} & \int_{\clW^\bbZ} \xi\{x \in \clW_z \mid \neg t(x,z)\} \zeta(\ud z) = 0
    \intertext{(Since $\xi(\cdot)$ is nonnegative, from \citet[Thm.~15.2]{billingsley2012probability},)}
    \Longleftrightarrow{} & \xi\{x \in \clW_z \mid \neg t(x,z)\} = 0, \text{for $\zeta$-a.e. $z$ on $\clW^\bbZ$},
  \end{flalign}
  which is ``$t(x,z)$ holds $\xi$-a.e. on $\clW_z$, for $\zeta$-a.e. $z$ on $\clW^\bbZ$''.
\end{proof}

\begin{lemma} \label{lem:assame-strip}
  Let $\clX, \clXt \in \scX$ such that $\clX \aseq{\xi} \clXt$.
  Then $\clX \times \bbZ \aseq{\xi\otimes\zeta} \clXt \times \bbZ$.
\end{lemma}
\begin{proof}
  Since $(\clX \times \bbZ) \triangle (\clXt \times \bbZ)
  = \big( (\clX \times \bbZ) \setminus (\clXt \times \bbZ) \big) \cup \big( (\clXt \times \bbZ) \setminus (\clX \times \bbZ) \big)
  = \big( (\clX \setminus \clXt) \cup (\clXt \setminus \clX) \big) \times \bbZ$, we can verify that
  $(\xi\otimes\zeta) \big( (\clX \times \bbZ) \triangle (\clXt \times \bbZ) \big)
  = (\xi\otimes\zeta) \big( \big( (\clX \setminus \clXt) \cup (\clXt \setminus \clX) \big) \times \bbZ \big)
  = \xi\big( (\clX \setminus \clXt) \cup (\clXt \setminus \clX) \big) \zeta(\bbZ)
  = \xi(\clX \triangle \clXt) \zeta(\bbZ) = 0$,
  where the last equality is verified in \bfii\bfone in Supplement~\ref{supp:meas-prod} when $\zeta(\bbZ) = \infty$.
\end{proof}

\subsection{Lemmas for \texorpdfstring{$\xi\otimes\zeta$}{\textbackslash xi \textbackslash otimes \textbackslash zeta}-Complete Component} \label{supp:lemma-irrcomp}

Echoing Def.~\ref{def:irrcomp}, a set $\clS \in \scX\otimes\scZ$ is called a \emph{$\xi\otimes\zeta$-complete component} of $\clW \in \scX\otimes\scZ$, if
\begin{align} \label{eqn:def-irrcomp-var}
  \clS^\sharp \cap \clW \aseq{\xi\otimes\zeta} \clS,
  \text{where } \clS^\sharp := \clS^\bbX \times \bbZ \cup \bbX \times \clS^\bbZ.
\end{align}
This means that $\clS$ is complete under \emph{stretching} and intersecting with $\clW$.

\begin{lemma} \label{lem:irrcomp-assubseteq}
  Let $\clS$ be a $\xi\otimes\zeta$-complete component of $\clW$.
  Then $\clS \assubseteq{\xi\otimes\zeta} \clW$.
\end{lemma}
\begin{proof}
  By construction, we have $\clS \subseteq \clS^\sharp$ so $\clS \setminus \clW = \clS \setminus (\clS \cap \clW) = \clS \setminus (\clS^\sharp \cap \clW)$.
  Hence, $(\xi\otimes\zeta)(\clS \setminus \clW) = (\xi\otimes\zeta)\big( \clS \setminus (\clS^\sharp \cap \clW) \big) = 0$ by definition \eqref{eqn:def-irrcomp-var} and Lemma~\ref{lem:assame-meas}.
\end{proof}

\begin{wrapfigure}{r}{.320\textwidth}
  \centering
  \vspace{-2pt}
  \includegraphics[width=.182\textwidth]{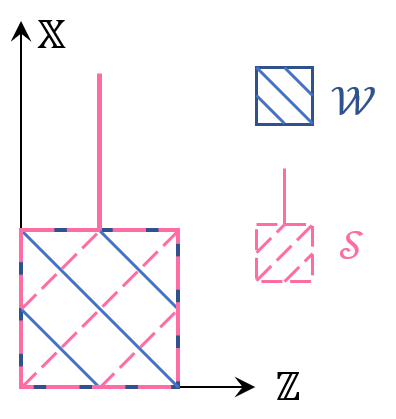}
  \vspace{-6pt}
  \caption{Example~\ref{exmp:irrcomp-projsubset} showing that a $\xi\otimes\zeta$-complete component of $\clW$ may not have its projection be an a.s. subset of that of $\clW$.
  }
  \vspace{-20pt}
  \label{fig:exmp-irrcomp-projsubset}
\end{wrapfigure}
\phantom{a}

\vspace{-14pt}
\begin{example} \label{exmp:irrcomp-projsubset}
  Note that when $\clS$ is a $\xi\otimes\zeta$-complete component of $\clW$, it may not hold that $\clS^\bbX \assubseteq{\xi} \clW^\bbX$ and $\clS^\bbZ \assubseteq{\zeta} \clW^\bbZ$.
  Fig.~\ref{fig:exmp-irrcomp-projsubset} shows an example,
  where $(\bbX,\scX,\xi)$ and $(\bbZ,\scZ,\zeta)$ are the one dimensional Euclidean spaces with line Borel sigma-field and line Lebesgue measure, $(\bbR,\scR,\lambda)$,
  and $\clW := [0,1]^2$ and $\clS := [0,1]^2 \cup ([1,2] \times \{\frac{1}{2}\})$.
  We have $\clS^\bbX = [0,2]$ so $\clS^\sharp = ([0,2] \times \bbR) \cup (\bbR \times [0,1])$ and $\clS^\sharp \cap \clW = \clW$.
  Since $\clS \triangle \clW = [1,2] \times \{\frac{1}{2}\}$ is a line segment that has measure zero under the plane Lebesgue measure $\xi\otimes\zeta = \lambda^2$, we have $\clS \aseq{\xi\otimes\zeta} \clW$ so $\clS$ is a $\xi\otimes\zeta$-complete component of $\clW$.
  But $\xi(\clS^\bbX \setminus \clW^\bbX) = \lambda([0,2] \setminus [0,1]) = \lambda(1,2] = 1$ is not zero, so $\clS^\bbX \assubseteq{\xi} \clW^\bbX$ does not hold.
\end{example}

\begin{lemma} \label{lem:irrcomp-assame}
  Let $\clS$ be a $\xi\otimes\zeta$-complete component of $\clW$,
  and $\clSt$ be a measurable set such that $\clSt \aseq{\xi\otimes\zeta} \clS$, $\clSt^\bbX \aseq{\xi} \clS^\bbX$ and $\clSt^\bbZ \aseq{\zeta} \clS^\bbZ$.
  Then this $\clSt$ is also a $\xi\otimes\zeta$-complete component of $\clW$.
\end{lemma}
\begin{proof}
  By Lemma~\ref{lem:assame-strip}, we know that $\clSt^\bbX \times \bbZ \aseq{\xi\otimes\zeta} \clS^\bbX \times \bbZ$, $\bbX \times \clSt^\bbZ \aseq{\xi\otimes\zeta} \bbX \times \clS^\bbZ$.
  Repeatedly applying Lemma~\ref{lem:assame-setop}, we have $\clSt^\sharp := \clSt^\bbX \times \bbZ \cup \bbX \times \clSt^\bbZ
  \aseq{\xi\otimes\zeta} \clS^\bbX \times \bbZ \cup \bbX \times \clSt^\bbZ
  \aseq{\xi\otimes\zeta} \clS^\bbX \times \bbZ \cup \bbX \times \clS^\bbZ =: \clS^\sharp$,
  and $\clSt^\sharp \cap \clW \aseq{\xi\otimes\zeta} \clS^\sharp \cap \clW$,
  which $\aseq{\xi\otimes\zeta} \clS \aseq{\xi\otimes\zeta} \clSt$.
  From the transitivity (Lemma~\ref{lem:assame-equiv}), we have $\clSt^\sharp \cap \clW \aseq{\xi\otimes\zeta} \clSt$.
\end{proof}

\begin{wrapfigure}{r}{.320\textwidth}
  \centering
  \vspace{-8pt}
  \includegraphics[width=.280\textwidth]{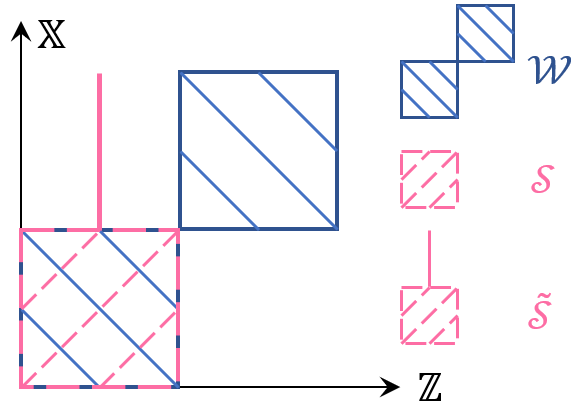}
  \vspace{-4pt}
  \caption{Example~\ref{exmp:irrcomp-assame} showing that in Lem.~\ref{lem:irrcomp-assame},
    only being $\xi\otimes\zeta$-a.s. the same as a $\xi\otimes\zeta$-complete component $\clSt$ of $\clW$
    is not sufficient for $\clSt$ to be also a $\xi\otimes\zeta$-complete component of $\clW$.
  }
  \vspace{-24pt}
  \label{fig:exmp-irrcomp-assame}
\end{wrapfigure}
\phantom{a}

\vspace{-14pt}
\begin{example} \label{exmp:irrcomp-assame}
  Note that only the $\clSt \aseq{\xi\otimes\zeta} \clS$ condition is not sufficient.
  Fig.~\ref{fig:exmp-irrcomp-assame} shows such an example,
  where $(\bbX,\scX,\xi)$ and $(\bbZ,\scZ,\zeta)$ are the one dimensional Euclidean spaces with line Borel sigma-field and line Lebesgue measure, $(\bbR,\scR,\lambda)$,
  and $\clW := [0,1]^2 \cup [1,2]^2$, $\clS := [0,1]^2$, and $\clSt := [0,1]^2 \cup ([1,2] \times \{\frac{1}{2}\})$.
  We have $\clS^\sharp = ([0,1] \times \bbR) \cup (\bbR \times [0,1])$ so $\clS^\sharp \cap \clW = \clS$, justifying that $\clS$ is a $\xi\otimes\zeta$-complete component of $\clW$.
  On the other hand, since $\clS \triangle \clSt = [1,2] \times \{\frac{1}{2}\}$ is a line segment that has measure zero under the plane Lebesgue measure $\xi\otimes\zeta = \lambda^2$, we have $\clSt \aseq{\xi\otimes\zeta} \clS$.
  But $\clSt^\bbX = [0,2]$ so $\clSt^\sharp = ([0,2] \times \bbR) \cup (\bbR \times [0,1])$, which leads to $\clSt^\sharp \cap \clW = \clW$.
  Since $\clSt \triangle \clW = ([1,2] \times \{\frac{1}{2}\}) \cup ([1,2] \times [1,2])$ has a nonzero measure under $\lambda^2$ (it equals to 1),
  we know that $\clSt$ is not a $\xi\otimes\zeta$-complete component of $\clW$.
\end{example}

\begin{lemma} \label{lem:irrcomp-inteq}
  Let $\clS$ be a $\xi\otimes\zeta$-complete component of $\clW$,
  and $f$ be an either nonnegative or $\xi\otimes\zeta$-integrable function on $\bbX\times\bbZ$.
  Then for any measurable sets $\clZ \subseteq \clS^\bbZ$ and $\clX \subseteq \clS^\bbX$, we have:
  \begin{align}
    \int_\clZ \int_{\clW_z} f(x,z) \xi(\ud x) \zeta(\ud z) = \int_\clZ \int_{\clS_z} f(x,z) \xi(\ud x) \zeta(\ud z), \\ 
    \int_\clX \int_{\clW_x} f(x,z) \zeta(\ud z) \xi(\ud x) = \int_\clX \int_{\clS_x} f(x,z) \zeta(\ud z) \xi(\ud x).
  \end{align}
  Particularly, $\int_{\clS^\bbZ} \int_{\clW_z} f(x,z) \xi(\ud x) \zeta(\ud z)
  = \int_{\clS^\bbX} \int_{\clW_x} f(x,z) \zeta(\ud z) \xi(\ud x)
  = \int_\clS f(x,z) (\xi\otimes\zeta)(\ud x \ud z)$.
\end{lemma}
\begin{proof}
  Since $\clS$ is a $\xi\otimes\zeta$-complete component of $\clW$, \eqref{eqn:def-irrcomp-var} holds.
  By Lemma~\ref{lem:slice-ae}, we know that for $\zeta$-a.e. $z$ on $\bbZ$, $\xi \big( (\clS^\sharp \cap \clW) \triangle \clS \big)_z
  = \xi \big( (\clS^\sharp_z \cap \clW_z) \triangle \clS_z \big) = 0$.
  Noting that $\clS^\sharp_z = \bbX$ for any $z \in \clS^\bbZ$,
  this subsequently means that $\xi(\clW_z \triangle \clS_z) = 0$ for $\zeta$-a.e. $z$ on $\clS^\bbZ$.
  By the additivity of integrals over a countable partition~\citep[Thm.~16.9]{billingsley2012probability} and that the integral over a measure-zero set is zero~\citep[p.226]{billingsley2012probability},
  we have $\int_{\clW_z} f(x,z) \xi(\ud x) = \int_{\clS_z} f(x,z) \xi(\ud x)$ for $\zeta$-a.e. $z$ on $\clS^\bbZ$.
  Since a.e.-equal functions have the same integral~\citep[Thm.~15.2(v)]{billingsley2012probability}, we have
  for any measurable $\clZ \subseteq \clS^\bbZ$, $\int_\clZ \int_{\clW_z} f(x,z) \xi(\ud x) \zeta(\ud z) = \int_\clZ \int_{\clS_z} f(x,z) \xi(\ud x) \zeta(\ud z)$.
  Similarly, for any measurable $\clX \subseteq \clS^\bbX$, $\int_\clX \int_{\clW_x} f(x,z) \zeta(\ud z) \xi(\ud x) = \int_\clX \int_{\clS_x} f(x,z) \zeta(\ud z) \xi(\ud x)$.

  For $\clZ = \clS^\bbZ$, we have $\int_{\clS^\bbZ} \int_{\clW_z} f(x,z) \xi(\ud x) \zeta(\ud z) = \int_{\clS^\bbZ} \int_{\clS_z} f(x,z) \xi(\ud x) \zeta(\ud z)$,
  which is $\int_\clS f(x,z) (\xi\otimes\zeta)(\ud x \ud z)$ by the generalized form \eqref{eqn:fubini-restr} of Fubini's theorem.
  Similarly, $\int_{\clS^\bbX} \int_{\clW_x} f(x,z) \zeta(\ud z) \xi(\ud x) = \int_\clS f(x,z) (\xi\otimes\zeta)(\ud x \ud z)$.
\end{proof}

\section{Proofs} \label{supp:proofs}

Recall that $(\bbX\times\bbZ, \scX\otimes\scZ, \xi\otimes\zeta)$ is the product measure space by the two individual ones $(\bbX, \scX, \xi)$ and $(\bbZ, \scZ, \zeta)$,
where $\xi$ and $\zeta$ are sigma-finite.

\subsection{The Joint-Conditional Absolute Continuity Lemma} \label{supp:proofs-joint-ac}

Although this lemma is not formally presented in the main text, we highlight it here since it answers an important question and the answer is not straightforward.

The lemma reveals the relation between the absolute continuity of a joint $\pi$ and that of its conditionals $\pi(\cdot|z)$, $\pi(\cdot|x)$.
Roughly, the former guarantees the latter on the supports of the marginals, and the reverse also holds, allowing one to safely use density function formulae for deduction.
But given two conditionals, one does not have the knowledge on the marginals \emph{a priori}.
For a more useful sufficient condition, one may consider the absolute continuity of the conditionals for $\zeta$-a.e. $z$ and $\xi$-a.e. $x$.
Unfortunately this is not sufficient, and an example (\ref{exmp:joint-ac}) is given after the proof.
The lemma shows it is sufficient if the absolute continuity of one of the conditionals, say $\pi(\cdot|z)$, holds for \emph{any} $z \in \bbZ$.
The condition in the compatibility criterion Thm.~\ref{thm:compt-ac} is also inspired from this lemma.

\begin{lemma}[joint-conditional absolute continuity] \label{lem:joint-ac}
  \bfi For a joint distribution $\pi$ on $(\bbX\times\bbZ, \scX\otimes\scZ)$, it is absolutely continuous $\pi \ll \xi\otimes\zeta$ if and only if $\pi(\cdot|z) \ll \xi$ for $\pi^\bbZ$-a.e. $z$ and $\pi(\cdot|x) \ll \zeta$ for $\pi^\bbX$-a.e. $x$.
  \bfii As a sufficient condition, $\pi \ll \xi\otimes\zeta$ if $\pi(\cdot|z) \ll \xi$ for $\zeta$-a.e. $z$ and $\pi(\cdot|x) \ll \zeta$ for \emph{any} $x \in \bbX$ (or for \emph{any} $z \in \bbZ$ and $\xi$-a.e. $x$).
\end{lemma}

\textbf{For conclusion \bfi:}
\begin{proof}
  \textbf{``Only if'':}
  Consider any $\clX \in \scX$ such that $\xi(\clX) = 0$.
  From the definition of conditional distribution \eqref{eqn:cond-int-prod-indep}, we have
  $\pi^\bbX(\clX) = \pi(\clX \times \bbZ) = \int_\bbZ \pi(\clX|z) \pi^\bbZ(\ud z) = 0$,
  so $\pi(\clX|z) = 0$ for $\pi^\bbZ$-a.e. $z$ since $\pi(\clX|z)$ is nonnegative~\citep[Thm.~15.2(ii)]{billingsley2012probability}.
  This means that $\pi(\cdot|z) \ll \xi$ for $\pi^\bbZ$-a.e. $z$.
  The same arguments apply symmetrically to $\pi(\cdot|x)$.

  Note that since $\pi(\cdot|z)$ is defined as the R-N derivative, it is allowed to take any nonnegative value on a $\pi^\bbZ$-measure-zero set.
  So we cannot guarantee its behavior for \emph{any} $z \in \bbZ$.

  \textbf{``If'':}
  Consider any $\clZ \in \scZ$ such that $\zeta(\clZ) = 0$.
  Since $\pi(\cdot|x) \ll \zeta$ for $\pi^\bbX$-a.e. $x$, we have $\pi(\clZ|x) = 0$ for $\pi^\bbX$-a.e. $x$.
  So from \eqref{eqn:cond-int-prod-indep} we have $\pi^\bbZ(\clZ) = \pi(\bbX \times \clZ) = \int_\bbX \pi(\clZ|x) \pi^\bbX(\ud x) = 0$~\citep[Thm.~15.2(i)]{billingsley2012probability}.
  This indicates that $\pi^\bbZ \ll \zeta$.

  Now consider any $\clW \in \scX\otimes\scZ$ such that $(\xi\otimes\zeta)(\clW) = 0$.
  By the definition of product measure \eqref{eqn:prod-meas}~\citep[Thm.~18.2]{billingsley2012probability}, we have $(\xi\otimes\zeta)(\clW) = \int_\bbZ \xi(\clW_z) \zeta(\ud z) = 0$,
  so $\xi(\clW_z) = 0$ for $\zeta$-a.e. $z$ since $\xi(\clW_z)$ is nonnegative~\citep[Thm.~15.2(ii)]{billingsley2012probability}.
  Due to that $\pi^\bbZ \ll \zeta$, this means that $\xi(\clW_z) = 0$ for $\pi^\bbZ$-a.e. $z$.
  Since $\pi(\cdot|z) \ll \xi$ for $\pi^\bbZ$-a.e. $z$, this in turn means that $\pi(\clW_z|z) = 0$ for $\pi^\bbZ$-a.e. $z$.
  Subsequently, we have $\int_\bbZ \pi(\clW_z|z) \pi^\bbZ(\ud z) = 0$~\citep[Thm.~15.2(i)]{billingsley2012probability}, which is $\pi(\clW) = 0$ by \eqref{eqn:cond-int-prod-dep}.
  So we get $\pi \ll \xi\otimes\zeta$.
\end{proof}

\begin{wrapfigure}{r}{.310\textwidth}
  \centering
  \vspace{-16pt}
  \includegraphics[width=.280\textwidth]{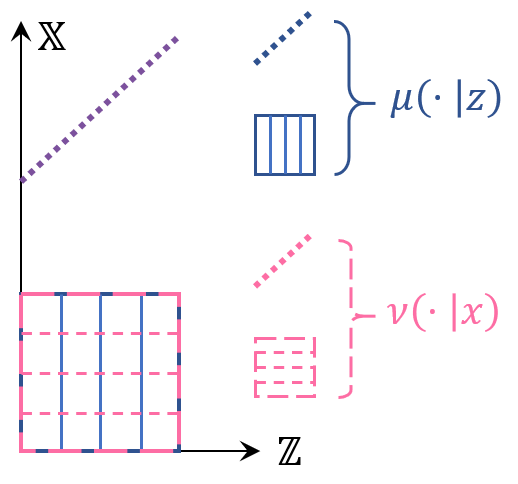}
  \vspace{-8pt}
  \caption{Illustration of the conditionals in \eqref{eqn:ac-exmp-conds} in Example~\ref{exmp:joint-ac}.
    Both conditionals are absolutely continuous for $\zeta$-a.e. $z$ or $\xi$-a.e. $x$, but they allow a compatible joint that is not absolutely continuous w.r.t $\xi\otimes\zeta$.
  }
  \vspace{-24pt}
  \label{fig:exmp-joint-ac}
\end{wrapfigure}

\textbf{For conclusion \bfii:}
\begin{proof}
  Consider any $\clZ \in \scZ$ such that $\zeta(\clZ) = 0$.
  Since $\pi(\cdot|x) \ll \zeta$ for any $x \in \bbX$, we know that $\pi(\clZ|x) = 0$ for any $x \in \bbX$.
  So from \eqref{eqn:cond-int-prod-indep} we have $\pi^\bbZ(\clZ) = \pi(\bbX \times \clZ) = \int_\bbX \pi(\clZ|x) \pi^\bbX(\ud x) = 0$.
  This indicates that $\pi^\bbZ \ll \zeta$.

  Now consider any $\clW \in \scX\otimes\scZ$ such that $(\xi\otimes\zeta)(\clW) = 0$.
  By the definition of product measure \eqref{eqn:prod-meas}~\citep[Thm.~18.2]{billingsley2012probability}, we have $(\xi\otimes\zeta)(\clW) = \int_\bbZ \xi(\clW_z) \zeta(\ud z) = 0$,
  so $\xi(\clW_z) = 0$ for $\zeta$-a.e. $z$ since $\xi(\clW_z)$ is nonnegative~\citep[Thm.~15.2(ii)]{billingsley2012probability}.
  Due to that $\pi(\cdot|z) \ll \xi$ for $\zeta$-a.e. $z$, this means that $\pi(\clW_z|z) = 0$ for $\zeta$-a.e. $z$.
  Since $\pi^\bbZ \ll \zeta$, this in turn means that $\pi(\clW_z|z) = 0$ for $\pi^\bbZ$-a.e. $z$.
  Subsequently, we have $\int_\bbZ \pi(\clW_z|z) \pi^\bbZ(\ud z) = 0$~\citep[Thm.~15.2(i)]{billingsley2012probability}, which is $\pi(\clW) = 0$ by \eqref{eqn:cond-int-prod-dep}.
  So we get $\pi \ll \xi\otimes\zeta$.
  The same arguments apply symmetrically when $\pi(\cdot|z) \ll \xi$ for \emph{any} $z \in \bbZ$ and $\pi(\cdot|x) \ll \zeta$ for $\xi$-a.e. $x$.
\end{proof}

\begin{example} \label{exmp:joint-ac}
  To see why it is not sufficient if the two conditionals are absolutely continuous only for $\zeta$-a.e. $z$ and $\xi$-a.e. $x$,
  we show an example below.

  Consider the one-dimensional Euclidean space $\bbX = \bbZ = \bbR$ with line Borel sigma-field $\scX = \scZ = \scR$ and line Lebesgue measure $\xi = \zeta = \lambda$. Let
  \begin{align}
    \mu(\cdot|z) := \begin{cases} \delta_{z+2}, & z \in \bbQ[0,1], \\ \Unif[0,1], & z \in \bar{\bbQ}[0,1], \\ 0, & \text{otherwise}, \end{cases} \;
    \nu(\cdot|x) := \begin{cases} \Unif[0,1], & x \in [0,1], \\ \delta_{x-2}, & x \in \bbQ[2,3], \\ 0, & \text{otherwise}, \end{cases}
    \label{eqn:ac-exmp-conds}
  \end{align}
  where $\bbQ[0,1] := [0,1] \cap \bbQ$ and $\bar{\bbQ} := [0,1] \backslash \bbQ$ are the rational and irrational numbers on $[0,1]$.
  The conditionals are illustrated in Fig.~\ref{fig:exmp-joint-ac}.
  Since $\lambda(\bbQ) = 0$, the two conditionals are absolutely continuous for $\zeta$-a.e. $z$ and $\xi$-a.e. $x$.
  Consider the joint distribution on $\bbX \times \bbZ = \bbR^2$:
  \begin{align}
    \pi := \frac{1}{2} \Unif([0,1] \times [0,1]) + \frac{1}{2} \sum_{z \in \bbQ[0,1]} \varrho(z) \delta_{(z+2, z)},
    \label{eqn:ac-exmp-joint}
  \end{align}
  where $\varrho$ is a distribution on the rationals $\bbQ[0,1]$ in $[0,1]$ with the sigma-field of all the subsets of $\bbQ[0,1]$.
  Such a distribution exists, for example, $\varrho(z) = 1/2^{n(z)}$ where $n: \bbQ[0,1] \to \bbN^*$ bijective is a numbering function of the countable set $\bbQ[0,1]$.
  In this way, each rational number $z \in \bbQ[0,1]$ has a positive probability, meanwhile we have $\varrho(\bbQ[0,1]) = \sum_{n=1}^\infty 1/2^n = 1$.
  Since $\pi(\{(z+2, z) \mid z \in \bbQ[0,1]\}) = \frac{1}{2}$ but $\lambda^2(\{(z+2, z) \mid z \in \bbQ[0,1]\}) = 0$ under the square Lebesgue measure $\lambda^2$, $\pi$ is not absolutely continuous.

  To verify compatibility, note that $\mu$ and $\nu$ here satisfy the corresponding measurability and integrability.
  To verify \eqref{eqn:cond-int-prod-indep} reduced from \eqref{eqn:cond-int-genrl} for defining a conditional, we first derive the marginals:
  \begin{align}
    \pi^\bbX = \frac{1}{2} \Unif[0,1] + \frac{1}{2} \sum_{z \in \bbQ[0,1]} \varrho(z) \delta_{z+2}, \;\;
    \pi^\bbZ = \frac{1}{2} \Unif[0,1] + \frac{1}{2} \sum_{z \in \bbQ[0,1]} \varrho(z) \delta_z.
  \end{align}
  For any $\clX \in \scX$ and $\clZ \in \scZ$, we have:
  \begin{align}
    \pi(\clX\times\clZ) ={} & \frac{1}{2} \Unif[0,1](\clX) \Unif[0,1](\clZ) + \frac{1}{2} \sum_{z \in \bbQ[0,1]} \varrho(z) \bbI[(z+2, z) \in \clX\times\clZ] \\
    ={} & \frac{1}{2} \lambda(\clX[0,1]) \lambda(\clZ[0,1]) + \frac{1}{2} \sum_{z \in \bbQ[0,1]} \varrho(z) \bbI[z \in (\clX-2) \cap \clZ],
  \end{align}
  where $\clX[0,1] := [0,1] \cap \clX$ and $\clZ[0,1] := [0,1] \cap \clZ$.
  To verify the conditional distribution $\mu(\clX|z)$, we have:
  \begin{align}
    \int_\clZ \mu(\clX|z) \pi^\bbZ(\ud z) ={} & \int_{\clZ \cap \bbQ[0,1]} \delta_{z+2}(\clX) \pi^\bbZ(\ud z) + \int_{\clZ \cap \bar{\bbQ}[0,1]} \Unif[0,1](\clX) \pi^\bbZ(\ud z) \\
    ={} & \int_{\clZ \cap \bbQ[0,1]} \bbI[z+2 \in \clX] \pi^\bbZ(\ud z) + \int_{\clZ \cap \bar{\bbQ}[0,1]} \lambda(\clX[0,1]) \pi^\bbZ(\ud z) \\
    ={} & \pi^\bbZ( (\clX-2) \cap \clZ \cap \bbQ[0,1] ) + \lambda(\clX[0,1]) \pi^\bbZ(\clZ \cap \bar{\bbQ}[0,1]) \\
    \intertext{(Since a countable set has measure zero under $\Unif[0,1]$, \ie $\lambda(\cdot \cap [0,1])$,)}
    ={} & \frac{1}{2} \sum_{z \in \bbQ[0,1]} \varrho(z) \bbI \big[ z \in (\clX-2) \cap \clZ \cap \bbQ[0,1] \big] \\*
    & {}+ \frac{1}{2} \lambda(\clX[0,1]) \bigg( \Unif[0,1](\clZ[0,1]) + \sum_{z \in \bbQ[0,1]} \varrho(z) \bbI \big[ z \in \clZ \cap \bar{\bbQ}[0,1] \big] \bigg) \\
    ={} & \frac{1}{2} \sum_{z \in \bbQ[0,1]} \varrho(z) \bbI[z \in (\clX-2) \cap \clZ] + \frac{1}{2} \lambda(\clX[0,1]) \lambda(\clZ[0,1])
	=     \pi(\clX\times\clZ).
  \end{align}
  For the conditional distribution $\nu(\clZ|x)$, we similarly have:
  \begin{align}
    \int_\clX \nu(\clZ|x) \pi^\bbX(\ud x) ={} & \int_{\clX[0,1]} \Unif[0,1](\clZ) \pi^\bbX(\ud x) + \int_{\clX \cap \bbQ[2,3]} \delta_{x-2}(\clZ) \pi^\bbX(\ud x) \\
    ={} & \lambda(\clZ[0,1]) \pi^\bbX(\clX[0,1]) + \pi^\bbX \big( \clX \cap \bbQ[2,3] \cap (\clZ+2) \big) \\
    ={} & \frac{1}{2} \lambda(\clZ[0,1]) \lambda(\clX[0,1]) + \frac{1}{2} \sum_{z \in \bbQ[0,1]} \varrho(z) \bbI[z+2 \in \clX \cap \bbQ[2,3] \cap (\clZ+2)] \\
    ={} & \frac{1}{2} \lambda(\clZ[0,1]) \lambda(\clX[0,1]) + \frac{1}{2} \sum_{z \in \bbQ[0,1]} \varrho(z) \bbI[z \in (\clX-2) \cap \clZ]
	=     \pi(\clX\times\clZ).
  \end{align}
  So the two conditionals $\mu(\cdot|z)$ and $\nu(\cdot|x)$ are compatible and $\pi$ is their joint distribution.
  This example illustrates that the absolute continuity of $\pi(\cdot|z)$ w.r.t $\xi$ for $\zeta$-a.e. $z$ and that of $\pi(\cdot|x)$ w.r.t $\zeta$ for $\xi$-a.e. $x$, does not indicate the absolute continuity of $\pi$ w.r.t $\xi\otimes\zeta$.

  This example does not contradict result \bfi of the Lemma.
  For any $z_0 \in \bbQ[0,1]$, we have that $\mu(\cdot|z_0) = \delta_{z_0 + 2}$ is not absolutely continuous w.r.t $\xi = \lambda$.
  But $\pi^\bbZ(\{z_0\}) = \frac{1}{2} \varrho(z_0) > 0$.
  So it is not that $\mu(\cdot|z) \ll \xi$ for $\pi^\bbZ$-a.e. $z$, which aligns with that $\pi$ is not absolutely continuous w.r.t $\xi\otimes\zeta = \lambda^2$.

  This example also shows that the absolute continuity of the compatible joint may depend on the joint itself, apart from the two conditionals.
  Consider another joint on $(\bbR^2, \scR^2, \lambda^2)$:
  \begin{align}
    \pit := \Unif([0,1] \times [0,1]).
    \label{eqn:ac-exmp-joint-var}
  \end{align}
  It is easy to see that $\pit^\bbX(\clX) = \Unif[0,1](\clX) = \lambda(\clX[0,1])$ and $\pit^\bbZ(\clZ) = \lambda(\clZ[0,1])$.
  For any $\clX \in \scX$ and $\clZ \in \scZ$, we have $\pit(\clX\times\clZ) = \lambda(\clX[0,1]) \lambda(\clZ[0,1])$.
  To verify \eqref{eqn:cond-int-prod-indep} for defining a conditional, we have:
  \begin{align}
    \int_\clZ \mu(\clX|z) \pit^\bbZ(\ud z) ={} & \int_{\clZ \cap \bbQ[0,1]} \delta_{z+2}(\clX) \pit^\bbZ(\ud z) + \int_{\clZ \cap \bar{\bbQ}[0,1]} \Unif[0,1](\clX) \pit^\bbZ(\ud z) \\
    ={} & \int_{\clZ \cap \bbQ[0,1]} \bbI[z+2 \in \clX] \pit^\bbZ(\ud z) + \int_{\clZ \cap \bar{\bbQ}[0,1]} \lambda(\clX[0,1]) \pit^\bbZ(\ud z) \\
    ={} & \lambda( (\clX-2) \cap \clZ \cap \bbQ[0,1] ) + \lambda(\clX[0,1]) \lambda(\clZ \cap \bar{\bbQ}[0,1]) \\
    \intertext{(Since $\lambda( (\clX-2) \cap \clZ \cap \bbQ[0,1] ) \le \lambda(\bbQ) = 0$ and $\lambda(\clZ \cap \bar{\bbQ}[0,1]) = \lambda(\clZ \cap \bar{\bbQ}[0,1]) + \lambda(\bbQ[0,1]) = \lambda(\clZ[0,1])$,)}
    ={} & \lambda(\clX[0,1]) \lambda(\clZ[0,1])
	=     \pit(\clX\times\clZ).
  \end{align}
  For the conditional distribution $\nu(\clZ|x)$, we similarly have:
  \begin{align}
    \int_\clX \nu(\clZ|x) \pit^\bbX(\ud x) ={} & \int_{\clX[0,1]} \Unif[0,1](\clZ) \pit^\bbX(\ud x) + \int_{\clX \cap \bbQ[2,3]} \delta_{x-2}(\clZ) \pit^\bbX(\ud x) \\
    ={} & \lambda(\clZ[0,1]) \pit^\bbX(\clX[0,1]) + \pit^\bbX \big( \clX \cap \bbQ[2,3] \cap (\clZ+2) \big) \\
    \intertext{(Since $\pit^\bbX \big( \clX \cap \bbQ[2,3] \cap (\clZ+2) \big) = \lambda \big( \clX \cap \bbQ[2,3] \cap (\clZ+2) \cap [0,1] \big) \le \lambda(\bbQ) = 0$,)}
    ={} & \lambda(\clZ[0,1]) \lambda(\clX[0,1])
	=     \pit(\clX\times\clZ).
  \end{align}
  So $\pit$ is also a compatible joint of $\mu(\cdot|z)$ and $\nu(\cdot|x)$.
  In this case, the violation set for $\mu(\cdot|z) \ll \xi$, \ie $\bbQ[0,1]$, has measure zero under $\pit^\bbZ$, and the violation set for $\nu(\cdot|x) \ll \zeta$, \ie $\bbQ[2,3]$, has measure zero under $\pit^\bbX$.
  So result \bfi of the Lemma asserts that $\pit \ll \xi\otimes\zeta$, which aligns with the example.
  This example shows that although both $\pi$ and $\pit$ are the compatible joint of the same conditionals, they have different absolute continuity.
  So the condition in result \bfi of the Lemma requires the knowledge of the marginals $\pi^\bbZ$ and $\pi^\bbX$.
\end{example}

\vspace{-10pt}
\subsection{Proof of Theorem~\ref{thm:compt-ac}} \label{supp:proofs-compt-ac}

\begin{wrapfigure}{r}{.470\textwidth}
  \centering
  \vspace{-34pt}
  \includegraphics[width=.460\textwidth]{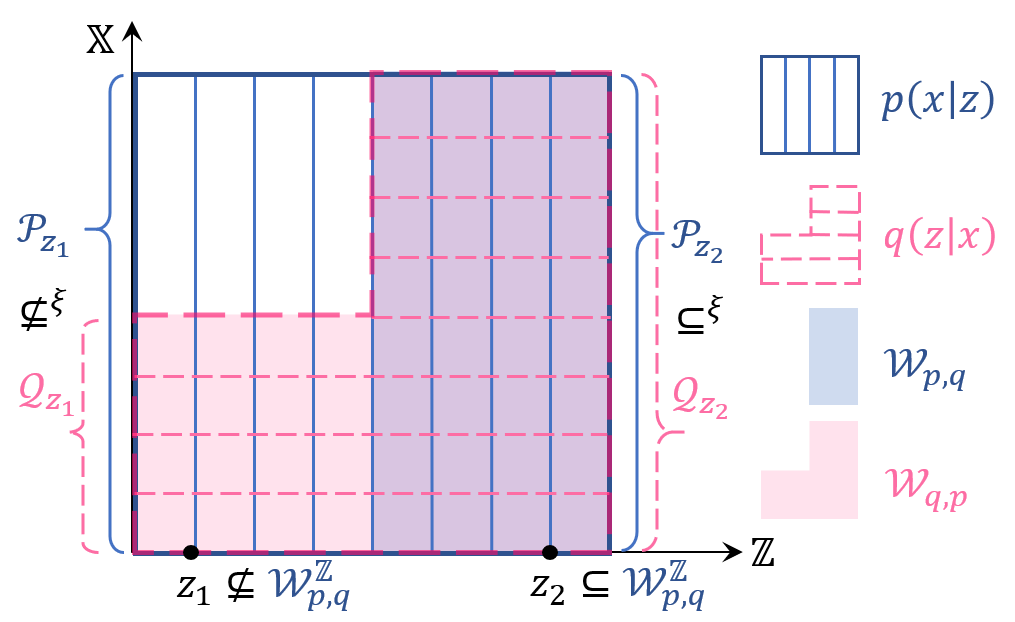}
  \vspace{-4pt}
  \caption{Illustration of Example~\ref{exmp:no-irrcomp}.
    The conditionals are uniform on the respective depicted slices. 
    On $\clW_{p,q} \cap \clW_{q,p}$, the conditional $q(z|x)$ is not normalized.
  }
  \vspace{-0pt}
  \label{fig:exmp-no-irrcomp}
\end{wrapfigure}
\phantom{a}

\vspace{-24pt}
\begin{example} \label{exmp:no-irrcomp}
  Before presenting the proof, we give an example showing that only being the intersection of $\clW_{p,q}$ and $\clW_{q,p}$ is not sufficient to make a valid support of a compatible joint.
  The example is illustrated in Fig.~\ref{fig:exmp-no-irrcomp}.
  The conditionals are uniform on the respective depicted slices, so conditions \bfiv and \bfv in the theorem (\ref{thm:compt-ac}) are satisfied.
  The sets $\clW_{p,q}$ and $\clW_{q,p}$ are depicted in the figure, and their intersection $\clW_{p,q} \cap \clW_{q,p}$ is the right half.
  Although on $\clW_{p,q} \cap \clW_{q,p}$, the conditionals do not render support conflict,
  the conditional $q(z|x)$ is unnormalized for a given $x$ from the bottom half: it integrates to $1/2$ on $(\clW_{p,q} \cap \clW_{q,p})_x$.
  This means that any compatible joint $\pi$ would also have its conditional $\pi(\cdot|x)$ unnormalized for such $x$, which is impossible.

  One may consider trying $\clW_{q,p}$ as the joint support.
  This renders a support conflict similar to the example in the main text:
  a joint on $\clW_{q,p}$ is required by $p(x|z)$ to also cover the top-left quadrant since $z$ values in the left half are covered,
  but this contradicts with the absence of mass by $q(z|x)$.
  In fact, in this example, there is no $\xi\otimes\zeta$-complete component of both $\clW_{p,q}$ and $\clW_{q,p}$, so the two conditionals are not compatible.
\end{example}

\begin{proof}
  Let $\mu(\clX|z)$ and $\nu(\clZ|x)$ be the two everywhere absolutely continuous conditional distributions of whom $p(x|z)$ and $q(z|x)$ are the density functions.
  We begin with some useful conclusions.

  \bfone By construction, for any $(x,z) \in \clW_{p,q}$, $p(x|z) > 0$.
  For any $z \in \clW_{p,q}^\bbZ$, we have $\xi\{x \in (\clW_{p,q})_z \mid q(z|x) = 0\} = \xi\{x \in \clP_z \mid x \notin \clQ_z\} = \xi(\clP_z \setminus \clQ_z) = 0$,
  which means that $q(z|x) > 0$, $\xi$-a.e. on $(\clW_{p,q})_z$.
  By Lemma~\ref{lem:slice-ae}, we have that $q(z|x) > 0$, $\xi\otimes\zeta$-a.e. on $\clW_{p,q}$.
  Symmetrically, $q(z|x) > 0$ on $\clW_{q,p}$, and $p(x|z) > 0$, $\xi\otimes\zeta$-a.e. on $\clW_{q,p}$.
  Particularly, the ratio $\frac{p(x|z)}{q(z|x)}$ is well-defined and is positive and finite, both $\xi\otimes\zeta$-a.e. on $\clW_{p,q}$ and $\xi\otimes\zeta$-a.e. on $\clW_{q,p}$.
  The conclusions also hold ($\xi\otimes\zeta$-a.e.) on any ($\xi\otimes\zeta$-a.s.) subset of $\clW_{p,q}$ or $\clW_{q,p}$.

  \textbf{``Only if'' (necessity):}

  Let $\pi$ be a compatible joint distribution of such conditional distributions $\mu(\cdot|z)$ and $\nu(\cdot|x)$.

  \bftwo Since ``for any'' indicates ``a.e.'' under any measure, the conditions in Lemma~\ref{lem:joint-ac} are satisfied, so we have $\pi \ll \xi\otimes\zeta$.
  By Lemma~\ref{lem:marg-ac}, we also have $\pi^\bbX \ll \xi$ and $\pi^\bbZ \ll \zeta$.
  So there exist density functions (R-N derivatives; \citealp[Thm.~32.2]{billingsley2012probability}) $u(x)$ and $v(z)$ such that for any $\clX \in \scX$ and $\clZ \in \scZ$, $\pi^\bbX(\clX) = \int_\clX u(x) \xi(\ud x)$ and $\pi^\bbZ(\clZ) = \int_\clZ v(z) \zeta(\ud z)$.
  This $u(x)$ is obviously $\xi$-integrable on any measurable subset of $\bbX$, since the integral is no larger (since $u$ is nonnegative) than $\int_\bbX u(x) \xi(\ud x) = 1$ which is finite.

  \bfthr By the definition of conditional distribution \eqref{eqn:cond-int-prod-indep}, for any $\clX\times\clZ \in \scX\times\scZ$, we have
  $\pi(\clX\times\clZ) = \int_\clZ \mu(\clX|z) \pi^\bbZ(\ud z) = \int_\clZ \mu(\clX|z) v(z) \zeta(\ud z) = \int_\clZ \int_\clX p(x|z) \xi(\ud x) v(z) \zeta(\ud z) = \int_{\clX\times\clZ} p(x|z) v(z) (\xi\otimes\zeta)(\ud x \ud z)$,
  where in the last equality, we have applied Fubini's theorem \eqref{eqn:fubini}~\citep[Thm.~18.3]{billingsley2012probability}.
  Similarly, $\pi(\clX\times\clZ) = \int_{\clX\times\clZ} q(z|x) u(x) (\xi\otimes\zeta)(\ud x \ud z)$.
  Noting that $\scX\times\scZ$ is the pi-system that generates $\scX\otimes\scZ$, this indicates that $p(x|z) v(z) = q(z|x) u(x)$, $\xi\otimes\zeta$-a.e. on $\bbX\times\bbZ$~\citep[Thm.~16.10(iii)]{billingsley2012probability}\footnote{
    The requirement that $\bbX\times\bbZ$ is a finite or countable union of $\scX\times\scZ$-sets is satisfied, since from the sigma-finiteness of $\xi$ and $\zeta$, $\bbX$ and $\bbZ$ are a finite or countable union of ($\xi$-finite) $\scX$-sets and ($\zeta$-finite) $\scZ$-sets, respectively.
  }. In other words, both $p(x|z) v(z)$ and $q(z|x) u(x)$ are density functions of $\pi$.
  \\ \bfthrone Subsequently, by leveraging Lemma~\ref{lem:slice-ae}, we have
  for $\zeta$-a.e. $z$ on $\bbZ$, $p(x|z) v(z) = q(z|x) u(x)$, $\xi$-a.e. on $\bbX$.

  \bffor Let $\clU := \{x \mid u(x) > 0\}$ and $\clV := \{z \mid v(z) > 0\}$, and define:
  \begin{align} \label{eqn:def-nece-S}
    \clS := (\clU\times\clV) \cap \clW_{p,q}, \;\;
    \clSt := (\clU\times\clV) \cap \clW_{q,p},
  \end{align}
  Since $u$ and $v$ are integrable thus measurable and $\bbR^{>0}$ is Lebesgue-measurable, we know that $\clU \in \scX$ and $\clV \in \scZ$ are also measurable.
  So $\clS$ and $\clSt$ are measurable.

  \bfforone We can verify that $\clS$ is a $\xi\otimes\zeta$-complete component of $\clW_{p,q}$.
  Since $\clS \subseteq \clW_{p,q}$, we only need to verify that:
  \begin{flalign}
    & (\xi\otimes\zeta) \big( \big[ (\clS^\bbX \times \bbZ \cup \bbX \times \clS^\bbZ) \cap \clW_{p,q} \big] \setminus \clS \big)
    \intertext{(Since clearly $\clS^\bbX \subseteq \clU$, $\clS^\bbZ \subseteq \clV$, and measures are monotone,)}
    \le{} & (\xi\otimes\zeta) \big( \big[ (\clU \times \bbZ \cup \bbX \times \clV) \cap \clW_{p,q} \big] \setminus \clS \big)
    & \hspace{-36in} \text{(Since $(\clA \cap \clC) \setminus (\clB \cap \clC) = (\clA \setminus \clB) \cap \clC$,)} \\
    ={} & (\xi\otimes\zeta) \big( \big[ (\clU \times \bbZ \cup \bbX \times \clV) \setminus (\clU \times \clV) \big] \cap \clW_{p,q} \big)
    & \hspace{-36in} \text{(Since $(\clA \cup \clB) \setminus \clC = (\clA \setminus \clC) \cup (\clB \setminus \clC)$,)} \\
    ={} & (\xi\otimes\zeta) \big( \big[ ((\clU \times \bbZ) \setminus (\clU \times \clV)) \cup ((\bbX \times \clV) \setminus (\clU \times \clV)) \big] \cap \clW_{p,q} \big) \\
    ={} & (\xi\otimes\zeta) \big( \big[ \{(x,z) \mid u(x) > 0, v(z) = 0\} \cup \clW_{p,q} \big] \cup \big[ \{(x,z) \mid u(x) = 0, v(z) > 0\} \cup \clW_{p,q} \big] \big) \\
    \intertext{(By Conclusion~\bfthr and adjusting the set by a set of $\xi\otimes\zeta$-measure-zero according to Lemma~\ref{lem:pmmeas0},)}
    ={} & (\xi\otimes\zeta) \big( \big[ \{(x,z) \mid u(x) > 0, v(z) = 0, q(z|x) = 0\} \cup \clW_{p,q} \big] \\*
    & {}\cup \big[ \{(x,z) \mid u(x) = 0, v(z) > 0, p(x|z) = 0\} \cup \clW_{p,q} \big] \big) \\
    \le{} & (\xi\otimes\zeta) \big( \{(x,z) \mid u(x) > 0, v(z) = 0, q(z|x) = 0\} \cup \clW_{p,q} \big) \\*
      &{} + (\xi\otimes\zeta) \big( \{(x,z) \mid u(x) = 0, v(z) > 0, p(x|z) = 0\} \cup \clW_{p,q} \big)
      & \hspace{-36in} \text{(By Conclusion~\bfone,)}
  \end{flalign}
  $=0$.
  Symmetrically, we can also verify that $\clSt$ is a $\xi\otimes\zeta$-complete component of $\clW_{q,p}$.

  \bffortwo We can also show that $\pi(\clS) = \pi(\clSt) = 1$.
  From Conclusion~\bfthr, we have:
  \begin{flalign}
    1 ={} & \pi(\bbX\times\bbZ) = \int_\bbZ \int_\bbX p(x|z) \xi(\ud x) v(z) \zeta(\ud z) \\
    \intertext{(Since the integral on a region with an a.e. zero value is zero~\citep[Thm.~15.2(i)]{billingsley2012probability},)}
    ={} & \int_\clV \int_{\clP_z} p(x|z) \xi(\ud x) v(z) \zeta(\ud z)
    & \hspace{-8in} \text{(Since $\clS^\bbZ \subseteq \clV$ and due to \citet[Thm.~16.9]{billingsley2012probability},)} \\
    ={} & \Big( \! \int_{\clS^\bbZ} \! \int_{\clP_z} \!\! + \! \int_{\clV \setminus \clS^\bbZ} \! \int_{\clP_z} \! \Big) p(x|z) \xi(\ud x) v(z) \zeta(\ud z)
    & \hspace{-8in} \text{(Since $\clV \setminus \clS^\bbZ = \clV \setminus (\clV \cap \clW_{p,q}^\bbZ) = \clV \setminus \clW_{p,q}^\bbZ$,)} \\
    ={} & \Big( \int_{\clS^\bbZ} \int_{\clP_z} + \int_{\clV \setminus \clW_{p,q}^\bbZ} \int_{\clP_z} \Big) p(x|z) \xi(\ud x) v(z) \zeta(\ud z).
    \label{eqn:pi_of_S_is_1-deduction}
  \end{flalign}
  For the second iterated integral, note that for any $z$ on $\clV$, $v(z) > 0$,
  so from Conclusion~\bfthrone, for $\zeta$-a.e. $z$ on $\clV$, we have $p(x|z) > 0 \Longrightarrow q(z|x) > 0$, $\xi$-a.e. on $\bbX$. 
  This means that for $\zeta$-a.e. $z$ on $\clV$, $\xi\{x \mid p(x|z) > 0, q(z|x) = 0\} = \xi(\clP_z \setminus \clQ_z) = 0$.
  So the set $\clV \setminus \clW_{p,q}^\bbZ = \{z \in \clV \mid \xi(\clP_z \setminus \clQ_z) > 0 \text{ or } \clP_z = \emptyset\}$
  has the same measure under $\zeta$ as the set $\{z \in \clV \mid \clP_z = \emptyset\}$.
  This means that the second iterated integral
  $\int_{\clV \setminus \clW_{p,q}^\bbZ} \int_{\clP_z} p(x|z) \xi(\ud x) v(z) \zeta(\ud z)
  = \int_{\{z \in \clV \mid \clP_z = \emptyset\}} \int_{\clP_z} p(x|z) \xi(\ud x) v(z) \zeta(\ud z)
  = 0$~\citep[p.226, Thm.~15.2(i)]{billingsley2012probability}.

  For the first iterated integral, note that by construction, $\clP_z = (\clW_{p,q})_z$ for any $z$ on $\clS^\bbZ$, since $\clS^\bbZ \subseteq \clW_{p,q}^\bbZ$.
  Moreover, from Conclusion~\bfforone that $\clS$ is a $\xi\otimes\zeta$-complete component of $\clW_{p,q}$,
  we can subsequently apply Lemma~\ref{lem:irrcomp-inteq}, so $\int_{\clS^\bbZ} \int_{(\clW_{p,q})_z} p(x|z) \xi(\ud x) v(z) \zeta(\ud z) = \int_\clS p(x|z) v(z) (\xi\otimes\zeta)(\ud x \ud z)$, which is $\pi(\clS)$ by Conclusion~\bfthr.
  This means that $\pi(\clS) = 1$.
  The same deduction applies symmetrically to $\clSt$, so we also have $\pi(\clSt) = 1$.

  \textbf{Main procedure (``only if'').}
  We will verify that the set $\clS$ given in \eqref{eqn:def-nece-S} satisfies all the necessary conditions.

  Conclusion~\bffortwo shows that $\pi(\clS) = 1 > 0$, and in Conclusion~\bftwo we have verified that $\pi \ll \xi\otimes\zeta$.
  So we have $(\xi\otimes\zeta)(\clS) > 0$, which verifies Condition~\bfiii.

  To verify Conditions~\bfiv and~\bfv,
  by Conclusions~\bfthr and~\bfone, we have $p(x|z) v(z) = q(z|x) u(x)$ and $q(z|x) > 0$, $\xi\otimes\zeta$-a.e. on $\clS$.
  By the construction of $\clS$, we also have $v(z) > 0$ and $p(x|z) > 0$ everywhere on $\clS$.
  So the ratio $\frac{p(x|z)}{q(z|x)}$ is finite and positive, $\xi\otimes\zeta$-a.e. on $\clS$,
  and it factorizes as $\frac{p(x|z)}{q(z|x)} = u(x) \frac{1}{v(z)}$, $\xi\otimes\zeta$-a.e. on $\clS$.
  By Conclusion~\bftwo, $a(x) := u(x)$ is $\xi$-integrable on $\clS^\bbX$.

  To verify Condition~\bfii, note that by construction, $\clS^\bbZ \subseteq \clV \cap \clW_{p,q}^\bbZ \subseteq \clW_{p,q}^\bbZ$ so $\clS^\bbZ \assubseteq{\zeta} \clW_{p,q}^\bbZ$.
  Note that $\clS^\bbX = \{x \in \clU \mid \clV \cap (\clW_{p,q})_x \ne \emptyset\}
  = \{x \in \clU \mid \exists z \in \clV \st x \in \clP_z, \clP_z \assubseteq{\xi} \clQ_z\}
  \subseteq \{x \in \clU \mid \exists z \in \clV \st x \in \clP_z\}
  = \{x \mid u(x) > 0, \exists z \st v(z) > 0, p(x|z) > 0\}$.
  By Conclusion~\bfthrone, for $\xi$-a.e. $x$ on $\clS^\bbX$, we have $\exists z \st q(z|x) u(x) = p(x|z) v(z) > 0$, so $q(z|x) > 0$ hence $\clQ_x \ne \emptyset$. 
  Moreover, for $\xi$-a.e. $x$ on $\clS^\bbX$ and $\zeta$-a.e. $z$ on $\clQ_x$, we have $p(x|z) v(z) = q(z|x) u(x) > 0$, so $p(x|z) > 0$ hence $\clQ_x \assubseteq{\zeta} \clP_x$.
  These two conclusions means that for $\xi$-a.e. $x$ on $\clS^\bbX$, we have $x \in \{x \mid \clQ_x \ne \emptyset, \clQ_x \assubseteq{\zeta} \clP_x\}$ which is exactly $\clW_{q,p}^\bbX$.
  Hence $\clS^\bbX \assubseteq{\xi} \clW_{q,p}^\bbX$.

  Now, all the left is to verify Condition~\bfi.
  Conclusion~\bfforone has verified that $\clS$ is a $\xi\otimes\zeta$-complete component of $\clW_{p,q}$.
  To verify that $\clS$ is a $\xi\otimes\zeta$-complete component also of $\clW_{q,p}$, note that Conclusion~\bfforone has also verified that $\clSt$ is a $\xi\otimes\zeta$-complete component of $\clW_{q,p}$,
  so by Lemma~\ref{lem:irrcomp-assame} it suffices to verify that $\clS \aseq{\xi\otimes\zeta} \clSt$, $\clS^\bbX \aseq{\xi} \clSt^\bbX$ and $\clS^\bbZ \aseq{\zeta} \clSt^\bbZ$.

  By construction, for any $(x,z) \in \clS$, we have $v(z) > 0$ and $p(x|z) > 0$, so from Conclusion~\bfthr the density function $p(x|z) v(z)$ of $\pi$ is positive.
  Similarly, the density function $q(z|x) u(x)$ of $\pi$ is positive everywhere on $\clSt$.
  Moreover, Conclusion~\bffortwo has shown that both $\pi(\clS) = 1$ and $\pi(\clSt) = 1$.
  So by Lemma~\ref{lem:support-asunique}, we know that $\clS \aseq{\xi\otimes\zeta} \clSt$.
  Also by construction, the density function $u(x)$ of $\pi^\bbX$ is positive everywhere on $\clS^\bbX$ and on $\clSt^\bbX$.
  Moreover, we have $\pi^\bbX(\clS^\bbX) = \pi(\clS^\bbX \times \bbZ) \ge \pi(\clS) = 1$ so $\pi^\bbX(\clS^\bbX) = 1$ and similarly $\pi^\bbX(\clSt^\bbX) = 1$.
  Again by Lemma~\ref{lem:support-asunique}, we have $\clS^\bbX \aseq{\xi} \clSt^\bbX$.
  It follows similarly that $\clS^\bbZ \aseq{\zeta} \clSt^\bbZ$.

  \textbf{``If'' (sufficiency):}

  \bffiv For Conditions~\bfiv and~\bfv, denote $\aat(x) := \lrvert{a(x)}$ and $\bbt(z) := \lrvert{b(z)}$, 
  and let $\At := \int_{\clS^\bbX} \aat(x) \xi(\ud x)$.
  \\ \bffivone From the definition of integrability in Supplement~\ref{supp:meas-int}, Condition~\bfv is equivalent to that $\aat(x)$ is $\xi$-integrable on $\clS^\bbX$.
  Particularly, $\At < \infty$.
  \\ Due to Condition~\bfi and Lemma~\ref{lem:irrcomp-assubseteq}, we have $\clS \assubseteq{\xi\otimes\zeta} \clW_{p,q}$.
  So by Condition~\bfiv and Conclusion~\bfone, we have $\frac{p(x|z)}{q(z|x)} = a(x) b(z) > 0$, $\xi\otimes\zeta$-a.e. on $\clS$.
  This means both $(\xi\otimes\zeta)\{(x,z) \in \clS \mid a(x) b(z) = 0\} = 0$ and $(\xi\otimes\zeta)\{(x,z) \in \clS \mid a(x) b(z) < 0\} = 0$, since their summation is zero.
  \\ \bffivtwo Since $\{(x,z) \mid x \in \clS^\bbX, a(x) = 0, z \in \clS_x\} \subseteq \{(x,z) \in \clS \mid a(x) b(z) = 0\}$,
  the second-last equation in Conclusion~\bffivone above means that $(\xi\otimes\zeta)\{(x,z) \mid x \in \clS^\bbX, a(x) = 0, z \in \clS_x\} = 0$.
  So $(\xi\otimes\zeta)\{(x,z) \mid x \in \clS^\bbX, a(x) \ne 0, z \in \clS_x\}
  = (\xi\otimes\zeta)\{(x,z) \mid x \in \clS^\bbX, z \in \clS_x\} - (\xi\otimes\zeta)\{(x,z) \mid x \in \clS^\bbX, a(x) = 0, z \in \clS_x\}
  = (\xi\otimes\zeta)(\clS) > 0$ by Condition~\bfiii.
  Moreover, by \eqref{eqn:prod-meas-restr}, we have $(\xi\otimes\zeta)\{(x,z) \mid x \in \clS^\bbX, a(x) \ne 0, z \in \clS_x\}
  = \int_{\{x \in \clS^\bbX \mid a(x) \ne 0\}} \zeta(\clS_x) \xi(\ud x) > 0$,
  so $\xi\{x \in \clS^\bbX \mid a(x) \ne 0\} > 0$~\citep[p.226]{billingsley2012probability}.
  Particularly, $\At > 0$~\citep[Thm.~15.2(ii)]{billingsley2012probability}.
  \\ \bffivthr Since $\{(x,z) \in \clS \mid a(x) b(z) \ne \aat(x) \bbt(z)\} = \{(x,z) \in \clS \mid a(x) b(z) < 0\}$, the last equation in Conclusion~\bffivone above means that $a(x) b(z) = \aat(x) \bbt(z)$, $\xi\otimes\zeta$-a.e. on $\clS$.
  So we have $\frac{p(x|z)}{q(z|x)} = \aat(x) \bbt(z)$, $\xi\otimes\zeta$-a.e. on $\clS$, following Condition~\bfiv.

  \bfsix Based on Conclusions~\bffivone and~\bffivtwo, we can define the following finite and nonnegative functions on $\bbX$ and $\bbZ$:
  \begin{align}
    u(x) := \begin{cases}
      \frac{1}{\At} \aat(x), & \text{if $x \in \clS^\bbX$ and $\aat(x) > 0$}, \\
      0, & \text{otherwise},
    \end{cases} \;\;
    v(z) := \begin{cases}
      \frac{1}{\At \bbt(z)}, & \text{if $z \in \clS^\bbZ$ and $\bbt(z) > 0$}, \\
      0, & \text{otherwise}.
    \end{cases}
  \end{align}
  \bfsixone By construction, $\aat(x) \bbt(z) = u(x) / v(z)$ on $\clS \cap \{(x,z) \mid \bbt(z) > 0\}$.
  So from Conclusion~\bffivthr, we have $p(x|z) v(z) = q(z|x) u(x)$, $\xi\otimes\zeta$-a.e. on $\clS \cap \{(x,z) \mid \bbt(z) > 0\}$.
  Moreover, following a similar deduction as in Conclusion~\bffivtwo, we know that $(\xi\otimes\zeta)(\clS \setminus \{(x,z) \mid \bbt(z) > 0\})
  = (\xi\otimes\zeta)\{(x,z) \mid z \in \clS^\bbZ, b(z) = 0, x \in \clS_z\} = 0$.
  So we have $p(x|z) v(z) = q(z|x) u(x)$, $\xi\otimes\zeta$-a.e. on $\clS$.
  \\ \bfsixtwo By construction, $\int_\bbX u(x) \xi(\ud x) = \int_{\clS^\bbX} u(x) \xi(\ud x) = 1$.

  \textbf{Main procedure (``if'').}
  We will show that the following function on $\scX\otimes\scZ$, which is the same as \eqref{eqn:def-suff-pi-unroll} in the theorem,
  is a distribution on $\bbX\times\bbZ$ such that $\mu(\clX|z)$ and $\nu(\clZ|x)$ are its conditional distributions:
  \begin{align} \label{eqn:def-suff-pi}
    \pi(\clW) := \int_{\clW \cap \clS} q(z|x) u(x) (\xi\otimes\zeta)(\ud x \ud z), \;\;
    \forall \clW \in \scX\otimes\scZ.
  \end{align}
  Consider any measurable rectangle $\clX\times\clZ \in \scX\times\scZ$. We have:
  \begin{flalign}
    & \pi(\clX\times\clZ) = \int_{\clX\times\clZ \cap \clS} q(z|x) u(x) (\xi\otimes\zeta)(\ud x \ud z)
    \intertext{(By the generalized form \eqref{eqn:fubini-restr} of Fubini's theorem,)}
    ={} & \int_{\clX \cap \clS^\bbX} \int_{\clZ \cap \clS_x} q(z|x) \zeta(\ud z) u(x) \xi(\ud x)
    & \hspace{-8in} \text{(By Condition~\bfi and applying Lemma~\ref{lem:irrcomp-inteq},)} \\
    ={} & \int_{\clX \cap \clS^\bbX} \int_{\clZ \cap (\clW_{q,p})_x} q(z|x) \zeta(\ud z) u(x) \xi(\ud x)
    \intertext{(Since $\clS^\bbX \assubseteq{\xi} \clW_{q,p}^\bbX$ by Condition~\bfii and $(\clW_{q,p})_x = \clQ_x$ on $\clW_{q,p}^\bbX$,)}
    ={} & \int_{\clX \cap \clS^\bbX} \int_{\clZ \cap \clQ_x} q(z|x) \zeta(\ud z) u(x) \xi(\ud x)
    \intertext{(Since by construction, $u(x) = 0$ outside $\clS^\bbX$ and $q(z|x) = 0$ outside $\clQ_x$,)}
    ={} & \int_\clX \int_\clZ q(z|x) \zeta(\ud z) u(x) \xi(\ud x)
    & \hspace{-8in} \text{(Recalling that $q(z|x)$ is the density function of $\nu(\cdot|x)$,)} \\
    ={} & \int_\clX \nu(\clZ|x) u(x) \xi(\ud x).
    \label{eqn:pi-is-int-over-nu}
  \end{flalign}
  Moreover, due to Conclusion~\bfsixone, we have $\pi(\clW) = \int_{\clW \cap \clS} p(x|z) v(z) (\xi\otimes\zeta)(\ud x \ud z)$ on $\scX\otimes\scZ$~\citep[Thm.~15.2(v)]{billingsley2012probability}.
  Using this form of $\pi$ and noting that the symmetrized conditions in the above deduction also hold, we have:
  \begin{align} \label{eqn:pi-is-int-over-mu}
    \pi(\clX\times\clZ) = \int_\clZ \mu(\clX|z) v(z) \zeta(\ud z).
  \end{align}

  Since both $q(z|x)$ and $u(x)$ are finite and nonnegative on $\bbX\times\bbZ$, $\pi$ is a measure on $\scX\otimes\scZ$ by its definition \eqref{eqn:def-suff-pi}~\citep[p.227]{billingsley2012probability}.
  Moreover, from \eqref{eqn:pi-is-int-over-nu}, we have $\pi(\bbX\times\bbZ) = \int_\bbX u(x) \xi(\ud x) = 1$ by Conclusion~\bfsixtwo.
  So $\pi$ is a distribution (probability measure) on $\bbX\times\bbZ$.

  From \eqref{eqn:pi-is-int-over-nu}, we have $\pi^\bbX(\clX) := \pi(\clX\times\bbZ) = \int_\clX u(x) \xi(\ud x)$.
  So $u(x)$ is a density function of $\pi^\bbX$, and \eqref{eqn:pi-is-int-over-nu} in turn becomes $\pi(\clX\times\clZ) = \int_\clX \nu(\clZ|x) \pi^\bbX(\ud x)$.
  This indicates that $\nu(\clZ|x)$ is a conditional distribution of $\pi$ w.r.t sub-sigma-field $\scX \times \{\clZ\}$ for any $\clZ \in \scZ$, due to \eqref{eqn:cond-int-prod-indep}.

  Similarly, from \eqref{eqn:pi-is-int-over-mu}, we have $\pi^\bbZ(\clZ) := \pi(\bbX\times\clZ) = \int_\clZ v(z) \zeta(\ud z)$.
  So $v(z)$ is a density function of $\pi^\bbZ$, and \eqref{eqn:pi-is-int-over-mu} in turn becomes $\pi(\clX\times\clZ) = \int_\clZ \mu(\clX|z) \pi^\bbZ(\ud z)$.
  This indicates that $\mu(\clX|z)$ is a conditional distribution of $\pi$ w.r.t sub-sigma-field $\{\clX\} \times \scZ$ for any $\clX \in \scX$, due to \eqref{eqn:cond-int-prod-indep}.
  The proof is completed.
\end{proof}

\subsection{Complete Support Proposition under the a.e.-Full Support Condition} \label{supp:proofs-full-supp-irrcomp}

\begin{proposition} \label{prop:full-supp-irrcomp}
  If $p(x|z)$ and $q(z|x)$ have a.e.-full supports, then $\clW_{p,q} \aseq{\xi\otimes\zeta} \clW_{q,p} \aseq{\xi\otimes\zeta} \bbX\times\bbZ$,
  and $\bbX\times\bbZ$ is the $\xi\otimes\zeta$-unique complete support of them when compatible.
\end{proposition}
\begin{proof}
  By definition, $p(x|z) > 0$ and $q(z|x) > 0$, $\xi\otimes\zeta$-a.e.
  By Lemma~\ref{lem:slice-ae}, this means that for $\xi$-a.e. $x$, $p(x|z) > 0$ $\zeta$-a.e. thus $\clP_x \aseq{\zeta} \bbZ$, and for $\zeta$-a.e. $z$, $\clQ_z \aseq{\xi} \bbX$.
  Similarly, we also have for $\xi$-a.e. $x$, $\clQ_x \aseq{\zeta} \bbZ$ thus $\clQ_x \aseq{\zeta} \clP_x$ by the transitivity (Lemma~\ref{lem:assame-equiv}) and subsequently $\clQ_x \assubseteq{\zeta} \clP_x$.
  Similarly, for $\zeta$-a.e. $z$, $\clP_z \assubseteq{\xi} \clQ_z$.
  This means that for $\xi$-a.e. $x$, $(\clW_{q,p})_x = \clQ_x \aseq{\zeta} \bbZ$ thus $z \in (\clW_{q,p})_x$, $\zeta$-a.e.
  By Lemma~\ref{lem:slice-ae}, this means that for $\xi\otimes\zeta$-a.e. $(x,z)$, we have $(x,z) \in \clW_{q,p}$, so $\clW_{q,p} \aseq{\xi\otimes\zeta} \bbX\times\bbZ$.
  Similarly, we have $\clW_{p,q} \aseq{\xi\otimes\zeta} \bbX\times\bbZ$.

  Let $\clS$ be a complete support of $p(x|z)$ and $q(z|x)$ when they are compatible.
  Since $(\xi\otimes\zeta)(\clS) = 1 > 0$ from Condition~\bfiii in Theorem~\ref{thm:compt-ac}, we know that $\xi(\clS^\bbX) > 0$ by \eqref{eqn:prod-meas-restr} and \citet[p.226]{billingsley2012probability}.
  So there is an $x \in \clS^\bbX$ such that $(\clW_{q,p})_x \aseq{\zeta} \bbZ$, otherwise there would be a non-measure-zero set of $x$ violating $(\clW_{q,p})_x \aseq{\zeta} \bbZ$.
  As a $\xi\otimes\zeta$-complete component of $\clW_{q,p}$ by Condition~\bfi in Theorem~\ref{thm:compt-ac}, we have $\clS \aseq{\xi\otimes\zeta} (\clS^\bbX \times \bbZ \cap \clW_{q,p}) \cup (\bbZ \times \clS^\bbZ \cap \clW_{q,p})
  \supseteq \clS^\bbX \times \bbZ \cap \clW_{q,p} \aseq{\xi\otimes\zeta} \clS^\bbX \times \bbZ$.
  This means that $(\xi\otimes\zeta)(\clS^\bbX \times \bbZ \setminus \clS) = \int_{\clS^\bbX} \zeta(\bbZ \setminus \clS_x) \xi(\ud x) = 0$ by \eqref{eqn:prod-meas-restr},
  so for $\xi$-a.e. $x$ on $\clS^\bbX$, $\zeta(\bbZ \setminus \clS_x) = 0$~\citep[Thm.~15.2(ii)]{billingsley2012probability} thus $\clS_x \aseq{\zeta} \bbZ$.
  So $\clS^\bbZ \aseq{\zeta} \bbZ$.
  Moreover, we also have $\clS \assupseteq{\xi\otimes\zeta} \bbX \times \clS^\bbZ$ which $\aseq{\xi\otimes\zeta} \bbX\times\bbZ$, so we have $\clS \aseq{\xi\otimes\zeta} \bbX\times\bbZ$.
  Similarly, from that $\clS$ is a $\xi\otimes\zeta$-complete component also of $\clW_{p,q}$, we have the same conclusion.
\end{proof}

\subsection{Proof of Theorem~\ref{thm:determ-ac}} \label{supp:proofs-determ-ac}
\begin{proof}
  Let $\pi$ and $\pit$ be two compatible joints of $p(x|z)$ and $q(z|x)$, and they are supported on the same complete support $\clS$.
  By Conclusions~\bftwo and~\bfthr in the proof (Supplement~\ref{supp:proofs-compt-ac}) of Theorem~\ref{thm:compt-ac}, there exist functions $u(x)$, $v(z)$ and $\uut(x)$, $\vvt(z)$ such that
  $p(x|z) v(z)$ and $q(z|x) u(x)$ are the densities of $\pi$, and $p(x|z) \vvt(z)$ and $q(z|x) \uut(x)$ of $\pit$,
  and $p(x|z) v(z) = q(z|x) u(x)$ and $p(x|z) \vvt(z) = q(z|x) \uut(x)$, $\xi\otimes\zeta$-a.e.
  By the definition of a support in Lemma~\ref{lem:support-asunique}, we know that the densities of $\pi$ and $\pit$ are positive $\xi\otimes\zeta$-a.e. on $\clS$,
  and $\int_{\clS^\bbX} u \dd \xi = \int_{\clS^\bbX} \uut \dd \xi = \int_{\clS^\bbZ} v \dd \zeta = \int_{\clS^\bbZ} \vvt \dd \zeta = 1$.

  Consequently, we have $\frac{p(x|z)}{q(z|x)} = \frac{u(x)}{v(z)} = \frac{\uut(x)}{\vvt(z)}$, $\xi\otimes\zeta$-a.e. on $\clS$.
  By Lemma~\ref{lem:slice-ae}, for $\zeta$-a.e. $z$ on $\clS^\bbZ$, we have $\frac{u(x)}{v(z)} = \frac{\uut(x)}{\vvt(z)}$ for $\xi$-a.e. $x$ on $\clS_z$,
  which means that $\int_{\clS_z} \frac{u(x)}{v(z)} \xi(\ud x) = \int_{\clS_z} \frac{\uut(x)}{\vvt(z)} \xi(\ud x)$.
  Since for $\zeta$-a.e. $z$ on $\clS^\bbZ$, $\clS_z \aseq{\xi} \clS^\bbX$, we have by Lemma~\ref{lem:assame-meas} that
  $\int_{\clS^\bbX} \frac{u(x)}{v(z)} \xi(\ud x) = \int_{\clS^\bbX} \frac{\uut(x)}{\vvt(z)} \xi(\ud x)$, which in turn gives
  $\frac{1}{v(z)} \int_{\clS^\bbX} u \dd \xi = \frac{1}{v(z)} = \frac{1}{\vvt(z)} \int_{\clS^\bbX} \uut \dd \xi = \frac{1}{\vvt(z)}$.
  So $v(z) = \vvt(z)$ for $\zeta$-a.e. $z$ on $\clS^\bbZ$, and similarly $u(x) = \uut(x)$ for $\xi$-a.e. $x$ on $\clS^\bbX$.
  Subsequently, the density $p(x|z) v(z)$ or $q(z|x) u(x)$ of $\pi$ is $\xi\otimes\zeta$-a.e. the same as the density $p(x|z) \vvt(z)$ or $q(z|x) \uut(x)$ of $\pit$.
  Hence, $\pi$ and $\pit$ are the same distribution.
\end{proof}

\subsection{The Dirac Compatibility Lemma} \label{supp:proofs-compt-dirac-exist}

Before proving the main instructive compatibility theorem~(\ref{thm:compt-dirac}) for the Dirac case,
we first present an existential equivalent criterion for compatibility, which provides insights to the problem.

\begin{lemma}[Dirac compatibility, existential] \label{lem:compt-dirac-exist}
  Conditional distribution $\nu(\clZ|x)$ is compatible with $\mu(\clX|z) := \delta_{f(z)}(\clX)$ where function $f: \bbZ \to \bbX$ is $\scX$/$\scZ$-measurable,
  if and only if there is a distribution $\beta$ on $(\bbZ,\scZ)$ such that $\nu(\clZ|x) = \frac{\ud \beta(\clZ \cap f^{-1}(\cdot))}{\ud \beta(f^{-1}(\cdot))}(x)$,
  and this $\beta$ is the marginal $\pi^\bbZ$ of a compatible joint $\pi$ of them.
\end{lemma}
\begin{proof}
  We first show the validity of the R-N derivative.
  Since $\beta$ is a distribution thus a finite measure, $\beta(\clZ \cap f^{-1}(\cdot))$ and $\beta(f^{-1}(\cdot))$ are also finite thus sigma-finite.
  For any $\clX \in \scX$ such that $\beta(f^{-1}(\clX)) = 0$, we have $\beta(\clZ \cap f^{-1}(\clX)) \le \beta(f^{-1}(\clX)) = 0$
  since $\clZ \cap f^{-1}(\clX) \subseteq f^{-1}(\clX)$ and measures are monotone.
  So $\beta(\clZ \cap f^{-1}(\clX)) = 0$ and $\beta(\clZ \cap f^{-1}(\cdot)) \ll \beta(f^{-1}(\cdot))$.
  By the R-N theorem~\citep[Thm.~32.2]{billingsley2012probability}, the R-N derivative exists.

  \textbf{``Only if'' (necessity):}
  Let $\pi$ be a compatible joint.
  Since $\mu$ and $\nu$ are its conditional distributions, by \eqref{eqn:cond-int-prod-indep}, we have:
  \begin{align}
    \pi(\clX\times\clZ) = \int_\clZ \mu(\clX|z) \pi^\bbZ(\ud z) = \int_\clX \nu(\clZ|x) \pi^\bbX(\ud x), \;\;
    \forall \clX\times\clZ \in \scX\times\scZ.
  \end{align}
  The first integral is $\int_\clZ \bbI[f(z) \in \clX] \pi^\bbZ(\ud z) = \int_\clZ \bbI[z \in f^{-1}(\clX)] \pi^\bbZ(\ud z) = \pi^\bbZ(\clZ \cap f^{-1}(\clX))$.
  Particularly, $\pi^\bbX(\clX) = \pi(\clX\times\bbZ) = \pi^\bbZ(f^{-1}(\clX))$, \ie $\pi^\bbX$ is the transformed (pushed-forward) distribution from $\pi^\bbZ$ by measurable function $f$~\citep[p.196]{billingsley2012probability}.
  On the other hand, the equality to the second integral means that $\pi^\bbZ(\clZ \cap f^{-1}(\clX)) = \int_\clX \nu(\clZ|x) \pi^\bbX(\ud x) = \int_\clX \nu(\clZ|x) \pi^\bbZ(f^{-1}(\ud x))$.
  This means that $\nu(\clZ|x)$ is the R-N derivative of $\clX \mapsto \pi^\bbZ(\clZ \cap f^{-1}(\clX))$ w.r.t $\clX \mapsto \pi^\bbZ(f^{-1}(\clX))$.
  Taking $\beta$ as $\pi^\bbZ$, which is a distribution on $(\bbZ,\scZ)$, yields the necessary condition.

  \textbf{``If'' (sufficiency):}
  For any measurable rectangle $\clX\times\clZ \in \scX\times\scZ$, define
  $\pi(\clX\times\clZ) := \int_\clZ \mu(\clX|z) \beta(\ud z)$ and $\pit(\clX\times\clZ) := \beta(\clZ \cap f^{-1}(\clX))$.
  Since for any $z \in \bbZ$, $f(z) \in \bbX$, so $\pi(\bbX\times\bbZ) = \int_\bbZ \mu(\bbX|z) \beta(\ud z) = \int_\bbZ \beta(\ud z) = 1$.
  Since $f^{-1}(\bbX) = \bbZ$, we have $\pit(\bbX\times\bbZ) = \beta(\bbZ) = 1$.
  So both $\pi$ and $\pit$ are finite thus sigma-finite.
  Moreover, for any $\clX\times\clZ \in \scX\times\scZ$, we have
  $\pi(\clX\times\clZ) = \int_\clZ \bbI[f(z) \in \clX] \beta(\ud z) = \int_\clZ \bbI[z \in f^{-1}(\clX)] \beta(\ud z) = \int_{\clZ \cap f^{-1}(\clX)} \beta(\ud z) = \beta(\clZ \cap f^{-1}(\clX)) = \pit(\clX\times\clZ)$,
  \ie $\pi$ and $\pit$ agree on the pi-system $\scX\times\scZ$.
  So by \citet[Thm.~10.3]{billingsley2012probability}, $\pi$ and $\pit$ extend to the same distribution (probability measure) on $(\bbX\times\bbZ, \scX\otimes\scZ)$.

  On the other hand, we have $\pi^\bbZ(\clZ) = \pi(\bbX\times\clZ) = \int_\clZ \mu(\bbX|z) \beta(\ud z) = \int_\clZ \beta(\ud z) = \beta(\clZ)$,
  and furthermore from this, $\mu$ is a conditional distribution of $\pi$ due to its construction and \eqref{eqn:cond-int-prod-indep}.
  Moreover, $\pit^\bbX(\clX) = \pit(\clX\times\bbZ) = \beta(\bbZ \cap f^{-1}(\clX)) = \beta(f^{-1}(\clX))$,
  and by the definition of $\nu(\clZ|x)$ as an R-N derivative, we have $\beta(\clZ \cap f^{-1}(\clX)) = \int_\clX \nu(\clZ|x) \beta(f^{-1}(\ud x))$,
  which is $\pit(\clX\times\clZ) = \int_\clX \nu(\clZ|x) \pit^\bbX(\ud x)$.
  So again due to \eqref{eqn:cond-int-prod-indep}, $\nu$ is a conditional distribution of $\pit$.
  Since $\pi$ and $\pit$ are the same distribution on $(\bbX\times\bbZ, \scX\otimes\scZ)$, we know that $\mu$ and $\nu$ are compatible.
\end{proof}

\textbf{Key insights.} \hspace{4pt}
Let $\pi$ be a compatible joint of $\mu(\clX|z) := \delta_{f(z)}(\clX)$ and $\nu(\clZ|x)$.
For any $\clX\times\clZ \in \scX\times\scZ$, we have:
\begin{align}
  \pi(\clX\times\clZ) = \pi^\bbZ(\clZ \cap f^{-1}(\clX))
  = \int_\clX \nu(\clZ|x) \pi^\bbZ(f^{-1}(\ud x))
  = \int_{f^{-1}(\clX)} \nu(\clZ|f(z)) \pi^\bbZ(\ud z),
\end{align}
where the last equality holds due to the rule of change of variables~\citep[Thm.~16.13]{billingsley2012probability}.
Let $f^{-1}(\scX) := \sigma(\{f^{-1}(\clX) \mid \clX \in \scX\})$ be the pulled-back sigma-field from $\scX$ by $f$.
It is a sub-sigma-field of $\scZ$ as every $f^{-1}(\clX) \in \scZ$ since $f$ is measurable.
So the last equality means that:
\begin{align}
  \nu(\clZ|f(z)) = \frac{\ud \pi^\bbZ(\clZ \cap \cdot)}{\ud \pi^\bbZ(\cdot)} \bigg|_{f^{-1}(\scX)} (z).
\end{align}
The expression on the left makes sense since for all values of $z$ that yield the same value of $f(z)$, the R-N derivative is the same.
The second equality also gives:
\begin{align}
  \nu(\clZ|x) = \frac{\ud \pi^\bbZ(\clZ \cap f^{-1}(\cdot))}{\ud \pi^\bbZ(f^{-1}(\cdot))} \bigg|_\scX (x).
\end{align}

\subsection{Proof of Theorem~\ref{thm:compt-dirac}} \label{supp:proofs-compt-dirac}

\begin{proof}
  \textbf{``Only if'' (necessity):}
  Suppose that $\nu(\clZ|x)$ and $\mu(\clX|z) := \delta_{f(z)}(\clX)$ are compatible but for \emph{any} $x \in \bbX$, $\nu(f^{-1}(\{x\}) | x) < 1$.
  Consider the set $\clS := \{(f(z),z) \mid z \in \bbZ\}$.
  Since $f$ is $\scX$/$\scZ$-measurable, this set $\clS$ is $\scX\otimes\scZ$-measurable.
  It is also easy to verify that $\clS_z = \{f(z)\}$ and $\clS_x = f^{-1}(\{x\})$.
  Now let $\pi$ be \emph{any} of their compatible joint distribution.
  From \eqref{eqn:cond-int-prod-dep}, we know that
  $\pi(\clS) = \int_\bbZ \mu(\clS_z|z) \pi^\bbZ(\ud z) = \int_\bbZ \delta_{f(z)}(\{f(z)\}) \pi^\bbZ(\ud z) = \int_\bbZ \pi^\bbZ(\ud z) = \pi^\bbZ(\bbZ) = 1$.
  On the other hand, also from \eqref{eqn:cond-int-prod-dep} and due to the compatibility, we have
  $\pi(\clS) = \int_\bbX \nu(\clS_x|x) \pi^\bbX(\ud x) = \int_\bbX \nu(f^{-1}(\{x\}) | x) \pi^\bbX(\ud x) < \int_\bbX \pi^\bbX(\ud x) = \pi^\bbX(\bbX) = 1$,
  which leads to a contradiction.
  So if $\nu(\clZ|x)$ and $\mu(\clX|z)$ are compatible, then there is $x_0 \in \bbX$ such that $\nu(f^{-1}(\{x_0\}) | x_0) = 1$.

  \textbf{``If'' (sufficiency):}
  Let $\beta(\clZ) := \nu(f^{-1}(\{x_0\}) \cap \clZ | x_0)$ be a set function on $\scZ$.
  We can verify that this $\beta$ is a distribution (probability measure) on $(\bbZ,\scZ)$ since $\nu(\cdot|x_0)$ is.
  Particularly, since $f$ is $\scX$/$\scZ$-measurable and $\{x_0\} \in \scX$ due to the assumption, we know that $f^{-1}(\{x_0\})$ thus $f^{-1}(\{x_0\}) \cap \clZ$ for any $\clZ \in \scZ$ are in $\scZ$;
  $\beta(\emptyset) = \nu(\emptyset|x_0) = 0$;
  $\beta(\bbZ) = \nu(f^{-1}(\{x_0\}) | x_0) = 1$ according to the assumption;
  for any countable disjoint $\scZ$-sets $\clZ^{(1)}, \clZ^{(2)}, \cdots$, it holds that $\clZ^{(1)} \cap f^{-1}(\{x_0\}), \clZ^{(2)} \cap f^{-1}(\{x_0\}), \cdots$ are also disjoint $\scZ$-sets, so
  $\beta(\bigcup_{i=1}^\infty \clZ^{(i)}) = \nu(f^{-1}(\{x_0\}) \cap \bigcup_{i=1}^\infty \clZ^{(i)} | x_0)
  = \nu(\bigcup_{i=1}^\infty f^{-1}(\{x_0\}) \cap \clZ^{(i)} | x_0) = \sum_{i=1}^\infty \nu(f^{-1}(\{x_0\}) \cap \clZ^{(i)} | x_0) = \sum_{i=1}^\infty \beta(\clZ^{(i)})$.

  Now we prove that $\beta(\clZ \cap f^{-1}(\clX)) = \int_\clX \nu(\clZ|x) \beta(f^{-1}(\ud x)), \forall \clX\times\clZ \in \scX\times\scZ$ 
  which is sufficient due to Lemma~\ref{lem:compt-dirac-exist}.
  For any $\clX\times\clZ \in \scX\times\scZ$, the l.h.s is
  $\beta(\clZ \cap f^{-1}(\clX)) = \nu(f^{-1}(\{x_0\}) \cap f^{-1}(\clX) \cap \clZ | x_0) = \nu(f^{-1}(\{x_0\}) \cap \clZ | x_0) \bbI[x_0 \in \clX]$,
  where the last equality holds since $z \in f^{-1}(\{x_0\}) \cap f^{-1}(\clX)$ if and only if $f(z) = x_0 \in \clX$.
  The integral on the r.h.s is $\int_\clX \nu(\clZ|x) \nu(f^{-1}(\{x_0\}) \cap f^{-1}(\ud x) | x_0)$.
  Since the measure $\clX \mapsto \nu(f^{-1}(\{x_0\}) \cap f^{-1}(\clX) | x_0)$ is zero on the set $\clX \setminus \{x_0\}$
  (if there is any $z \in f^{-1}(\{x_0\}) \cap f^{-1}(\clX \setminus \{x_0\})$, then we have $f(z) = x_0$ and $f(z) \in \clX \setminus \{x_0\}$, which is a contradiction),
  the integral can be reduced on $\{x_0\} \cap \clX$~\citep[Thm.~16.9]{billingsley2012probability}:
  $\int_{\{x_0\} \cap \clX} \nu(\clZ|x) \nu(f^{-1}(\{x_0\}) \cap f^{-1}(\ud x) | x_0) = \bbI[x_0 \in \clX] \nu(\clZ|x_0) \nu(f^{-1}(\{x_0\}) | x_0) = \nu(\clZ|x_0) \bbI[x_0 \in \clX]$.
  Moreover, $\nu(\clZ|x_0) = \nu(\clZ \cap f^{-1}(\{x_0\}) | x_0) + \nu(\clZ \setminus f^{-1}(\{x_0\}) | x_0)$
  where $\nu(\clZ \setminus f^{-1}(\{x_0\}) | x_0) \le \nu(\bbZ \setminus f^{-1}(\{x_0\}) | x_0) = 1 - \nu(f^{-1}(\{x_0\}) | x_0) = 0$,
  we have $\nu(\clZ \setminus f^{-1}(\{x_0\}) | x_0) = 0$ and $\nu(\clZ|x_0) = \nu(\clZ \cap f^{-1}(\{x_0\}) | x_0)$.
  So the integral on the r.h.s is $\nu(\clZ \cap f^{-1}(\{x_0\}) | x_0) \bbI[x_0 \in \clX]$, which is the same as the l.h.s.
  So the equality is verified.
\end{proof}

\setlength{\parskip}{\parskip-1pt}
\belowdisplayskip=4pt
\belowdisplayshortskip=4pt

\section{Topics on the Methods of \ourmodel} \label{supp:meth}

\subsection{Relation to other auto-encoder regularizations} \label{supp:meth-aereg}

There are methods that consider regularizing the standard auto-encoder (AE)~\citep{rumelhart1986learning, baldi1989neural} with deterministic encoder $g(x)$ and decoder $f(z)$ for certain robustness.
These regularizations are introduced in addition to the standard AE loss, \ie the reconstruction loss: $\bbE_{p^*(x)} \ell(x, f(g(x)))$,
where $\ell(x,x')$ is a measure of similarity between $x$ and $x'$.
If $\ell(x, f(z))$ can be treated as a (scaled) negative log-likelihood $-\log p(x|z)$ on $\bbX$ (\eg, squared 2-norm $\ell$ for a Gaussian $p(x|z)$, cross entropy $\ell$ for a Bernoulli/categorical $p(x|z)$),
then we can adopt a distributional view of the decoder as $p(x|z)$ and the encoder as $\delta_{g(x)}(z)$
\footnote{This is the notation of a Dirac's delta function, which is not a function in the usual sense. We adopt this form for the similarity to the DAE loss.},
and reformulate the reconstruction loss also under the distributional view: $\bbE_{p^*(x)} [-\log p(x|g(x))] = \bbE_{p^*(x) \delta_{g(x)}(z)} [-\log p(x|z)]$.

\textbf{Comparison with Jacobian norm regularizations.} \hspace{4pt}
Contractive AE (CAE)~\citep{rifai2011contractive, rifai2012generative} regularizes the Jacobian norm of the encoder, $\lambda \bbE_{p^*(x)} \lrVert{\nabla_x g\trs(x)}_F^2$ ($\lambda$ controls the scale),
in hope to encourage the robustness of the encoded representation against local changes around training data.
When it is combined with the reconstruction loss which preserves data variation in the representation for reconstruction, the robustness is confined to the orthogonal direction to the data manifold, which often does not reflect semantic meanings of interest.
In other words, the variation in this orthogonal direction is contracted in the representation, hence the name.
When applied to a linear encoder, this becomes the well-known weight-decay regularizer.
Note that CAE does not have a generative modeling utility, as it uses a deterministic encoder which leads to insufficient determinacy (Sec.~\ref{sec:determ-dirac}).

Denoising AE (DAE)~\citep{vincent2008extracting, bengio2013generalized, bengio2014deep} considers the robustness to random corruption/perturbation on data,
so its encoding process is $z = g(x + \epsilon_e)$ where $\epsilon_e \sim \clN(0, \sigma_e^2 I_{d_\bbX})$
(or any other distribution with $\bbE[\epsilon_e] = 0$ and $\Var[\epsilon_e] = \sigma_e^2 I_{d_\bbX}$), which defines a probabilistic encoder $q(z|x)$ (note that this is different from an additive Gaussian encoder).
The goal for training a DAE is thus to try to reconstruct the input under the random corruption, by minimizing the DAE loss:
\begin{align}
  \bbE_{p^*(x) q(z|x)} [-\log p(x|z)],
  \label{eqn:dae}
\end{align}
which resembles the distributional form of the standard reconstruction loss.
For infinitesimal corruption variance $\sigma_e^2$ and squared 2-norm $\ell$, the DAE loss \eqref{eqn:dae} is roughly equivalent to regularizing the standard reconstruction loss with
$\sigma_e^2 \bbE_{p^*(x)} \lrVert{\nabla_x (f \circ g)\trs(x)}_F^2$, \ie the Jacobian norm of the reconstruction function~\citep{rifai2011contractive, alain2014regularized}.
So DAE can be viewed to promote the robustness of reconstruction while CAE of the representation~\citep{rifai2011contractive}.

In contrast, for additive Gaussian decoder (\ie, squared 2-norm $\ell$) and encoder, our compatibility regularization \eqref{eqn:compt-obj-jac} is $\bbE_{\rho(x,z)} \lrVert{ \frac{1}{\sigma_d^2} \big( \nabla_z f\trs(z) \big)\trs - \frac{1}{\sigma_e^2} \nabla_x g\trs(x) }_F^2$,
which is different from CAE and DAE regularizations.
Ideologically, the compatibility loss is an intrinsic constraint to make use of the distributional nature of the encoder and decoder, and is not motivated from the additional requirement of robustness in some sense.

\textbf{Comparison with a more accurate DAE reformulation.} \hspace{4pt}
In fact, the analysis in~\citep{rifai2011contractive, alain2014regularized} for DAE as a regularization of the reconstruction loss is inaccurate.
Key ingredients for the analysis are the Taylor expansions:
$\lrVert{x + \varepsilon}_2^2 = \lrVert{x}_2^2 + 2 x\trs \varepsilon + \varepsilon\trs \varepsilon$,
$\exp\{x + \varepsilon\} = \exp\{x\} (1 + \varepsilon + \frac{1}{2} \varepsilon^2) + o(\varepsilon^2)$, and
$\log(1 + \varepsilon) = \varepsilon - \frac{1}{2} \varepsilon^2 + o(\varepsilon^2)$.
In the following, we consider $\ell(x,x') = \lrVert{x-x'}_2^2$, corresponding to an additive Gaussian decoder $p(x|z)$.

First consider the additive Gaussian encoder, $q(z|x) = \clN(z | g(x), \sigma_e^2 I_{d_\bbZ})$, or $z = g(x) + \epsilon_e, \epsilon_e \sim \clN(0, \sigma_e^2 I_{d_\bbZ})$.
Consider the case for infinitesimal $\sigma_e$.
For the DAE loss \eqref{eqn:dae}, we have (omitting the expectation over $p^*(x)$):
\begin{align}
  & \bbE_{q(z|x)} [-\log p(x|z)] = \bbE_{q(z|x)} [\ell(x, f(z))] = \bbE_{q(z|x)} \lrVert{x - f(z)}_2^2 = \bbE_{p(\epsilon_e)} \lrVert{x - f(g(x) + \epsilon_e)}_2^2 \\
  ={} & \bbE_{p(\epsilon_e)} \lrbrack{ \lrVert{x - f(g(x)) - (\nabla f\trs)\trs \epsilon_e - \frac{1}{2} \epsilon_e\trs (\nabla^2 f) \epsilon_e + o(\epsilon_e^2)}_2^2 } \\
  ={} & \bbE_{p(\epsilon_e)} \Big[ \lrVert{x - f(g(x))}_2^2 - 2 \big( x - f(g(x)) \big)\trs \Big( (\nabla f\trs)\trs \epsilon_e + \frac{1}{2} \epsilon_e\trs (\nabla^2 f) \epsilon_e \Big) \\*
  & {}+ \epsilon_e\trs (\nabla f\trs) (\nabla f\trs)\trs \epsilon_e + o(\epsilon_e^2) \Big] \\
  ={} & \ell \big( x, f(g(x)) \big) - \sigma_e^2 \big( x - f(g(x)) \big)\trs \Delta f + \sigma_e^2 \lrVert{\nabla f\trs}_F^2 + o(\sigma_e^2),
  \label{eqn:dae-addgauss}
\end{align}
where $(\nabla f\trs)\trs$ is the Jacobian of $f$, $(\epsilon_e\trs (\nabla^2 f) \epsilon_e)_i := \sum_{j,k = 1..d_\bbZ} (\epsilon_e)_i (\epsilon_e)_j \partial_{z_i} \partial_{z_j} f_i$,
and $(\Delta f)_i := \sum_{j = 1..d_\bbZ} \partial_{z_j} \partial_{z_j} f_i$,
and they are evaluated at $z = g(x)$.
Note that in addition to the Jacobian norm regularization term discovered in~\citep{rifai2011contractive, alain2014regularized},
there is a second regularization term $- \sigma_e^2 \big( x - f(g(x)) \big)\trs \Delta f$ that DAE imposes.

For the data-fitting loss of \ourmodel \eqref{eqn:mle}, a similar approximation can be derived (again, omitting the expectation over $p^*(x)$):
\begin{align}
  & \log \bbE_{q(z|x)} [1 / p(x|z)] = \log \bbE_{q(z|x)} \exp\{-\log p(x|z)\} = \log \bbE_{q(z|x)} \exp\{\ell(x, f(z))\} \\
  ={} & \log \bbE_{q(z|x)} \exp\{ \lrVert{x - f(z)}_2^2 \} = \log \bbE_{p(\epsilon_e)} \exp\{ \lrVert{x - f(g(x) + \epsilon_e)}_2^2 \} \\
  ={} & \log \bbE_{p(\epsilon_e)} \exp\lrbrace{ \lrVert{x - f(g(x)) - (\nabla f\trs)\trs \epsilon_e - \frac{1}{2} \epsilon_e\trs (\nabla^2 f) \epsilon_e + o(\epsilon_e^2)}_2^2 } \\
  ={} & \log \bbE_{p(\epsilon_e)} \exp\Big\{ \lrVert{x - f(g(x))}_2^2
    - 2 \big( x - f(g(x)) \big)\trs \Big( (\nabla f\trs)\trs \epsilon_e + \frac{1}{2} \epsilon_e\trs (\nabla^2 f) \epsilon_e \Big) \\*
  & {}+ \epsilon_e\trs (\nabla f\trs) (\nabla f\trs)\trs \epsilon_e + o(\epsilon_e^2) \Big\} \\
  ={} & \log \bbE_{p(\epsilon_e)} \Big[ \exp\{\lrVert{x - f(g(x))}_2^2\} \Big( 1
      - 2 \big( x - f(g(x)) \big)\trs \Big( (\nabla f\trs)\trs \epsilon_e + \frac{1}{2} \epsilon_e\trs (\nabla^2 f) \epsilon_e \Big) \\*
      & {}+ \epsilon_e\trs (\nabla f\trs) (\nabla f\trs)\trs \epsilon_e
      + 2 \Big( \big( x - f(g(x)) \big)\trs (\nabla f\trs)\trs \epsilon_e \Big)^2 + o(\epsilon_e^2)
  \Big) \Big] \\
  ={} & \log \Big[ \exp\{\lrVert{x - f(g(x))}_2^2\} \Big( 1
      - \sigma_e^2 \big( x - f(g(x)) \big)\trs \Delta f \\*
      & {}+ \sigma_e^2 \lrVert{\nabla f\trs}_F^2
      + 2 \sigma_e^2 \lrVert{(\nabla f\trs) \big( x - f(g(x)) \big)}_2^2 + o(\sigma_e^2)
  \Big) \Big] \\
  ={} & \ell\big( x - f(g(x)) \big) - \sigma_e^2 \big( x - f(g(x)) \big)\trs \Delta f + \sigma_e^2 \lrVert{\nabla f\trs}_F^2
  + 2 \sigma_e^2 \lrVert{(\nabla f\trs) \big( x - f(g(x)) \big)}_2^2 + o(\sigma_e^2).
\end{align}
This is different from the regularization interpretation of DAE \eqref{eqn:dae-addgauss} as a third regularization term $2 \sigma_e^2 \lrVert{(\nabla f\trs) \big( x - f(g(x)) \big)}_2^2$ is presented.

The compatibility loss \eqref{eqn:compt-obj-jac} in \ourmodel becomes: $\bbE_{\rho(x,z)} \lrVert{ \frac{1}{\sigma_d^2} \big( \nabla_z f\trs(z) \big)\trs - \frac{1}{\sigma_e^2} \nabla_x g\trs(x) }_F^2$,
where $\sigma_d^2$ is the Gaussian variance of the decoder $p(x|z)$ (inverse scale for $\ell(\cdot,\cdot)$).
When $\rho(x,z) = p^*(x) q(z|x)$, this can be further reduced to (omitting the expectation over $p^*(x)$):
\begin{align}
  & \bbE_{q(z|x)} \lrVert{ \frac{1}{\sigma_d^2} \big( \nabla_z f\trs(z) \big)\trs - \frac{1}{\sigma_e^2} \nabla_x g\trs(x) }_F^2
  = \bbE_{p(\epsilon_e)} \lrVert{ \frac{1}{\sigma_d^2} \big( \nabla f\trs( g(x) + \epsilon_e ) \big)\trs - \frac{1}{\sigma_e^2} \nabla g\trs }_F^2 \\
  ={} & \bbE_{p(\epsilon_e)} \lrVert{ \frac{1}{\sigma_d^2} \Big( (\nabla f\trs)\trs + (\nabla^2 f\trs)\trs \epsilon_e + \frac{1}{2} \epsilon_e\trs (\nabla^3 f\trs)\trs \epsilon_e + o(\epsilon_e^2) \Big) - \frac{1}{\sigma_e^2} \nabla g\trs }_F^2 \\
  ={} & \lrVert{\frac{1}{\sigma_d^2} (\nabla f\trs)\trs - \frac{1}{\sigma_e^2} \nabla g\trs}_F^2 + \frac{\sigma_e^2}{\sigma_d^4} (\nabla f\trs) : (\nabla \Delta f\trs) + \frac{\sigma_e^2}{\sigma_d^4} \lrVert{\nabla^2 f\trs}_F^2 + o(\sigma_e^2),
\end{align}
where $\big( (\nabla^2 f\trs)\trs \epsilon_e \big)_{ij} := \sum_{k = 1..d_\bbZ} (\epsilon_e)_k \partial_{z_k} \partial_{z_j} f_i$,
$\big( \epsilon_e\trs (\nabla^3 f\trs)\trs \epsilon_e \big)_{ij} := \sum_{k,k' = 1..d_\bbZ} (\epsilon_e)_k (\epsilon_e)_{k'} \partial_{z_k} \partial_{z_{k'}} \partial_{z_j} f_i$,
and $(\nabla f\trs) : (\nabla \Delta f\trs) := \sum_{\substack{i = 1..d_\bbX, \\ j,k = 1..d_\bbZ}} (\partial_{z_j} f_i) (\partial_{z_j} \partial_{z_k} \partial_{z_k} f_i)$,
$\lrVert{\nabla^2 f\trs}_F^2 := \sum_{\substack{i = 1..d_\bbX, \\ j,k = 1..d_\bbZ}} \big( \partial_{z_k} \partial_{z_j} f_i \big)^2$,
and all terms are evaluated at $z = g(x)$.
This is different from the regularization of CAE and the regularization explanation of DAE.

For the corruption encoder, $z = g(x + \epsilon_e), \epsilon_e \sim \clN(0, \sigma_e^2 I_{d_\bbX})$, approximations of the DAE loss \eqref{eqn:dae} and the data-fitting loss of \ourmodel \eqref{eqn:mle} (\ie, negative data likelihood loss) are similar to the above expansions except that derivatives of $f$ are replaced with those of $f \circ g$.
Particularly, from \eqref{eqn:dae-addgauss}, we find that the conclusion in~\citep{rifai2011contractive, alain2014regularized} missed the term $- \sigma_e^2 \big( x - f(g(x)) \big)\trs \Delta (f \circ g)$ that is also of order $\sigma_e^2$.
For the compatibility loss, as there is no explicit expression of $\log q(z|x)$ (unless $g(x)$ is invertible), the above expression does not hold.
But anyway, it is different from CAE and DAE regularizations.

\textbf{Relation to the tied weights trick.} \hspace{4pt}
The compatibility loss also explains the ``tied weights'' trick in AE, which is widely adopted and is vital for the success of AE~\citep{ranzato2007sparse, vincent2008extracting, vincent2011connection, rifai2011contractive, alain2014regularized}.
The trick is considered when components of $x$ and $z$ are binary, and a one-layer, product-of-Bernoulli encoder $q(z|x) = \prod_{j=1}^{d_\bbZ} \Bern(z_j | s((W_e)_{j,:} x + (b_e)_j))$ and decoder $p(x|z) = \prod_{i=1}^{d_\bbX} \Bern(x_i | s((W_d)_{i,:} z + (b_d)_i))$ are used,
where $s(l) := 1 / (1 + \exp\{-l\})$ denotes the sigmoid activation function.
For the encoder, we have $q(z|x) = \frac{\exp\{z\trs (W_e x + b_e)\}}{\prod_{j=1}^{d_\bbZ} ( 1 + \exp\{(W_e)_{j,:} x + (b_e)_j\} )}$ thus $\log q(z|x) = z\trs (W_e x + b_e) - \sum_{j=1}^{d_\bbZ} \log (1 + \exp\{(W_e)_{j,:} x + (b_e)_j\})$,
so $\nabla_x \nabla_z\trs \log q(z|x) = W_e\trs$.
Similarly for the decoder, we have $\nabla_x \nabla_z\trs \log p(x|z) = W_d$.
So the compatibility loss \eqref{eqn:compt-obj-jac} in this case is $\lrVert{W_d - W_e\trs}_F^2$, which leads to $W_d = W_e\trs$ when it is zero.
This recovers the tied weight trick.

In this Bernoulli case, the CAE regularizer is $\bbE_{p^*(x)} \sum_{j=1}^{d_\bbZ} \frac{\exp\{-2((W_e)_{j,:} x + (b_e)_j)\} \sum_{i=1}^{d_\bbX} (W_e)_{ji}^2}{\big( 1 + \exp\{-((W_e)_{j,:} x + (b_e)_j)\} \big)^4}
= \bbE_{p^*(x)} \sum_{j=1}^{d_\bbZ} s((W_e)_{j,:} x + (b_e)_j)^2 \big( 1 - s((W_e)_{j,:} x + (b_e)_j) \big)^2 \sum_{i=1}^{d_\bbX} (W_e)_{ji}^2$
and DAE does not have the Jacobian-norm regularization explanation, so they are different from the compatibility loss.

\subsection{Gradient estimation for flow-based models without tractable inverse} \label{supp:meth-flownoinv}

\textbf{Flow-based density models.} \hspace{4pt}
As the insight we draw from the analysis on Gaussian VAE in Sec.~\ref{sec:meth-compt}, it is inappropriate to implement both conditionals $p_\theta(x|z)$, $q_\phi(z|x)$ using additive Gaussian models.
So we need more flexible and expressive probabilistic models that also allow explicit density evaluation (so implicit models like GANs are not suitable).
Flow-based models~\citep{dinh2015nice, papamakarios2017masked, kingma2018glow, behrmann2019invertible, chen2019residual} are a good choice.
They also allow direct sampling with reparameterization for efficiently estimating and optimizing the data-fitting loss \eqref{eqn:mle} (for which energy-based models are costly),
and have been used as the inference model $q_\phi(z|x)$ of VAEs~\citep{rezende2015variational, kingma2016improved, van2018sylvester, grathwohl2019ffjord}.
For a connection to these prior works, we use a flow-based model also for the inference model $q_\phi(z|x)$.
An additive-Gaussian likelihood model $p_\theta(x|z)$ is then allowed for learning nonlinear representations.

To define the distribution $q_\phi(z|x)$, a flow-based model uses a parameterized \emph{invertible} differentiable transformation $z = T_\phi(e|x)$ to map a random seed $e$ (of the same dimension $d_\bbZ$) following a simple base distribution $p(e)$ \footnote{
  Although some flow-based models (\eg, the Sylvester flow~\citep{van2018sylvester}) also incorporate the dependency on $x$ in the base distribution $\ppt_\phi(\eet|x)$ (\eg, $\clN(\eet | \mu_\phi(x), \Sigma_\phi(x))$),
  we can reparameterize this distribution as transformed from a ``more basic'' parameter-free base distribution $p(e)$ (\eg, $\clN(0,I_{d_\bbZ})$) by an $x$-dependent transformation (\eg, $\eet = \mu_\phi(x) + \Sigma_\phi(x)^{1/2} e$) and concatenate this transformation to the original one as $z = T_\phi(e|x)$.
}(\eg, a standard Gaussian) to $\bbZ = \bbR^{d_\bbZ}$.
By deliberately designed architectures, the transformation $T_\phi(\cdot|x)$ is guaranteed to be invertible, yet still being expressive, with some examples that are even universal approximators~\citep{teshima2020coupling}.
Benefited from the invertibility, the defined density can be explicitly given by the rule of change of variables~\citep[Thm.~17.2]{billingsley2012probability}:
\begin{align} \label{eqn:flow-density}
  q_\phi(z|x) = p(e = T^{-1}_\phi(z|x)) \lrvert{\nabla_z T^{-\top}_\phi(z|x)},
\end{align}
where $\lrvert{\nabla_z T^{-\top}_\phi(z|x)}$ is the absolute value of the determinant of the Jacobian of $T^{-1}_\phi(z|x)$ (w.r.t $z$).

\textbf{Problem for evaluating the compatibility loss.} \hspace{4pt}
Although $T_\phi(z|x)$ is guaranteed to be invertible, in common instances computing its inverse is intractable~\citep{rezende2015variational, kingma2016improved, van2018sylvester} or costly~\citep{grathwohl2019ffjord, behrmann2019invertible, chen2019residual}
(however, they all guarantee an easy calculation of the Jacobian determinant $\lrvert{\nabla_z T^{-\top}_\phi(z|x)}$ for efficient density evaluation).
This means that density estimation of $q_\phi(z|x)$ is intractable for an arbitrary $z$ value, but is only possible for a generated $z$ value, whose inverse $e$ is known in advance (the generated $z$ is computed from this $e$).
This however, introduces problems when computing the gradients $\nabla_x \log q_\phi(z|x)$, $\nabla_z \log q_\phi(z|x)$ for the compatibility loss (\eqref{eqn:compt-obj-jac} or \eqref{eqn:compt-obj-jac-hut}).

To see this, it is important to distinguish the ``formal arguments'' and ``actual arguments'' of a function.
It makes a difference when taking derivatives if the actual arguments are fed to formal arguments in an involved way.
What we need is the derivatives w.r.t the formal arguments, but automatic differentiation tools (\eg, the \texttt{autograd} utility in PyTorch~\citep{paszke2019pytorch}) could only compute the derivatives w.r.t the actual arguments.
We use capital subscripts for formal arguments and lowercase letters for actual arguments.
Following this rule, we denote $\log q^\phi_{Z|X}(z|x)$ for $\log q_\phi(z|x)$ above, so $\nabla_Z \log q^\phi_{Z|X}$ denotes the gradient function that differentiates the first formal argument $Z$ of $\log q^\phi_{Z|X}$, and similarly for $\nabla_X \log q^\phi_{Z|X}$. 
Then at a generated value of $z = T_\phi(e|x)$ from a random seed $e$, the required gradients in the compatibility loss are w.r.t to the formal arguments $\nabla_Z \log q^\phi_{Z|X}(T_\phi(e|x) | x)$ and $\nabla_X \log q^\phi_{Z|X}(T_\phi(e|x) | x)$,
while automatic differentiation tools could only give the gradients w.r.t the actual arguments $\nabla_e \log q^\phi_{Z|X}(T_\phi(e|x) | x)$ and $\nabla_x \log q^\phi_{Z|X}(T_\phi(e|x) | x)$, which are not the desired gradients.
Note that we do not know the exact calculation rule of $\log q^\phi_{Z|X}(z|x)$ for arbitrary $z$ and $x$, but can only evaluate $h^\phi(e,x) := \log q^\phi_{Z|X}(T_\phi(e|x) | x)$ from a given $e$ and $x$.
Automatic differentiation could only evaluate the gradients of this $h^\phi(e,x)$ but not of $\log q^\phi_{Z|X}(z|x)$.

\textbf{Solution.} \hspace{4pt}
An explicit deduction is thus required for an expression of the correct gradients in terms of what automatic differentiation could evaluate.
From the chain rule, we have:
\begin{align}
  \nabla_e h^\phi(e,x) = \nabla_e \log q^\phi_{Z|X}(T_\phi(e|x) | x)
  ={} & \big( \nabla_e T_\phi\trs(e|x) \big) \big( \nabla_Z \log q^\phi_{Z|X}(T_\phi(e|x) | x) \big),
  \label{eqn:grad-deduc-ht-e} \\
  \nabla_x h^\phi(e,x) = \nabla_x \log q^\phi_{Z|X}(T_\phi(e|x) | x)
  ={} & \big( \nabla_x T_\phi\trs(e|x) \big) \big( \nabla_Z \log q^\phi_{Z|X}(T_\phi(e|x) | x) \big) \\*
  & {}+ \nabla_X \log q^\phi_{Z|X}(T_\phi(e|x) | x).
  \label{eqn:grad-deduc-ht-x}
\end{align}
The first equation gives one of the desired gradients:
$\nabla_Z \log q^\phi_{Z|X}(T_\phi(e|x) | x) = \big( \nabla_e T_\phi\trs(e|x) \big)^{-1} \big( \nabla_e h^\phi(e,x) \big)$.
The term $\nabla_e h^\phi(e,x)$ can be evaluated using automatic differentiation, as mentioned.
The other term, \ie the Jacobian $\nabla_e T_\phi\trs(e|x)$, can also use automatic differentiation by tracking the forward flow computation $z = T_\phi(e|x)$,
but it is often available in closed-form for flow-based models, as flow-based models need to evaluate its determinant anyway so the architecture is designed to give its closed-form expression.

The second equation gives an expression of the other desired gradient:
$\nabla_X \log q^\phi_{Z|X}(T_\phi(e|x) | x) = \nabla_x h^\phi(e,x) - \big( \nabla_x T_\phi\trs(e|x) \big) \big( \nabla_Z \log q^\phi_{Z|X}(T_\phi(e|x) | x) \big)$.
Again, the first term $\nabla_x h^\phi(e,x)$ can be evaluated using automatic differentiation.
The term $\big( \nabla_Z \log q^\phi_{Z|X}(T_\phi(e|x) | x) \big)$ can be evaluated using the expression we just derived above.
For the rest term, \ie the Jacobian $\nabla_x T_\phi\trs(e|x)$, it can also be evaluated using automatic differentiation by tracking the forward flow computation $z = T_\phi(e|x)$.
For computation efficiency, this can be implemented by taking the gradient of $z = T_\phi(e|x)$ w.r.t $x$ with the \texttt{grad\_outputs} argument of \texttt{torch.autograd.grad} fed by $\nabla_Z \log q^\phi_{Z|X}(T_\phi(e|x) | x)$
(gradients w.r.t $x$ will not be back-propagated through this $\nabla_Z \log q^\phi_{Z|X}(T_\phi(e|x) | x)$).
This reduces computation complexity from $O(d_\bbX d_\bbZ)$ down to $O(d_\bbX + d_\bbZ)$.
In summary, the desired gradients can be computed via the following expressions:
\begin{align}
  \nabla_Z \log q^\phi_{Z|X}(T_\phi(e|x) | x)
  ={} & \big( \nabla_e T_\phi\trs(e|x) \big)^{-1} \big( \nabla_e h^\phi(e,x) \big),
  \label{eqn:grad-deduc-h-z} \\
  \nabla_X \log q^\phi_{Z|X}(T_\phi(e|x) | x)
  ={} & \nabla_x h^\phi(e,x) - \big( \nabla_x T_\phi\trs(e|x) \big) \big( \nabla_Z \log q^\phi_{Z|X}(T_\phi(e|x) | x) \big).
  \label{eqn:grad-deduc-h-x}
\end{align}
Second-order differentiations for the compatibility loss can be done in a similar way.

\textbf{A simplified compatibility loss.} \hspace{4pt}
For the compatibility loss \eqref{eqn:compt-obj-jac-hut} in the form of Hutchinson's trace estimator, a further simplification is possible.
The loss is given by:
\begin{align}
  C(\theta,\phi) = \bbE_{\rho(x,z)} \bbE_{p(\eta_x)} \lrVert{
    \nabla_Z \big( \eta_x\trs \nabla_X \log p^\theta_{X|Z}(x|z) - \eta_x\trs \nabla_X \log q^\phi_{Z|X}(z|x) \big)
  }_2^2,
\end{align}
with any random vector $\eta_x$ satisfying $\bbE[\eta_x] = 0, \Var[\eta_x] = I_{d_\bbX}$.
The reference distribution $\rho(x,z)$ can be taken as $p^*(x) q_\phi(z|x)$ for practical sampling for estimating the expectation.
For a flow-based $q_\phi(z|x)$, sampling from $(x,z) \sim p^*(x) q_\phi(z|x)$ is equivalent to $(x, T_\phi(e|x)), e \sim p(e)$.
So the loss can be reformulated as:
\begin{align}
  C(\theta,\phi) = \bbE_{p^*(x) p(e)} \bbE_{p(\eta_x)} \lrVert{
    \nabla_Z \big( \eta_x\trs \nabla_X \log p^\theta_{X|Z}(x | T_\phi(e|x)) - \eta_x\trs \nabla_X \log q^\phi_{Z|X}(T_\phi(e|x) | x) \big)
  }_2^2.
\end{align}
Note from \eqref{eqn:grad-deduc-h-z}, the gradient w.r.t $Z$ is an invertible linear transformation of the gradient w.r.t $e$,
so its norm equals zero if and only if the gradient w.r.t $e$ has a zero norm.
So to avoid this matrix inversion, we consider a simpler loss that achieves the same optimal solution:
\begin{align} \label{eqn:compt-obj-jac-hut-simp}
  \Ct(\theta,\phi) := \bbE_{p^*(x) p(e)} \bbE_{p(\eta_x)} \lrVert{
    \nabla_e \big( \eta_x\trs \nabla_X \log p^\theta_{X|Z}(x | T_\phi(e|x)) - \eta_x\trs \nabla_X \log q^\phi_{Z|X}(T_\phi(e|x) | x) \big)
  }_2^2.
\end{align}
For additive Gaussian $p^\theta_{X|Z}$, its gradient $\nabla_X \log p^\theta_{X|Z}$ is available in closed-form.
For $\nabla_X \log q^\phi_{Z|X}(T_\phi(e|x) | x)$ in the second term, it can be estimated using \eqref{eqn:grad-deduc-h-x} we just developed.
The subsequent gradient w.r.t $e$ can be evaluated by automatic differentiation.
So this loss is tractable to estimate and optimize.

\section{Experiment Details} \label{supp:expm}

\subsection{Baseline Methods} \label{supp:expm-baselines}

We compare our proposed \ourmodel with bi-directional models (composed of both a likelihood model and an inference model) including Denoising Auto-Encoder (DAE), Variational Auto-Encoder (VAE) and BiGAN.
Sketches of these models are introduced below.

\textbf{DAE}~\citep{vincent2008extracting} \hspace{4pt}
first corrupts a real input data $x$ with a local noise and then pass it through an encoder to define $q_\phi(z|x)$.
The latent code $z$ is then decoded to the data space by a decoder $p_\theta(x|z)$.
The objective is to minimize the reconstruction error (\eqref{eqn:dae}): $\bbE_{p^*(x)} \bbE_{q_\phi(z|x)} [-\log p_\theta(x|z)]$.
Compared with VAE, it can be seen as the version with $\beta = 0$, \ie it does not involve a prescribed prior $p(z)$.
Nevertheless, optimizing the objective w.r.t $\phi$ may lead to undesired results.
Particularly, for any given $x$, it may drive $q_\phi(z|x)$ to only concentrate on $z$ values that maximizes $\log p_\theta(x|z)$.
This renders incompatibility and an insufficient determinacy (see Sec.~\ref{sec:meth-data}).

\textbf{VAE}~\citep{kingma2014auto} \hspace{4pt}
defines a joint distribution $p_\theta(x,z) = p(z) p_\theta(x|z)$ using a specified prior $p(z)$.
It learns $p_\theta(x|z)$ to match data distribution $p^*(x)$ with the help of an inference model $q_\phi(z|x)$, using the Evidence Lower BOund (ELBO) objective:
\begin{align} \label{eqn:elbo}
  \min_{\theta, \phi} \bbE_{p^*(x)} \bbE_{q_\phi(z|x)} \left[ - \log p_\theta(x|z) \right]
  + \beta \bbE_{p^*(x)} [\KL(q_\phi(z|x) \Vert p(z))].
\end{align}
When $\beta = 1$, the negative objective is a lower bound of the data likelihood (evidence) $\bbE_{p^*(x)} [\log p_\theta(x)]$ where $p_\theta(x) := \int_\bbZ p(z) p_\theta(x|z) \ud z$, hence the name.
Optimizing it w.r.t $\phi$ also drives $q_\phi(z|x)$ to the true posterior $p_\theta(z|x)$ and also makes the bound tighter.
A $\beta$ other than 1 is considered when there is some desideratum on the latent variable, \eg, disentanglement~\citep{higgins2017beta}.

\textbf{BiGAN}~\citep{donahue2017adversarial, dumoulin2017adversarially}. \hspace{4pt}
In addition to learning the data distribution $p^*(x)$ using GAN~\citep{goodfellow2014generative}, BiGAN also aims to learn a representation extractor, so it introduces an inference model $q_\phi(z|x)$ which is often deterministic (\ie, a Dirac distribution).
The likelihood model (generator) $p_\theta(x|z)$ defines a joint $p(z) p_\theta(x|z)$ with the help of a prescribed prior $p(z)$, while the inference model also defines a joint $p^*(x) q_\phi(z|x)$.
Samples from both distributions can be easily drawn, so BiGAN seeks to match them using the GAN loss (Jensen-Shannon divergence) with the help of a discriminator $D(x,z)$.
In each training step, the discriminator $D(x,z)$ is first updated on a mini-batch of $p^*(x) q_\phi(z|x)$ data $x^+ \sim p^*(x),  z^+ \sim q_\phi(\cdot|x^+)$ with positive labels $y^+=1$ and a mini-batch of $p(z) p_\theta(x|z)$ data $z^- \sim p(z), x^- \sim p_\theta(\cdot|z^-)$ with negative labels $y^-=0$.
The goal of the training the discriminator is to minimize the binary cross entropy loss $\text{BCE}(D(x^+, z^+), y^+) + \text{BCE}(D(x^-, z^-), y^-)$.
The conditional models $q_\phi(z|x)$ and $p_\theta(x|z)$ are then updated to maximize the same loss $\text{BCE}(D(x^+, z^+), y^+) + \text{BCE}(D(x^-, z^-), y^-)$.

\textbf{GibbsNet}~\citep{lamb2017gibbsnet} \hspace{4pt}
is also considered, which is similar to BiGAN, except that BiGAN's prior-driven joint sample generation $z^- \sim p(z), x^- \sim p(x|z^-)$ is modified to run through multiple cycles:
$z_0 \sim p(z), x_0 \sim p(x|z_0), z_1 \sim q(z|x_0), x_1 \sim p(x|z_1), \cdots, z^- \sim q(z|x_{K-1}), x^- \sim p(x|z^-)$.
This resembles a Gibbs chain, and is considered in GibbsNet for reducing the influence of a specified prior, as the stationary distribution of the Markov chain does not rely on the initial distribution but is determined by the two conditional models (see Sec.~\ref{sec:relw} (paragraph~2) for its limitation).
But this iterated application of the likelihood and inference models makes gradient back-propagation involved.
The gradient accumulates with the cycling iteration, resulting in a scale much larger than usual.
This makes gradient-based optimization unstable and even leads to numerical instability.
We did not find a reasonable result in any experiment using the same architecture so we omit the comparison.

\subsection{Model Architecture} \label{supp:expm-arch}

\begin{figure}[h]
  \centering
  \includegraphics[width=.85\textwidth]{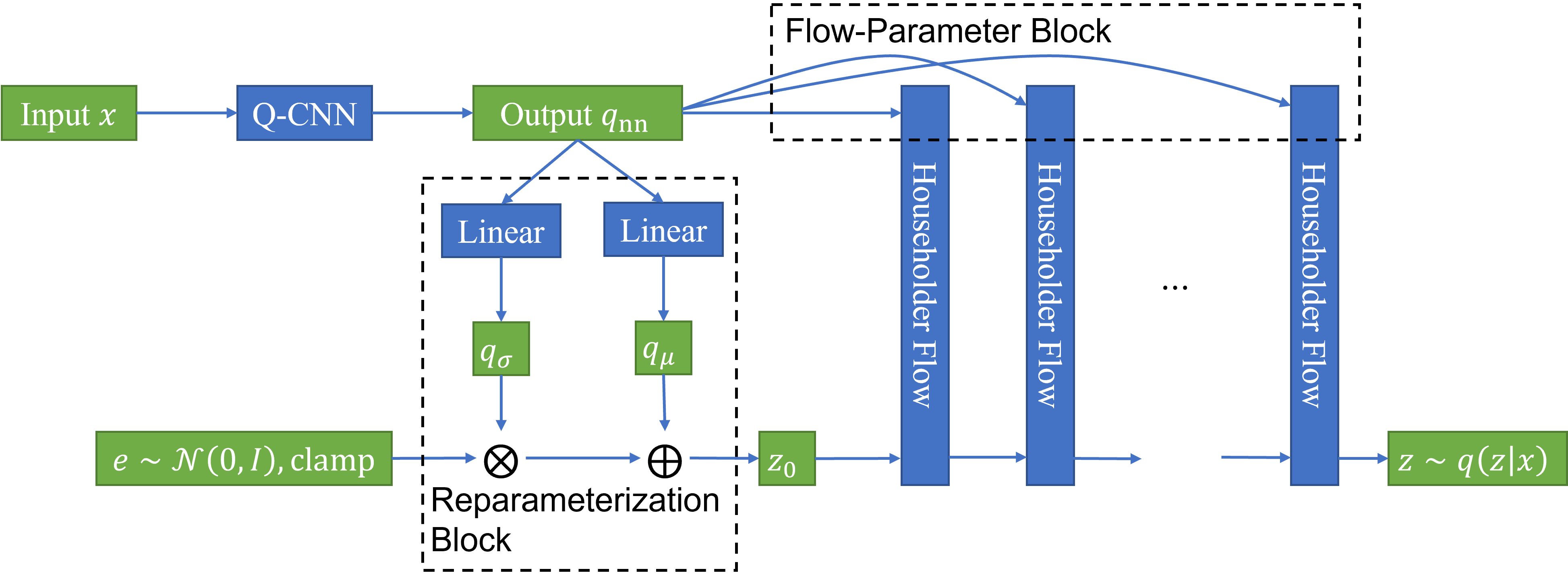}
  \caption{Flow architecture of the inference model $q_\phi(z|x)$.
    See Table~\ref{tab:arch} for detailed specification.
  }
  \vspace{-6pt}
  \label{fig:flow-arch}
\end{figure}

Our code is developed based on the repositories of
the Sylvester flow\footnote{\url{https://github.com/riannevdberg/sylvester-flows}}~\citep{van2018sylvester} and
FFJORD\footnote{\url{https://github.com/rtqichen/ffjord}}~\citep{grathwohl2019ffjord} for the task environment and flow architectures.
VAE, DAE and \ourmodel share the same architecture of $p_\theta(x|z)$ and of $q_\phi(z|x)$, which are detailed in Table~\ref{tab:arch}.
The inference model $q_\phi(z|x)$ adopts the architecture of Sylvester flow~\citep{van2018sylvester}, illustrated in Fig.~\ref{fig:flow-arch}.
It consists of a neural network (denoted as C-QNN) that outputs $q_\text{nn}$, a reparameterization module, and a set of consecutive $N$ flows.
The outputs $q_\mu$ and $q_\sigma$ are used to parameterize the diagonal Gaussian distribution for initializing $z_0$, the input to the flows.
For implementation simplicity, we choose the Householder version of the Sylvester flow.
For each flow layer, the output $\mathbf{z}_t$ of the flow given input $\mathbf{z}_{t-1}$ is:
\begin{align}
  \mathbf{z}_t = \mathbf{z}_{t-1} + \mathbf{A}_t h(\mathbf{B}_t \mathbf{z}_{t-1} + \mathbf{b}_t),
\end{align}
where $\mathbf{A}_t = \mathbf{Q}_t \mathbf{R}_t$, $\mathbf{B}_t = \tilde{\mathbf{R}}_t \mathbf{Q}_t$ and $\mathbf{b}_t$ are parameters of the $t$-th flow and $h$ is the hyperbolic-tangent activation function.
Let $\mathbf{A} = \mathbf{Q} \mathbf{R}$, $\mathbf{B} = \tilde{\mathbf{R}} \mathbf{Q}\trs$, where $\mathbf{R}$ and $\tilde{\mathbf{R}}$ are upper triangular matrices, and $\mathbf{Q} = \prod_{i=1}^H (\mathbf{I} - \frac{2\mathbf{v}_i \mathbf{v}_i\trs}{\mathbf{v}_i\trs \mathbf{v}_i})$ is a sequence of $H = 8$ Householder transformations.
All the flow parameters, \ie $\mathbf{v}_{1:H}$, $\mathbf{b}$, $\mathbf{R}$ and $\tilde{\mathbf{R}}$, depend on $q_\text{nn}$ via a flow-parameter block.

\subsection{Synthetic Experiments} \label{supp:expm-synth}

The synthetic datasets (``pinwheel'' in the main text and ``8gaussians'' in this appendix) are adopted from the above mentioned repositories of the Sylvester flow and FFJORD.
The dimension of the data space is $d_\bbX = 2$, and we take the latent space to be of the same dimension $d_\bbZ = 2$.

For the inference model $q_\phi(z|x)$, we use a three-layer MLP for the C-QNN component, with 8 hidden nodes in each layer.
After the reparameterization block, consecutive $N = 32$ Householder flow layers are concatenated.
Each flow layer has $H = 2$ Householder transformations\footnote{To make an invertible transformation, the number of Householder transformations $H$ needs to be no larger than the dimension of $z$, which is 2 in this synthetic experiment}.
For the likelihood model $p_\theta(x|z)$, it is implemented as an additive Gaussian model.
Its mean function is a three-layer MLP with 16 hidden nodes in each layer, and its variance is taken isotropic with fixed scale $0.01$.

For training, we use the Adam optimizer~\citep{kingma2014adam} with batch-size 1000 and weight decay parameter $1\e{-5}$ for all methods.
All methods use a learning rate of $1\e{-3}$ except for DAE which uses $1\e{-4}$.
For BiGAN, the generator is updated once per 128 updates of the discriminator using step size $1\e{-4}$.
For \ourmodel, conditional models $p_\theta(x|z)$ and $q_\phi(z|x)$ are trained by minimizing: $1\e{-5} \times {}$ compatibility loss \eqref{eqn:compt-obj-jac-hut-simp} ${}+{}$ data-fitting loss \eqref{eqn:mle},
where the expectation in the data-fitting loss is estimated using 16 samples from $q_\phi(z|x)$ with reparameterization~\citep{kingma2014auto}.
For the version \ourmodelpt with PreTraining, the conditional models are first pretrained as in a VAE by minimizing the ELBO objective \eqref{eqn:elbo} (with $\beta = 1$) using the standard Gaussian prior for 1000 epochs,
and are then trained as in \ourmodel by minimizing the above objective with a $10$-times smaller learning rate for the likelihood model $p_\theta(x|z)$ (same learning rate for the inference model $q_\phi(z|x)$).
DAE is also pretrained in this way.

For data generation, VAE and BiGAN use ancestral sampling: first draw a sample of $z$ from the standard Gaussian prior $p(z)$, and then draw a data sample $x$ from the likelihood model $p_\theta(x|z)$.
For DAE and \ourmodel, they do not have a prior model for ancestral sampling.
They use MCMC methods, including Gibbs sampling and SGLD (see Sec.~\ref{sec:meth-gen}).
One difference for the synthetic experiment is that the SGLD version is done by passing through the likelihood model $p_\theta(x|z)$ with prior samples drawn via SGLD in the latent space $\bbZ$ similar to \eqref{eqn:langevin-px}:
{\abovedisplayskip=2pt
\begin{align} \label{eqn:langevin-pz}
  \textstyle \!\!
  z^{(t+1)} = z^{(t)} + \varepsilon \nabla_{z^{(t)}} \log \frac{q_\phi(z^{(t)}|x^{(t)})}{p_\theta(x^{(t)}|z^{(t)})} + \sqrt{2 \varepsilon} \, \eta^{(t)}_z, \text{where }
  x^{(t)} \! \sim p_\theta(x|z^{(t)}),
  \eta^{(t)}_z \! \sim \clN(0,I_{d_\bbZ}),
\end{align} }%
and $\varepsilon$ is a step size parameter, taken as $3\e{-4}$.
Both Gibbs sampling and SGLD are run for 100 iterations (transition steps).
Their comparison for \ourmodel is shown in Fig.~\ref{tab:pinw-gibbs-compt}, where SGLD is much better.
For DAE, using Gibbs sampling and using SGLD produce similar data generation results.

We show more results next. \vspace{-1pt}

\begin{table}[t]
  \vspace{-10pt}
  \centering
  \setlength{\tabcolsep}{0.6pt}
  \begin{tabular}{c@{\hspace{4pt}}|@{\hspace{4pt}}c@{\hspace{4pt}}cccc@{\hspace{1.5pt}}|@{\hspace{1.5pt}}c}
    \multirow{5}{*}{
      \begin{minipage}{\elemfigwidth}
        \centering
        iteration 1000 \\[1.3cm]
        VAE pretraining \\[2pt]
        \raisebox{-.5\height}{\includegraphics[width=\elemfigwidth]{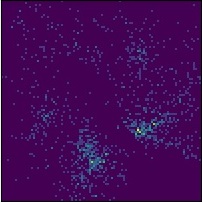}} \\[2pt]
        compt. loss $1.6\e{4}$
      \end{minipage}
    }
    &
    iteration &
    1100 & 1200 & 1300 & 1400 & 30000
    \\ \cmidrule(lr){2-7}
    &
    \textbf{\ourmodelpt} &
    \begin{minipage}{\elemfigwidth}
      \centering
      \raisebox{-.5\height}{\includegraphics[width=\elemfigwidth]{synth/\ourfigspt/xhist-langv-z-1100.jpg}}
    \end{minipage}
    &
    \begin{minipage}{\elemfigwidth}
      \centering
      \raisebox{-.5\height}{\includegraphics[width=\elemfigwidth]{synth/\ourfigspt/xhist-langv-z-1200.jpg}}
    \end{minipage}
    &
    \begin{minipage}{\elemfigwidth}
      \centering
      \raisebox{-.5\height}{\includegraphics[width=\elemfigwidth]{synth/\ourfigspt/xhist-langv-z-1300.jpg}}
    \end{minipage}
    &
    \begin{minipage}{\elemfigwidth}
      \centering
      \raisebox{-.5\height}{\includegraphics[width=\elemfigwidth]{synth/\ourfigspt/xhist-langv-z-1400.jpg}}
    \end{minipage}
    &
    \begin{minipage}{\elemfigwidth}
      \centering
      \raisebox{-.5\height}{\includegraphics[width=\elemfigwidth]{synth/\ourfigspt/xhist-langv-z-30000.jpg}}
    \end{minipage}	
    \\ \addlinespace[2pt]
    &
    compt. loss &
    $7.0\e{3}$ & $5.4\e{3}$ & $7.7\e{3}$ & $6.2\e{3}$ & $4.6\e{3}$
    \\ \cmidrule(lr){2-7}
    &
    \makecell[c]{\textbf{\ourmodelpt} w/o \\ compt. loss} &
    \begin{minipage}{\elemfigwidth}
      \centering
      \raisebox{-.5\height}{\includegraphics[width=\elemfigwidth]{synth/\ourfigsptnocompt/xhist-langv-z-1100.jpg}}
    \end{minipage}
    &
    \begin{minipage}{\elemfigwidth}
      \centering
      \raisebox{-.5\height}{\includegraphics[width=\elemfigwidth]{synth/\ourfigsptnocompt/xhist-langv-z-1200.jpg}}
    \end{minipage}
    &
    \begin{minipage}{\elemfigwidth}
      \centering
      \raisebox{-.5\height}{\includegraphics[width=\elemfigwidth]{synth/\ourfigsptnocompt/xhist-langv-z-1300.jpg}}
    \end{minipage}
    &
    \begin{minipage}{\elemfigwidth}
      \centering
      \raisebox{-.5\height}{\includegraphics[width=\elemfigwidth]{synth/\ourfigsptnocompt/xhist-langv-z-1400.jpg}}
    \end{minipage}
    &
    \begin{minipage}{\elemfigwidth}
      \centering
      \raisebox{-.5\height}{\includegraphics[width=\elemfigwidth]{synth/\ourfigsptnocompt/xhist-langv-z-30000.jpg}}
    \end{minipage}
    \\ \addlinespace[2pt]
    &
    compt. loss &
    $1.1\e{5}$ & $1.6\e{5}$ & $2.6\e{5}$ & $8.6\e{5}$ & $1.2\e{8}$
    \\ \cmidrule(lr){2-7}
    &
    DAE &
    \begin{minipage}{\elemfigwidth}
      \centering
      \raisebox{-.5\height}{\includegraphics[width=\elemfigwidth]{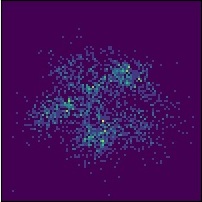}}
    \end{minipage}
    &
    \begin{minipage}{\elemfigwidth}
      \centering
      \raisebox{-.5\height}{\includegraphics[width=\elemfigwidth]{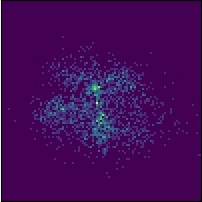}}
    \end{minipage}
    &
    \begin{minipage}{\elemfigwidth}
      \centering
      \raisebox{-.5\height}{\includegraphics[width=\elemfigwidth]{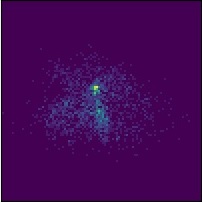}}
    \end{minipage}
    &
    \begin{minipage}{\elemfigwidth}
      \centering
      \raisebox{-.5\height}{\includegraphics[width=\elemfigwidth]{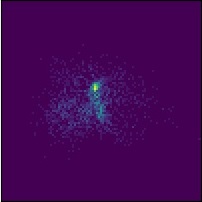}}
    \end{minipage}
    &
    \begin{minipage}{\elemfigwidth}
      \centering
      \raisebox{-.5\height}{\includegraphics[width=\elemfigwidth]{synth/\daefigs/xhist-langv-z-9200.jpg}}
    \end{minipage}
    \\ \addlinespace[2pt]
    &
    compt. loss &
    $6.3\e{3}$ & $2.2\e{3}$ & $1.9\e{3}$ & $9.7\e{2}$ & $2.2\e{2}$
  \end{tabular}
  \vspace{6pt}
  \captionsetup[table]{name=Figure}
  \captionof{table}{Generated data using $\bbZ$-space SGLD (\eqref{eqn:langevin-pz}) along the training process after VAE pretraining (iter. 1000)
    using \ourmodelpt, \ourmodelpt without compatibility loss, and DAE.
    The last DAE result is at iteration 9200, after which numerical overflow occurs.
  }
  \label{tab:pinw-compt-dae}
  \vspace{-20pt}
\end{table}

\textbf{Impact of the compatibility loss.} \hspace{4pt}
Fig.~\ref{tab:pinw-compt-dae} shows the training process of \ourmodelpt, \ourmodelpt without compatibility loss, and DAE, all after pretrained as a VAE (iter. 1000; leftmost panel), in terms of generated data distribution. 
We see that the normal \ourmodelpt behaves stably to the end and well approximates the data distribution along the training process.
Its compatibility loss is indeed decreasing.
On the other hand, \ourmodelpt without the compatibility loss diverges eventually, with an exploding compatibility loss.
Although it well optimizes the data-fitting loss \eqref{eqn:mle}, if compatibility is not enforced, the loss is not the data likelihood that we want to optimize.

\vspace{-1pt}
Note that \ourmodelpt without compatibility loss improves generation quality upon the VAE-pretrained model in the first few training iterations (\eg, at iter. 1100).
This is because the ELBO objective (\eqref{eqn:elbo}) of VAE also drives $q_\phi(z|x)$ towards the true posterior $p_\theta(z|x) \propto p(z) p_\theta(x|z)$ (defined with a specified prior $p(z)$) so compatibility approximately holds in the first few iterations, which makes the data-fitting loss (\eqref{eqn:mle}) effective.

\textbf{The collapse process of DAE.} \hspace{4pt}
Fig.~\ref{tab:pinw-compt-dae} (rows~1,3) also shows the comparison with DAE.
We see that after pretraining, DAE quickly shrinks its data distribution, ending up with a collapsed data distribution, and even finally comes to a numerical problem.
This is due to the mode-collapse behavior of the inference model $q_\phi(z|x)$ from minimizing the DAE loss \eqref{eqn:dae} and the subsequent insufficient determinacy, as explained in Sec.~\ref{sec:meth-data}.
Although its compatibility loss is also decreasing, this comes at the cost of the insufficient determinacy which hinders capturing the data distribution and also makes the training process unstable.

\begin{wraptable}{r}{.42\textwidth}
  \vspace{-20pt}
  \centering
  \setlength{\tabcolsep}{1.6pt}
  \begin{tabular}{ccc}
    VAE & \textbf{\ourmodel} & \textbf{\ourmodelpt} \\
    \includegraphics[width=\elemfigwidth]{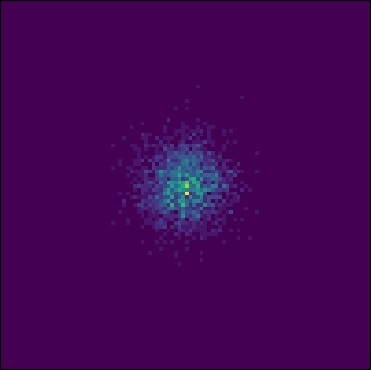} &
    \includegraphics[width=\elemfigwidth]{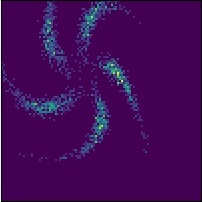} &
    \includegraphics[width=\elemfigwidth]{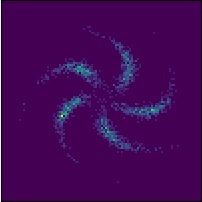}
  \end{tabular}
  \vspace{-2pt}
  \captionsetup[table]{name=Figure}
  \captionof{table}{Prior distributions $p(z)$ of VAE, \ourmodel (without pretraining), and \ourmodelpt (with VAE pretraining).
    Prior samples of \ourmodel/\ourmodelpt are drawn using $\bbZ$-space SGLD (\eqref{eqn:langevin-pz}).
  }
  \vspace{-12pt}
  \label{tab:pinw-prior}
\end{wraptable}

\textbf{Incorporating knowledge into the conditionals.} \hspace{4pt}
We plot the prior distributions in Fig.~\ref{tab:pinw-prior} of VAE, \ourmodel, and \ourmodelpt with VAE pretraining, in the form of the histogram of the drawn $z$ samples.
For \ourmodel/\ourmodelpt, samples are drawn by $\bbZ$-space SGLD (\eqref{eqn:langevin-pz}) using the same step size $\varepsilon = 3\e{-4}$ and number of iterations 100.
Compared with VAE, the priors learned by \ourmodel and \ourmodelpt are more expressive. 
For \ourmodel which is not subjected to any further constraints, there may be multiple $p_\theta(x|z)$ and $q_\phi(z|x)$ that are compatible and well match the given data distribution.
Using a standard Gaussian prior for pretraining successfully incorporates the knowledge of a centered and centrosymmetric prior into the conditional model $p_\theta(x|z)$.
The arbitrariness of possible $p_\theta(x|z)$ and $q_\phi(z|x)$ is largely mitigated in this way.
This observation meets the discussion in Sec.~\ref{sec:expm-synth} (paragraph~4) on the aggregated posteriors in Fig.~\ref{tab:pinw-gen-post}.

\begin{SCtable}[][t]
  \vspace{-4pt}
  \centering
  \setlength{\tabcolsep}{0.6pt}
  \begin{tabular}{cccc@{\hspace{1pt}}|@{\hspace{1pt}}cc}
    data & DAE & VAE & BiGAN & \textbf{\ourmodel} & \textbf{\ourmodelpt} \\
    \begin{minipage}{\elemfigwidth}
      \centering
      \raisebox{-.5\height}{\includegraphics[width=\elemfigwidth]{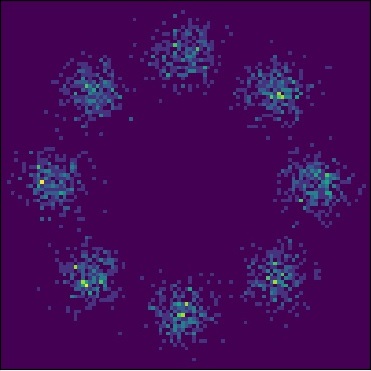}}
    \end{minipage}
    &
    \begin{minipage}{\elemfigwidth}
      \centering
      \raisebox{-.5\height}{\includegraphics[width=\elemfigwidth]{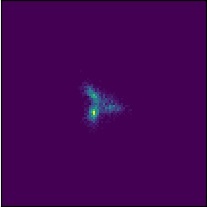}}
    \end{minipage}
    &
    \begin{minipage}{\elemfigwidth}
      \centering
      \raisebox{-.5\height}{\includegraphics[width=\elemfigwidth]{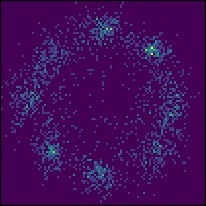}}
    \end{minipage}
    &
    \begin{minipage}{\elemfigwidth}
      \centering
      \raisebox{-.5\height}{\includegraphics[width=\elemfigwidth]{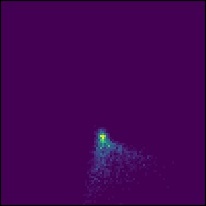}}
    \end{minipage}
    &
    \begin{minipage}{\elemfigwidth}
      \centering
      \raisebox{-.5\height}{\includegraphics[width=\elemfigwidth]{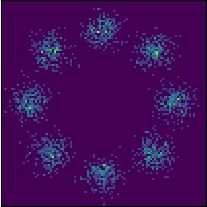}}
    \end{minipage}
    &
    \begin{minipage}{\elemfigwidth}
      \centering
      \raisebox{-.5\height}{\includegraphics[width=\elemfigwidth]{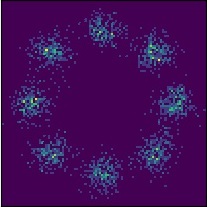}}
    \end{minipage}
    \\ \addlinespace[3pt]
    \makecell[r]{\small class-wise \\ \small aggregated \\ \small posterior} &
    \begin{minipage}{\elemfigwidth}
      \centering
      \raisebox{-.5\height}{\includegraphics[width=\elemfigwidth+\elemfigwidthamend]{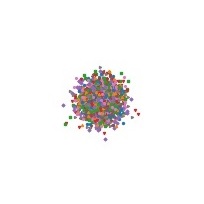}}
    \end{minipage}
    &
    \begin{minipage}{\elemfigwidth}
      \centering
      \raisebox{-.5\height}{\includegraphics[width=\elemfigwidth+\elemfigwidthamend]{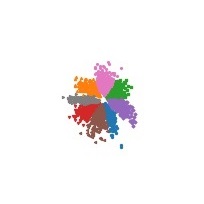}}
    \end{minipage}
    &
    \begin{minipage}{\elemfigwidth}
      \centering
      \raisebox{-.5\height}{\includegraphics[width=\elemfigwidth+\elemfigwidthamend]{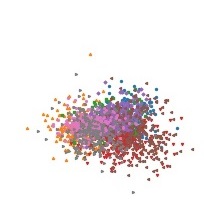}}
    \end{minipage}
    &
    \begin{minipage}{\elemfigwidth}
      \centering
      \raisebox{-.5\height}{\includegraphics[width=\elemfigwidth+\elemfigwidthamend]{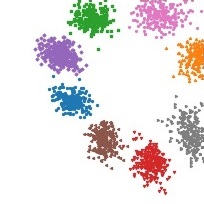}}
    \end{minipage}
    &
    \begin{minipage}{\elemfigwidth}
      \centering
      \raisebox{-.5\height}{\includegraphics[width=\elemfigwidth+\elemfigwidthamend]{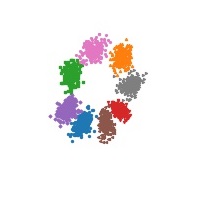}}
    \end{minipage}
  \end{tabular}
  \captionsetup[table]{name=Figure,font=small,skip=-2pt}
  \captionof{table}{Generated data (DAE and \ourmodel use $\bbZ$-space SGLD \eqref{eqn:langevin-pz}) and class-wise aggregated posteriors of DAE, VAE, BiGAN and \ourmodel,
    on the ``\textbf{8gaussians}'' dataset.
    Also shows results of \ourmodelpt that is PreTrained as a VAE.
    (Best view in color.)
  }
  \label{tab:8gauss-gen-post}
\end{SCtable}

\textbf{Results on another similar dataset (8gaussians).} \hspace{4pt}
We repeat all the settings above (for the ``pinwheel'' dataset) on another similar synthetic dataset ``8gaussians'' (Fig.~\ref{tab:8gauss-gen-post} top-left).
Its data distribution also has a (nearly) non-connected support, with 8 connected components.
The only difference in settings is that BiGAN uses a learning rate $1\e{-4}$ for updating its generator and $3\e{-5}$ for its discriminator.
We see similar observations from the results in Fig.~\ref{tab:8gauss-gen-post}.

For data generation, DAE again produces a collapsed data distribution.
VAE's data distribution is blurred and the 8 clusters are touching, making the support connected due to using the standard Gaussian prior.
BiGAN is unstable on this dataset and does not produce a reasonable data distribution, despite some parameter tuning.
In contrast, \ourmodel and \ourmodelpt still recover the data distribution faithfully, with clear 8 clusters.
Again, the advantage to overcome manifold mismatch is demonstrated.

For representation learning, DAE again collapses the class-wise aggregated posteriors of all classes and puts them in the same place.
VAE identifies the latent clusters but which are compressed to the origin and squeezed together to border each other.
BiGAN's aggregated posteriors of different classes largely overlap each other.
In contrast, \ourmodel and \ourmodelpt separate the latent clusters clearly.
\ourmodelpt additionally embodies the knowledge from the VAE-pretrained likelihood model that the prior hence the (all-class) aggregated posterior is centered and centrosymmetric,
without suffering fitting the data distribution.

In all, these observations again verify that \ourmodel achieves both superior generation and representation learning performances.

\subsection{Real-World Experiments} \label{supp:expm-real}

\textbf{Model architecture.} \hspace{4pt}
All methods use the same architecture of the inference model (encoder) and the likelihood model (decoder), illustrated in Fig.~\ref{fig:flow-arch} and detailed in Table~\ref{tab:arch}
(except for the reported results from~\citep{lamb2017gibbsnet} of BiGAN and GibbsNet listed in Table~\ref{tab:downstream}, which are around random guess using the same flow architecture).
The Gaussian variance of the likelihood model $p_\theta(x|z)$ is selected to be $0.01$ for all dimensions.
All methods use the Adam optimizer~\citep{kingma2014adam} with learning rate $1\e{-4}$ and batch-size 100 for 100 epochs.

\textbf{Data generation.} \hspace{4pt}
VAE uses ancestral sampling, and DAE uses its standard Gibbs sampling procedure initialized from $z_0 \sim \clN(0,I_{d_\bbZ})$.
\ourmodelpt generates data by running $\bbX$-space SGLD (\eqref{eqn:langevin-px}) with step size $\varepsilon = 1\e{-3}$ initialized from $x_0 \sim p_\theta(\cdot|z_0)$ where $z_0 \sim \clN(0,I_{d_\bbZ})$.

\textbf{Downstream classification.} \hspace{4pt}
For the downstream classification task on the latent space $\bbZ$, 
we sample one latent representation $z$ for each data point $x$ directly from the learned inference model $q_\phi(z|x)$,
and then train a 2-layer MLP classifier with 10 hidden nodes on top of the latent representation.
Results of DAE, VAE and \ourmodelpt in Table~\ref{tab:downstream} are averaged over 10 random trials.
All downstream classifiers are trained for 100 epochs.

\textbf{Training strategy of our method.} \hspace{4pt}
\ourmodelpt first pretrains its likelihood and inference models as in a VAE using the ELBO objective \eqref{eqn:elbo} (with $\beta = 1$ on MNIST and $\beta = 0.01$ on SVHN) for 100 epochs, 
and then trains them using \ourmodel ($1\e{-3} \times$ compatibility loss \eqref{eqn:compt-obj-jac-hut-simp} ${}+{}$ data-fitting loss \eqref{eqn:mle}) with a $10$-times smaller learning rate for the likelihood model (same learning rate for the inference model).
On SVHN, we clamp the standard Gaussian random seed $e$ in the reparameterization step that initializes $z_0$ to be within the interval $[-0.1, 0.1]$ (element-wise).

\textbf{Remark on the VAE pretraining.} \hspace{4pt}
For real-world images, the optimization process of \ourmodel from a cold start is unstable,
possibly because the data distribution roughly concentrates on a low-dimensional space thus is nearly not absolutely continuous,
while the \ourmodel{}-defined distribution is absolutely continuous (see Lem.~\ref{lem:joint-ac} in Appx.~\ref{supp:proofs-joint-ac}).
Improved techniques to handle this issue are important future work.
Moreover, the VAE pretraining downweighs the effect of the prior (\ie, using a small $\beta$ on SVHN), so it is approximately a DAE, which does not show a reasonable result (Fig.~\ref{tab:realdata}).
So \ourmodel is the key to the high-quality results.

\begin{table}[t]
  \centering
  \caption{Inference and likelihood model architectures for MNIST and SVHN}
  \label{tab:arch}
  \begin{tabular}{c|c|c}
    \toprule
    Layers & In-Out Size & Stride \\
    \midrule
    \multicolumn{3}{c}{\textbf{Inference Model $q_\phi(z|x)$ Architecture for MNIST}-C-QNN} \\
    Input $x$ &  $1 \times 28 \times 28$ & \\
    $5 \times 5$ GatedConv2d (32),  Sigmoid & $32 \times 28 \times 28$ & 1 \\
    $5 \times 5$ GatedConv2d (32),  Sigmoid & $32 \times 14 \times 14$ & 2 \\
    $5 \times 5$ GatedConv2d (64),  Sigmoid & $64 \times 14 \times 14$ & 1 \\
    $5 \times 5$ GatedConv2d (64),  Sigmoid & $64 \times  7 \times  7$ & 2 \\
    $5 \times 5$ GatedConv2d (64),  Sigmoid & $64 \times  7 \times  7$ & 1 \\
    $7 \times 7$ GatedConv2d (256), Sigmoid & $256 \times 1 \times 1$ & 1 \\
    Output $q_\text{nn}$, squeeze & 256 & \\
    \midrule
    \multicolumn{3}{c}{\textbf{Inference Model $q_\phi(z|x)$ Architecture for SVHN}-C-QNN} \\
    Input $x$     &  $3 \times 32 \times 32$ & \\
    $5 \times 5$ Conv2d (32),  LReLU   & $32 \times 28 \times 28$ & 1 \\
    $4 \times 4$ Conv2d (64),  LReLU   & $64 \times 13 \times 13$ & 2 \\
    $4 \times 4$ Conv2d (128), LReLU   & $128 \times 10 \times 10$ & 1 \\
    $4 \times 4$ Conv2d (256), LReLU   & $256 \times 4 \times 4$ & 2 \\
    $4 \times 4$ Conv2d (512), LReLU   & $512 \times 1 \times 1$ & 1 \\
    $4 \times 4$ Conv2d (256), Sigmoid & $256 \times 1 \times 1$ & 1 \\
    Output $q_\text{nn}$, squeeze & 256 & \\
    \midrule
    \multicolumn{3}{c}{\textbf{Reparameterization Block for $q_\phi(z|x)$ for MNIST and SVHN}} \\
    Input $q_\text{nn}$ & 256 & \\
    Output-1 $q_\mu$: Linear $256 \times 64$ & 64 & \\
    Draw $e \sim \clN(0, I_{d_\bbZ})$ and output $z_0 = q_\mu + e \odot q_\sigma$ & 64 & \\
    \midrule
    \multicolumn{3}{c}{\textbf{Flow-Parameter Block for $q_\phi(z|x)$ for MNIST and SVHN}} \\
    Input $q_\text{nn}$ & 256 & \\
    Output-1 $\mathbf{v}_{1:8}$: Linear $256 \times 512$ & 512 & \\
    Output-2 $\mathbf{b}$: Linear $256 \times 8$ & 512 & \\
    Output-3 $\mathbf{R}$: Linear $256 \times (64 \times 64)$ & $(64 \times 64)$ & \\
    Output-4 $\tilde{\mathbf{R}}$: Linear $256 \times (64 \times 64)$ & $(64 \times 64)$ & \\
    \midrule
    \midrule
    \multicolumn{3}{c}{\textbf{Likelihood Model $p_\theta(z|x)$ Architecture for MNIST}} \\
    Input $z$ & $64 \times 1 \times 1$ & \\
    $7 \times 7$ GatedConvT2d (64), Sigmoid & $64 \times 7 \times 7$ & 1 \\
    $5 \times 5$ GatedConvT2d (64), Sigmoid & $64 \times 7 \times 7$ & 1 \\
    $5 \times 5$ GatedConvT2d (64), Sigmoid & $64 \times 14 \times 14$ & 2 \\
    $5 \times 5$ GatedConvT2d (32), Sigmoid & $32 \times 14 \times 14$ & 1 \\
    $5 \times 5$ GatedConvT2d (32), Sigmoid & $32 \times 28 \times 28$ & 2 \\
    $5 \times 5$ GatedConvT2d (32), Sigmoid & $32 \times 28 \times 28$ & 1 \\
    $1 \times 1$ GatedConv2d (1), Sigmoid & $1 \times 28 \times 28$ & 2 \\
    Output $x$ & $1 \times 28 \times 28$ & \\
    \midrule
    \multicolumn{3}{c}{\textbf{Likelihood Model $p_\theta(z|x)$ Architecture for SVHN}} \\
    Input $z$     &  $64 \times 1 \times 1$ & \\
    $4 \times 4$ ConvT2d (256), LReLU & $256 \times 4 \times 4$ & 1 \\
    $4 \times 4$ ConvT2d (128), LReLU & $128 \times 10 \times 10$ & 1 \\
    $4 \times 4$ ConvT2d (64), LReLU & $64 \times 13 \times 13$ & 1 \\
    $4 \times 4$ ConvT2d (32), LReLU & $32 \times 28 \times 28$ & 2 \\
    $5 \times 5$ ConvT2d (32), LReLU & $32 \times 32 \times 32$ & 1 \\
    $1 \times 1$ ConvT2d (32), LReLU & $32 \times 32 \times 32$ & 1 \\
    $1 \times 1$ Conv2d (32), Sigmoid & $32 \times 32 \times 32$ & 1 \\
    Output $x$ & $3 \times 32 \times 32$ & \\
    \bottomrule
  \end{tabular}
\end{table}

\end{document}